\documentclass[sigconf]{acmart}

\usepackage{mathtools}
\usepackage{tabularx}
\usepackage{comment}
\usepackage{algorithm, algorithmicx, algpseudocode}
\usepackage[mathscr]{euscript}
\usepackage{graphicx}
\usepackage{enumitem}
\usepackage{subfigure}
\usepackage{float}
\usepackage{amsmath,amsthm}
\newtheorem{theorem}{Theorem}
\newtheorem{lemma}[theorem]{Lemma}
\newtheorem{assumption}{Assumption}
\newtheorem{proposition}{Proposition}
\newtheorem{definition}{Definition}
\newtheorem{remark}{Remark}

\DeclareMathOperator*{\argmax}{\arg\max}

\newcommand{\indep}{\perp \!\!\! \perp}

\AtBeginDocument{%
  }

\setcopyright{acmlicensed}
\copyrightyear{2025}
\acmYear{2025}
\acmDOI{XXXXXXX.XXXXXXX}

\acmConference[KDD '25]{31st ACM SIGKDD Conference on Knowledge Discovery and Data Mining}{August 3--7, 2025}{Toronto, Canada}
\acmISBN{XXX-X-XXXX-XXXX-X/25/08}

\begin{document}

\title{Progressive Generalization Risk Reduction for Data-Efficient Causal Effect Estimation}

\author{Hechuan Wen}
\affiliation{%
  \institution{The University of Queensland}
  \city{Brisbane}
  \country{Australia}
}
\email{h.wen@uq.edu.au}

\author{Tong Chen}
\affiliation{%
  \institution{The University of Queensland}
  \city{Brisbane}
  \country{Australia}
}
\email{tong.chen@uq.edu.au}

\author{Guanhua Ye}
\affiliation{%
  \institution{Beijing University of Posts and Telecommunications, China}
  \city{}
  \country{}
}
\email{g.ye@bupt.edu.cn}

\author{Li Kheng Chai}
\affiliation{%
 \institution{Health and Wellbeing Queensland}
 \city{Brisbane}
 \country{Australia}}
 \email{likheng.chai@hw.qld.gov.au}

\author{Shazia Sadiq}
\affiliation{%
  \institution{The University of Queensland}
  \city{Brisbane}
  \country{Australia}
}
\email{shazia@eecs.uq.edu.au}

\author{Hongzhi Yin}
\authornote{Corresponding Author}
\affiliation{%
  \institution{The University of Queensland}
  \city{Brisbane}
  \country{Australia}
}
\email{h.yin1@uq.edu.au}

\renewcommand{\shortauthors}{Wen et al.}

\begin{abstract}
  Causal effect estimation (CEE) provides a crucial tool for predicting the unobserved counterfactual outcome for an entity. As CEE relaxes the requirement for ``perfect'' counterfactual samples (e.g., patients with identical attributes and only differ in treatments received) that are impractical to obtain and can instead operate on observational data, it is usually used in high-stake domains like medical treatment effect prediction. Nevertheless, in those high-stake domains, gathering a decently sized, fully labelled observational dataset remains challenging due to hurdles associated with costs, ethics, expertise and time needed, etc., of which medical treatment surveys are a typical example. Consequently, if the training dataset is small in scale, low generalization risks can hardly be achieved on any CEE algorithms. 
  
  Unlike existing CEE methods that assume the constant availability of a dataset with abundant samples, in this paper, we study a more realistic CEE setting where the labelled data samples are scarce at the beginning, while more can be gradually acquired over the course of training -- assuredly under a limited budget considering their expensive nature. Then, the problem naturally comes down to actively selecting the best possible samples to be labelled, e.g., identifying the next subset of patients to conduct the treatment survey. However, acquiring quality data for reducing the CEE risk under limited labelling budgets remains under-explored until now. To fill the gap, we theoretically analyse the generalization risk from an intriguing perspective of progressively shrinking its upper bound, and develop a principled label acquisition pipeline exclusively for CEE tasks. With our analysis, we propose the Model Agnostic Causal Active Learning (MACAL) algorithm for batch-wise label acquisition, which aims to reduce both the CEE model's uncertainty and the post-acquisition distributional imbalance simultaneously at each acquisition step. Extensive experiments are conducted on three datasets, where a clear empirical performance gain from MACAL is observed over state-of-the-art active learning baselines. The implementation repository is open-sourced at: \url{https://github.com/uqhwen2/MACAL}.
\end{abstract}



\maketitle

\section{Introduction}
Understanding causal effects to support decision-making in high-stake domains is crucial, where typical examples include randomized control trials in medication \cite{pilat2015exploring}, A/B testing for business decision-making \cite{kohavi2015online}, and the potential in advancing big data management \cite{nguyen2017argument,hung2017computing,nguyen2017retaining,yin2020overcoming}. As performing large-scale and statistically reliable human tests is prohibitively costly, algorithms for causal effect estimation (CEE) using passively observed data samples have become a promising solution \cite{johansson2016learning,yao2021survey,wang2024optimal}. In short, a CEE algorithm is trained with observational data to predict the counterfactual outcome for an entity, e.g., what the outcome will be if a patient received the other treatment, instead of the one already had.

To perform CEE, a common practice is to build a regression model that estimates a continuous effect value \cite{johansson2016learning,yao2021survey,shalit2017estimating,yao2018representation,chauhan2024dynamic}, which is trained on the observational data containing two groups of samples. The groups are formed based on the treatment\footnote{Depending on the context of applications, treatments can also be interpreted as interventions, services, or information provided to an individual.} imposed on each sample, where each sample consists of raw attributes drawn from a well-defined feature space (e.g., a patient's health indicators), and a label that corresponds to the observed outcome after receiving the binary treatment (e.g., blood sugar concentration after taking one diabetes medicine).  


For training a capable CEE model, a quality observational dataset with diverse and abundant samples is highly desirable. On the one hand, as in many other tasks, richer training data enables the model to better capture predictive patterns. On the other hand, this also helps maintain some pivotal CEE assumptions \cite{imbens2015causal} on the training data, where the positivity (a.k.a. overlapping) assumption is arguably a very fundamental one. Practically, as each distinct sample only receives one treatment, positivity requires statistically identical attribute distributions between two treatment groups, such that counterfactual predictions can be confidently made. 
Given that, the majority of CEE models \cite{shalit2017estimating, louizos2017causal,yao2018representation, wen2023variational} are trained on a fix-sized dataset with sufficient samples, where the positivity assumption can easily hold. However, such a setting oversimplifies the data availability in high-stake domains -- the major adopters of CEE. The challenge often lies in obtaining the ground truth label on the treatment outcome of each sample. For instance, though clinics record patients' health-related attributes when performing a treatment, the real post-treatment outcome can only be obtained through longitudinal surveys \cite{rosenbaum2005danish} over a long time period, and is subject to ethical concerns. Furthermore, in a business context, the effect of a treatment (e.g., a sales campaign) cannot be reliably quantified without ample expertise and evidence. As a result, the sufficiency of labelled training data in CEE is not always guaranteed, hindering the real-life practicality of existing CEE pipelines. 

As a response, in this paper, we subsume CEE under a more realistic setting: \textit{the availability of labelled samples in both treatment groups monotonically grows}. Essentially, this translates into an active learning (AL) paradigm \cite{settles2009active} for CEE tasks. Considering the expensive nature of labelling the treatment outcome of all samples, we allow a CEE model's training to start with a very small portion of labelled data within both treatment groups, then gradually and selectively extend to the remaining unlabelled samples by assigning post-treatment outcome labels.  
As a side effect of this more practical setting, a CEE model with low generalization risk is harder to obtain, especially at early training stages where the labelled dataset is small in scale, limiting the informativeness and compliance to positivity. 
With a standard AL algorithm \cite{wang2015active, gal2017deep, sener2018active, ren2021survey, zhan2022comparative}, during the progress of label acquisition, informative samples can be selected for labelling and enriching the training data. In scenarios where counterfactual predictions are not needed, such active label acquisition is proven useful \cite{ren2021survey}. However, in CEE, with the existence of two treatment groups, the direct adoption of AL will incur sub-optimal results within the given labelling budget. This is because the acquisition criterion is not designed to account for the crucial positivity assumption, thus failing to align the sample distributions between the two groups. Consequently, the generalizability of the actively trained CEE algorithm will be harmed by the ill-posed data distribution.  




Bearing this motivation, we aim to answer the important question: 
\textit{how to label the most informative samples in CEE tasks?} 
Assuming the label availability of the samples in both treatment groups, active learning for CEE should meet two desiderata: 1) maximize the positivity among the chosen samples to be labelled during dataset expansion; 2) improve the generalizability of the CEE model. Recently, this niche area of study has started drawing more attention, however, as we will discuss later, only a few models \cite{qin2021budgeted,jesson2021causal,addanki2022sample} are suited for CEE task with AL. In this paper, we propose an intriguing perspective to conduct efficient selective labelling exclusively for the CEE task under the AL paradigm. Unlike the other closely-related approaches \cite{qin2021budgeted}, we look directly into the theoretical analysis of the risk upper bound without loosening it, where we then propose a theory-inspired, simplified yet effective label acquisition criterion for batch-mode AL with paired samples. We summarize our contributions as follows:

\begin{itemize}
    \item We study the well-under-explored yet important and practical topic -- active learning for causal effect estimation, where the conventional active learning paradigm failed to obtain the optimal label acquisition scheme, and the existing studies are yet able to well solve the violation of positivity during the process of the label acquisition.
    
    \item We propose a theoretical framework for causal effect estimation under the active learning paradigm, where a more informative risk upper bound is decomposed and leveraged for algorithm design. Inspired by the proposed theory, we come up with a simplified yet effective label acquisition criterion, namely Model Agnostic Causal Active Learning (MACAL) for label acquisition by promoting the individual sample diversity in different treatment groups and penalizing the treatment pair dissimilarity. Also, the mathematically guaranteed risk convergence is given under certain conditions to justify the acquisition algorithm.
    
    \item We compete against numerous SOTA baselines by fixing the label acquisition criterion as the only variable during benchmarking. Extensive experiments are conducted on various combinations of different datasets and downstream CEE models, and demonstrable performance gain from MACAL is observed across all comparisons.
\end{itemize}

\section{Preliminaries}
\subsection{Causal Effect Estimation}

Under the potential outcome framework \cite{imbens2015causal}, the individual treatment effect (ITE) is expected to be estimated with the tabular dataset $\mathcal{D}=\{\textbf{x}_{i}, t_{i}, y_{i}\}_{i=1}^{N}$, where $\textbf{x}_i$, $t_i$, $y_i$ are respectively the raw feature variables, observed treatment, treatment outcome that correspond to the $i$-th individual. For simplicity, we consider the binary treatment $t$ of 1 and 0 to denote the different treatment statuses, respectively. The ground truth ITE for an individual with feature vector $\textbf{x}$ is defined as:
\begin{equation}\label{eq:true_effect}
   \tau(\textbf{x}) = \mathbb{E}[Y^{t=1} - Y^{t=0}| \textbf{x}],
\end{equation} where $Y^{t=1}$ and $Y^{t=0}$ are the unobserved potential outcomes with treatment $t=1$ and $t=0$ respectively. Generally, under the deep neural network learning framework \cite{shalit2017estimating, louizos2017causal}, the common practice is to transform the raw variable $\textbf{x}_{i}$ into the latent representation $\textbf{z}_i$ for individual $i$, then such representation is utilised for ITE prediction. To evaluate the performance of the CEE model, the generalization risk of the model denoted by $\epsilon_{\text{PEHE}}$, is defined in (\ref{eq:pehe}) according to the well-recognized literature \cite{hill2011bayesian}. The lower the value of $\epsilon_{\text{PEHE}}$, the better the performance of the predictor.

\begin{definition}
    \textit{The expected Precision in Estimation of Heterogeneous Effect (PEHE) of the CEE model $f=\{\phi, \Phi\}$ with squared loss metric $L(\cdot, \cdot)$ is defined as:}
\begin{equation}
    \epsilon_{\text{PEHE}}(f)=\int_{\mathcal{X}}L_{f}(\mathbf{x})p(\mathbf{x})d\mathbf{x},
    \label{eq:pehe}
\end{equation} where we denote $L(\hat{\tau}(\mathbf{x}), \tau(\mathbf{x}))$ as $L_{f}(\mathbf{x})$ for notation simplicity. The $\tau(\mathbf{x})$ is the ground truth treatment effect defined in (\ref{eq:true_effect}) and $\hat{\tau}(\mathbf{x})$ is its estimation.
\end{definition}

For clarity, we formally define the CEE problem as follows:

\begin{definition}[Causal Effect Estimation\label{definition:cee}]
    Given the dataset $\mathcal{D}$, the common pipeline is to train an estimator $f$ that can give the causal effect estimation $\hat{\tau}(\textbf{x})$ to be as accurate as possible to the ground truth $\tau(\textbf{x})$ for entity \textbf{x}, i.e., we aim to empirically minimize the evaluation metric $\epsilon_{\text{PEHE}}(f)$.
\end{definition}

To validate the CEE $\hat{\tau}(\textbf{x})$, three common assumptions from the causal inference literature are needed to lay the theoretical foundation. These assumptions are stated as follows:

\begin{assumption}[Stable Unit Treatment Value Assumption (SUTVA) \cite{imbens2015causal}\label{assumption:sutva}]
For any individual: (a) the potential outcomes for the individual do not vary with the treatment assigned to other individuals; and (b) there are no different forms or versions of each treatment that may lead to different potential outcomes.
\end{assumption}

\begin{assumption}[Unconfoundedness\label{assumption:unconf}]
The independence relation $\{Y^{t=0}, Y^{t=1}\}\indep t | \textbf{x}$ holds, where treatment assignment $t$ is independent to the potential outcomes $\{Y^{t=0}, Y^{t=1}\}$ given the covariate \textbf{x}. 
\end{assumption}

\begin{assumption}[Positivity\label{assumption:positivity}]
For every individual with feature covariate $\textbf{x}$, the treatment assignment mechanism obeys: $0<p(t=1|\textbf{x})<1$.
\end{assumption}

The causal effect identifiability stated in Proposition \ref{proposition:identifiability} is needed to finalize the validation of the estimation results. The proof of the proposition is provided in Appendix \ref{appendix:identifiability}.

\begin{proposition}[Identifiability\label{proposition:identifiability}]
The causal effect is identifiable if and only if the SUTVA, the unconfoundedness, and the positivity assumptions hold.
\end{proposition}

\begin{figure*}[ht!]
  \centering
  \includegraphics[scale=0.65]{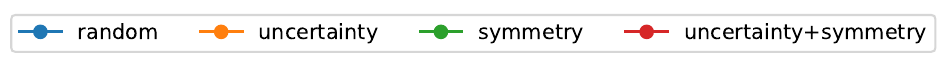}\\
  \vspace{-0.3cm}
  \subfigure[$\sqrt{\epsilon_{\text{PEHE}}}$\label{fig:epsilon}]{\includegraphics[width=0.325\textwidth]{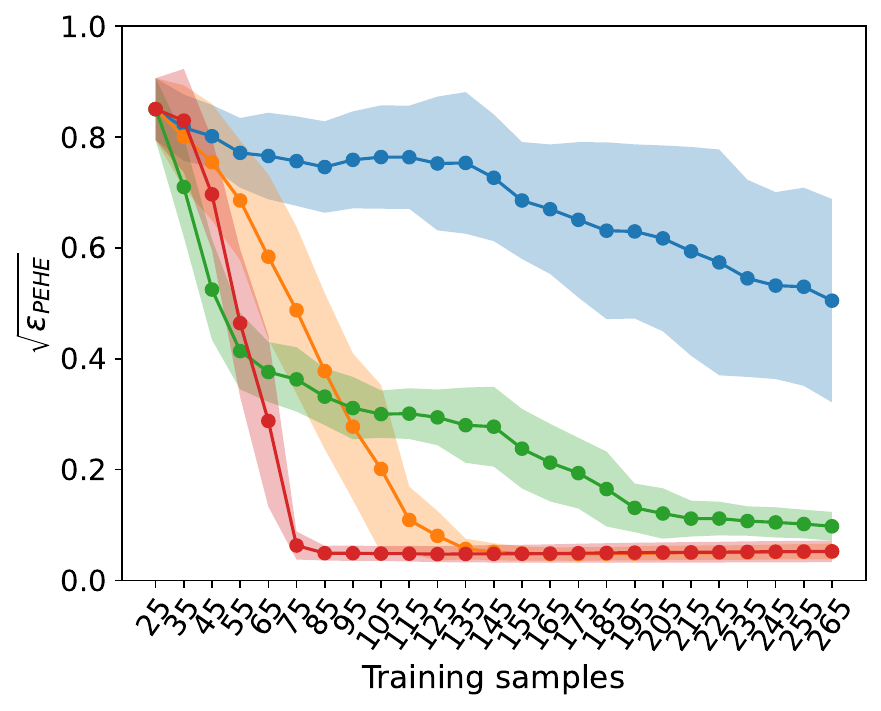}}
  \subfigure[Wasserstein Distance\label{fig:wass}]{\includegraphics[width=0.32\textwidth]{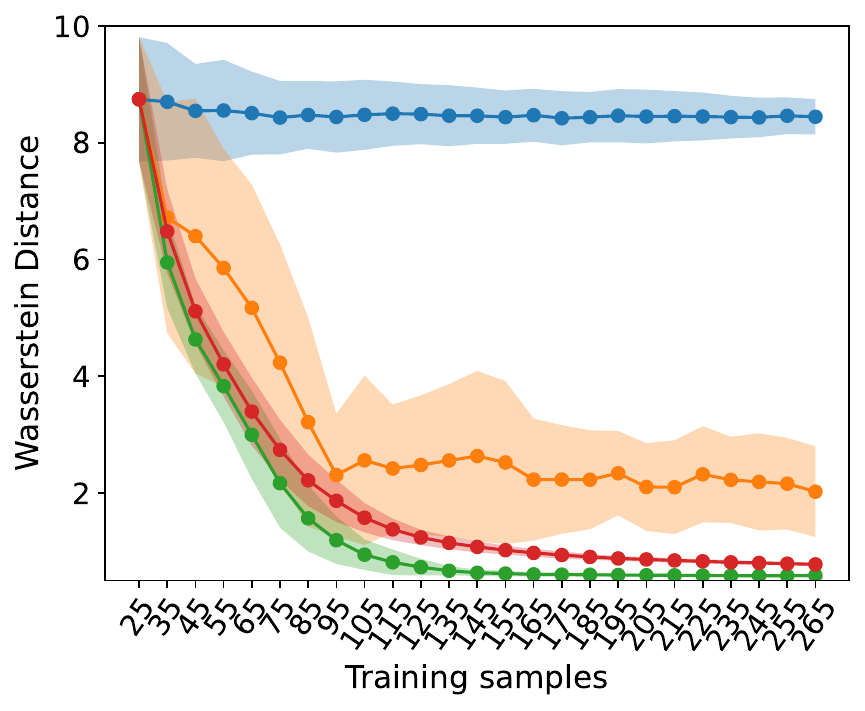}}
  \subfigure[Model Variance\label{fig:var}]{\includegraphics[width=0.33\textwidth]{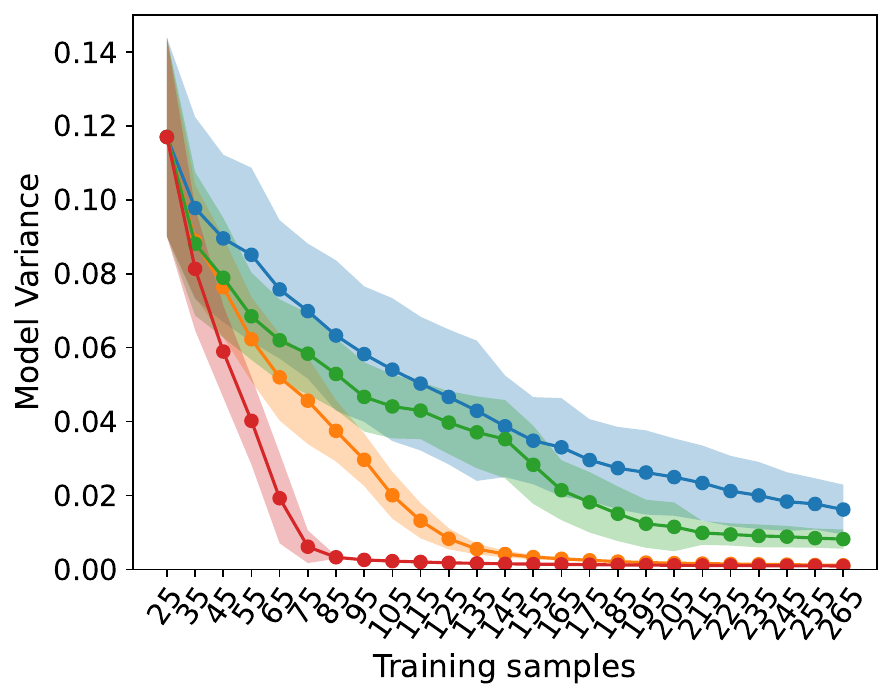}}
  \vspace{-0.3cm}
  \caption{Toy example indicating the importance of considering the reduction of both distributional discrepancy and model variance to help quickly achieve the lowest risk.}
  \label{fig:toy_dataset}
  \vspace{-0.3cm}
\end{figure*}

\subsection{Active Learning as the Challenge and Opportunity\label{section:optimal_ac_on_cee}}

When CEE meets active learning, the general logistics of the query steps become: 1) Let the CEE model $f$ get trained on the labelled training set $\mathcal{D}_{\text{train}}=\{\textbf{x}_{i}, t_{i}, y_{i}\}_{i=1}^{N_{\text{train}}}$. 2) Given the unlabelled pool set $\mathcal{D}_{\text{pool}}=\{\textbf{x}_{i}, t_{i}\}_{i=1}^{N_{\text{pool}}}$, the pre-defined label acquisition criterion (normally has trained model $f$ embedded, e.g., uncertainty-aware CEE model \cite{jesson2021causal}) examines through the pool set and returns a subset of it, i.e., $\Tilde{\mathcal{D}}$, for the oracle to label. 3) The labelled subset $\Tilde{\mathcal{D}}$ is added to the training set for which the CEE model $f$ can get updated before the upcoming querying round starts, then return to Step 1). Note, that the samples with attributes are already available in the pool set but without the labels, the process only attaches ground truth labels to them. Such a recursive procedure terminates mostly when the desired performance is reached or the labelling budget gets exhausted. Subsequently, let's take one step further from the conventional CEE problem defined in Definition \ref{definition:cee}, we form the research problem of CEE under the AL paradigm as follows:

\begin{definition}[Active Causal Effect Estimation\label{definition:acee}]
    Causal effect estimations with active learning aims to expand the current dataset with more informative samples such that the trained model's estimation risk $\epsilon_{\text{PEHE}}$ can be significantly reduced before the exhaustion of the labelling budget.
\end{definition}

It is noted that the Assumption \ref{assumption:positivity} regarding positivity is a vulnerable one in the real-world scenario for the CEE research field. The risk of deviating from such an assumption and thus leading to the unidentifiable causal effect has been widely discussed in \cite{jesson2020identifying,wen2023predict}. As a result of the active causal effect estimation defined in Definition \ref{definition:acee}, if the selective labelling process keeps introducing more imbalance (e.g., a non-negative crucial part of the CEE risk upper bound \cite{shalit2017estimating}) to the current treatment groups' distributions, the use of AL to expand the labelled dataset will hardly help obtain optimal CEE model with significant estimation risk reduction. However, as we will discuss in the following, AL also sheds light on reducing the estimation risk by expanding the dataset if it is properly configured.

Assume that we initially have a large enough pool set $\mathcal{D}_{\text{pool}}$, but labelling all the samples, i.e., obtaining all ITEs, is infeasible due to the considerable cost of time and capital. Ideally, there exists a smallest optimal subset $\mathcal{D}_{\text{opt}}$ where the positivity assumption holds across the sample space $\mathcal{X}$. Additionally, treatment groups' distributions, $p^{t=1}_{\text{opt}}$ and $p^{t=0}_{\text{opt}}$, are identical such that the distributional discrepancy measured by integral probability metric (IPM) is statistically zero, i.e., $\text{IPM}(p^{t=1}_{\text{opt}}, p^{t=0}_{\text{opt}})=0$. Given a sparse warm-up set (at $i=0$ query step), the current distributions of the treatment groups, i.e., $p^{t=1}_{i=0}$ and $p^{t=0}_{i=0}$, are realistically not the same, i.e.,  $\text{IPM}(p^{t=1}_{i=0}, p^{t=0}_{i=0})=\mathcal{I}_{i=0}\neq0$. During the recursive selective labelling process, a growing number of samples are added into the training set $\mathcal{D}_{\text{train}}$ -- \textit{the challenge is}, the current disparity between different treatment groups' distributions can be amplified, e.g., $\mathcal{I}_{i=10}\gg\mathcal{I}_{i=0}$ after 10 uncontrolled query steps, thus further countering the positivity even with more data. Meanwhile, the opportunity is, with proper acquisition setup, we can not only reduce the imbalance in the training set after every query step, i.e., $\mathcal{I}_{i+1}\ll\mathcal{I}_{i}$, but also quickly converge to the optimal set by using the smallest budget, e.g., for each of the treatment group $t$, $\text{IPM}(p^{t}_{\text{opt}}, p^{t}_{i=\text{I}})\rightarrow0$ after I iterations, reaching the lowest risk. 

Therefore, keep reconciling the positivity assumption during the active learning process plays a crucial role in obtaining a lower risk for CEE. In what follows, we detail our label acquisition design in every query step to fulfil this principle.

\section{Methodology\label{section:methodology}}

\subsection{Theory and Practice\label{section:theory_practice}}

In this paper, we focus on batch-mode active learning (BMAL). The reason for conducting batch-mode active learning is to acquire more samples at one query step for the oracle to label them, thus reducing the frequency of retraining the model in case the model training is costly. In the following, we propose a maximum risk upper bound reduction theorem for CEE with active learning, and the main proof of the general theorem -- Theorem \ref{theorem:1} is provided in Appendix \ref{appendix:theorem_1}, followed by the sub-proofs for each of the convergence analysis in Appendix \ref{proof:convergence_var} and \ref{proof:convergence_wass}.

\begin{theorem}\label{theorem:1}
    With budget $\mathcal{M}$, the maximum risk upper bound reduction $\Delta_{\mathcal{B}_{\text{overall}}}$ is achieved at the termination of the entire I data query steps given that the generalization risk upper bound shrinkage $\Delta_{\mathcal{B}_{i}}$ is maximized at each query step $\forall i$, i.e.:
    \begin{equation}
        \begin{split}
            \argmax_{\Tilde{\mathcal{D}}_{\text{overall}}}~&\Delta_{\mathcal{B}_{\text{overall}}}=\bigcup_{i=1}^{I}\argmax_{\Tilde{\mathcal{D}}_{i}}\Delta_{\mathcal{B}_{i}}\\
            &s.t.~ |\Tilde{\mathcal{D}}_{\text{overall}}|\leq\mathcal{M},
        \end{split}
    \end{equation}
    where $\Tilde{\mathcal{D}}_{\text{overall}}$ is the overall acquired data, $\Tilde{\mathcal{D}}_{i}$ is the acquired batch at $i$-th query step, and $\Delta_{\mathcal{B}_{i}}=\sum_{t\in\{0,1\}}\Delta_{\text{Var}_{i}}^{t}(\Tilde{\mathcal{D}}_{i})+C_{\phi}\Delta_{\text{IPM}_{i}}(\Tilde{\mathcal{D}}_{i})$. The convergence rate of the risk upper bound has the following guaranteed behaviours under certain circumstances:
    \begin{enumerate}[label=\roman*)] 
    \item When variance reduction $\sum_{t\in\{0,1\}}\Delta_{\text{Var}_{i}}^{t}(\Tilde{\mathcal{D}}_{i})$ becomes the dominant part of the risk upper bound, the risk convergence is lower-bounded by $\Omega(\beta^{i})$ with constant $\beta\in[0,1)$.
    \item While, with dominant constant $C_{\phi}$, the risk convergence is upper-bounded by $\mathcal{O}(\frac{1}{i+\gamma_{0}})$ with constant $\gamma_{0}\in\mathbb{R}^{+}$.
    \end{enumerate}
\end{theorem}

\begin{remark}
    The bound shrinkage $\Delta_{\mathcal{B}_{i}}$ at $i$-th query step constitutes the variance difference and the distributional discrepancy difference, solely focusing on the reduction of one term would not contribute to optimal risk reduction for active causal effect estimation, while a proper combination of such two terms can lead to the optimal result.\label{rmk:implication}
\end{remark}

\textbf{Toy Dataset}: We design a fully synthetic 1-dimensional toy dataset and conduct experiments on four kinds of acquisition functions, i.e., Random, Uncertainty, Symmetry and Uncertainty + Symmetry (our proposed method MACAL), to illustrate the importance of considering both model variance and data distributional discrepancy reduction in each query step. The simulation of the toy dataset are described in Appendix \ref{appendix:toy_dataset}. 

In Figure \ref{fig:toy_dataset}, we present the empirical evaluation of the four methods in terms of the PEHE -- $\sqrt{\epsilon_{\text{PEHE}}}$, Wasserstein distance, and model variance respectively. As shown in Figure \ref{fig:wass} and Figure \ref{fig:var}, the Symmetry and Uncertainty acquisition strategies achieve the fastest reduction solely in distributional discrepancy and model variance respectively among the three naive methods. When mapping their performance into the empirical risk reduction shown in Figure \ref{fig:epsilon}, Symmetry has faster empirical risk reduction than Uncertainty in the early stage, but it saturates due to the incapability of capturing the informative uncertain samples in the late stage. While MACAL combines both aspects, it brings consistent and significant risk reduction before convergence, and it is the quickest one to achieve the lowest risk.

\subsection{Algorithm}

According to Theorem \ref{theorem:1}, we aim to optimize the upper bound shrinkage at each query step, additionally, it is well recognized that in BMAL, sample diversity in the acquired batch is crucial \cite{kirsch2019batchbald, jesson2021causal},  thus we modify the shrinkage without loss of generality to cater the BMAL and aim to maximize the following at $i$-th query step: 
\begin{equation}
\begin{split}
    \Tilde{\mathcal{D}}^{*}_{i} = \argmax_{\Tilde{\mathcal{D}}_{i}=\Tilde{\mathcal{D}}^{t=1}_{i}\cup\Tilde{\mathcal{D}}^{t=0}_{i}\subseteq\mathcal{D}_{\text{pool}}}~&\sum_{t\in\{0,1\}}\mathbb{H}(\Tilde{\mathcal{D}}^{t}_{i})\cdot\Delta_{\text{Var}_{i}}^{t}(\Tilde{\mathcal{D}}^{t}_{i})+\\
    &
    C_{\phi}\Delta_{\text{IPM}_{i}}(\Tilde{\mathcal{D}}_{i})\cdot\prod_{t\in\{0,1\}}\mathbb{H}(\Tilde{\mathcal{D}}^{t}_{i}),
    \label{eq:optimization}
\end{split}
\end{equation} where $\mathbb{H}(\cdot)$ measures the entropy of the set, and the union of the batches for each of the treatment groups, $\Tilde{\mathcal{D}}^{t=1}_{i}$ and $\Tilde{\mathcal{D}}^{t=0}_{i}$, renders the acquired batch $\Tilde{\mathcal{D}}_{i}$ at $i$-th query step.

Note, that the optimization in (\ref{eq:optimization}) is a combinatorial problem. For example, to label $|\Tilde{\mathcal{D}}_{i}|=N_{i}$ samples out of the pool samples $|\mathcal{D}_{\text{pool}}|=N_{\text{pool}}$ at a time, we face a combinatorial search space which takes $\mathcal{O}(\frac{N_{\text{pool}}!}{N_{i}!(N_{\text{pool}}-N_{i})!})$ time to get the optimum. The brute-force suffers from such time complexity is prohibitive as $N_{\text{pool}}$ goes up given fixed $1\ll N_{i}\ll N_{\text{pool}}$. Thus, instead of leveraging the prototype criterion in (\ref{eq:optimization}), we propose a model agnostic method to approximate the terms in (\ref{eq:optimization}) to reduce the NP-hard problem to one that can be solved in polynomial time. In the following, we analyse the optimization objective in (\ref{eq:optimization}) separately and combine them to conquer afterwards. 

\textbf{Diversity.} To deal with the diversity term $\mathbb{H}(\Tilde{\mathcal{D}}^{t}_{i})$ with combinatorial nature, the key step here is to select the sample which is most distinguished from the acquired data in the batch iteratively one at a time. We use the Euclidean distance $d(\cdot,\cdot)$ to measure the similarity between two points for label acquisition. For a batch selection on the treatment group $t$, we do $|\Tilde{\mathcal{D}}^{t}_{i}|$ times iteratively. Since $|\Tilde{\mathcal{D}}^{t}_{i}|\ll N^{t}_{\text{pool}}$, the time complexity of the batch acquisition is capped by $\mathcal{O}(N^{t}_{\text{pool}})$.

\begin{lemma}\label{lemma:convergence_var}
    Given the variance counted by Gaussian process regression model $f^{t}$ on treatment group $t$, by acquiring the most uncertain samples that have the maximum predictive variance $\sigma^{2}_{f^{t}}$, the slowest convergence rate of the model variance is lower-bounded by $\Omega(\beta^{i})$, where $0\leq\beta<1$.
\end{lemma}

\textbf{Uncertainty.} To gain the maximum variance reduction over the sample space $\mathcal{X}$, labelling the most uncertain sample gives the highest variance reduction and the model variance can converge as depicted in Lemma \ref{lemma:convergence_var}, where the proof is provided in Appendix \ref{proof:convergence_var}.

\begin{figure*}[ht!]
    \centering
    \includegraphics[scale=0.9]{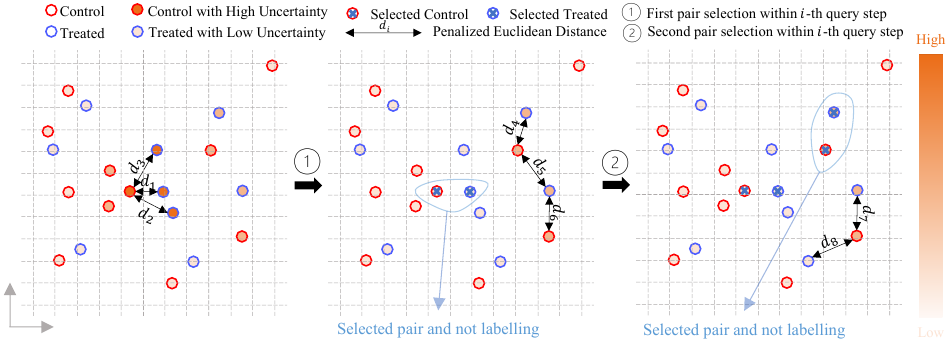}
    \vspace{-0.3cm}
    \caption{At $i$-th query step, the entire batch selection is divided into several pairs selection. Left: The most uncertain points are the candidates for selection, and the pair score is penalized by the distance between two points; Mid: The two most uncertain and closest points are selected (not labelled yet) and objectively bringing down other unlabelled points' uncertainty in proximity; Right: Select the next pair recursively until the batch is filled up.}
    \label{fig:selection_process}
\end{figure*}

However, the variance term is model-dependent, after one sample point is added into the acquired batch $\Tilde{\mathcal{D}}^{t}_{i}$, the model should ideally get retrained and update its confidence on the pool set for the next selection to maintain the batch diversity $\mathbb{H}(\Tilde{\mathcal{D}}^{t}_{i})$ and variance reduction $\Delta_{\text{Var}_{i}}^{t}(\Tilde{\mathcal{D}}^{t}_{i})$ at a high level, but retraining $|\Tilde{\mathcal{D}}^{t}_{i}|$ times is not cost-efficient. To overcome this issue, we perform an approximation for the variance term to leave it model-independent since we care less about its exact value but more about its relative magnitude for comparison, i.e., acquiring the most distinct point from the pool set $\mathcal{D}^{t}_{\text{pool}}$ with the highest Euclidean distance from its nearest neighbour in the training set (with previously acquired unlabelled samples integrated). Thus, we firstly calculate the minimum distance between every candidate sample from pool set $\mathcal{D}^{t}_{\text{pool}}$ and the acquired sample from training set $\mathcal{D}^{t}_{\text{train}}$, which results in a set of distance $\{d^{\text{min}}_{i}\}$ of size $N^{t}_{\text{pool}}$. This set of values embeds both diversity and uncertainty in terms of Euclidean distance since the higher the minimum value, the more distinct the candidate point from the training set. Subsequently, the maximum one is selected from the set as follows:
\begin{equation}
        (\Tilde{x}^{t}_{i})^{*} = \text{arg max}_{\Tilde{x}^{t}_{i}\in\mathcal{D}^{t}_{\text{pool}}} \text{min}_{x'_{i}\in\mathcal{D}^{t}_{\text{train}}}d(\Tilde{x}^{t}_{i}, x'_{i}).
        \label{eq:uncertainty_diversity}
\end{equation} Note, that the unlabelled sample $(\Tilde{x}^{t}_{i})^{*}$ should be added into the training set before the next selection starts because a similar or repetitive sample is redundant. It should also be noted that (\ref{eq:uncertainty_diversity}) does not necessarily return the most uncertain sample, but the larger distance from its in-sample nearest neighbour is positively correlated with higher uncertainty. Thus, we take this trade-off to approximately approach the NP-hard term $\mathbb{H}(\Tilde{\mathcal{D}}^{t}_{i})\Delta_{\text{Var}_{i}}^{t}(\Tilde{\mathcal{D}}^{t}_{i})$ by iterating $|\Tilde{\mathcal{D}}^{t}_{i}|$ times. By considering the enumeration through both of the treatment groups, the time complexity of one query step for (\ref{eq:uncertainty_diversity}) is $\mathcal{O}(\max\{(N^{t=1}_{\text{pool}})^2,(N^{t=0}_{\text{pool}})^2\})$, which is obviously capped by $\mathcal{O}(N_{\text{pool}}^2)$.

\textbf{Distributional Discrepancy.} 
In light of the above-mentioned iterative acquisition within a single query step, we can apply a similar mentality to avoid the combinatorial nature of the second term in (\ref{eq:optimization}) for the batch acquisition. To obtain a high-level reduction $\Delta_{\text{IPM}_{i}}(\Tilde{\mathcal{D}}_{i})$, an effective labelling in terms of reducing the imbalance would lead to symmetrical acquisition, namely labelling the identical sample from different treatment groups to make a pair. Thus the local distributional discrepancy (within the acquired dataset $\Tilde{\mathcal{D}}_{i}$) becomes zero if identical (or very similar) samples can be collected to counteract the violation of positivity locally. Subsequently, the accumulated global distributional discrepancy gains an asymptotic behaviour approaching zero as more symmetrical (or similar) samples are added into the training set (as also empirically observed in Figure \ref{fig:wass}). We propose Lemma \ref{lemma:convergence_wass} concerning the convergence rate with the proof provided in Appendix \ref{proof:convergence_wass}:

\begin{lemma}\label{lemma:convergence_wass}
    Given two empirical distributions $p^{t=1}$ and $p^{t=0}$ for different treatment groups, the distributional discrepancy given by 1-Wasserstein distance $W_{1}(p^{t=1},p^{t=0})$ has a convergence rate of $\mathcal{O}(\frac{1}{i+\gamma_{0}})$ if the identical samples from two groups can always be found throughout the query steps.
\end{lemma}

Note, that under the deep learning framework \cite{shalit2017estimating, louizos2017causal}, the distributional discrepancy is calculated over the latent space with $\phi:\mathcal{X}\rightarrow\mathcal{Z}$, where $\phi$ is one-on-one mapping. Thus, samples that are identical or similar in the original space $\mathcal{X}$ should still preserve their semantic manifold in the latent space, such that labelling similar points over the raw space is the same as the one in latent space. To label a pair, we calculate the Euclidean distance between each sample from different treatment groups, and the optimal pair is selected with the smallest distance. This selection is flexible since it does not constrain the identical acquisition but the most similar pair. Thus, the time complexity to get the optimal pair takes $\mathcal{O}(N^{t=1}_{\text{pool}}\cdot N^{t=0}_{\text{pool}})$, which is capped by $\mathcal{O}(N_{\text{pool}}^2)$. At a single iteration, we acquire a pair as follows:
\begin{equation}
        \{(\Tilde{x}^{t=1}_{i},\Tilde{x}^{t=0}_{j})^{*}\} =\argmax_{\Tilde{x}^{t=1}_{i}\in\mathcal{D}^{t=1}_{\text{pool}},\Tilde{x}^{t=0}_{j}\in\mathcal{D}^{t=0}_{\text{pool}}} -d(\Tilde{x}^{t=1}_{i},\Tilde{x}^{t=0}_{j}).
        \label{eq:discrepancy}
\end{equation}

\textbf{MACAL.} In this paper, we combine the optimization in (\ref{eq:uncertainty_diversity}) and (\ref{eq:discrepancy}) altogether, a pair $(\Tilde{x}^{t=1},\Tilde{x}^{t=0})^{*}$ that maximizes the following term is selected from both treatment groups:
\begin{equation}
\begin{split}
    &(\Tilde{x}^{t=1},\Tilde{x}^{t=0})^{*} =\\
    &\argmax_{\Tilde{x}^{t=1}\in\mathcal{D}^{t=1}_{\text{pool}},\Tilde{x}^{t=0}\in\mathcal{D}^{t=0}_{\text{pool}}} \Sigma_{t\in\{0,1\}}\text{min}_{x'\in\mathcal{D}^{t}_{\text{train}}}d(\Tilde{x}^{t}, x')-\alpha\cdot d(\Tilde{x}^{t=1},\Tilde{x}^{t=0}).\label{eq:final_opt}
\end{split}
\end{equation} Note, that we require the batch size $|\Tilde{\mathcal{D}}_{i}|$ to be an even number, which is quite easy to satisfy. Thus, we do $|\Tilde{\mathcal{D}}_{i}|/2$ efficient iterations to obtain the batch in one query step. Also, we set the coefficient $\alpha$ to penalize the acquisition that violates the positivity assumption, this regularization constant is taken as a hyperparameter and its impact is further discussed in Appendix \ref{section:ablations} since it is hardly possible to compute the exact value for the bounded constant $C_{\phi}$ \cite{shalit2017estimating}. We visualize the dynamic selection process in Figure \ref{fig:selection_process} according to the selection criterion defined in (\ref{eq:final_opt}). For the case where one of the treatment pool sets is exhausted, the acquisition is down to only one side by simply updating the $\alpha=0$ since no counterpart can be acquired anymore. The full algorithm's pseudo code is provided in Appendix \ref{appendix:algorithm}.


In summary, MACAL promotes variance and discrepancy reduction by labelling diverse uncertain samples and it penalizes the dissimilarity of the paired samples via the Euclidean distance. To label a batch of samples of size $|\Tilde{\mathcal{D}}_{i}|$, we take $\mathcal{O}(N_{\text{pool}}^2)$ time complexity to obtain the optimal batch at each query step, which is significantly lower than the cost to solve the NP-hard problem by brute-force.

\section{Related Work}

\textbf{Active Learning.} The history of active learning can be traced back to over a century ago \cite{smith1918standard}, with such a long time progress till nowadays, it has become a frontier research branch of machine learning \cite{settles2009active, ren2021survey, zhan2022comparative}. The core of active learning is to make model performance cost-efficient, i.e., obtaining relatively low model risk with as few labelled samples as possible. Generally, the active learning approach can be portioned into three scenarios: query synthesis \cite{wang2015active}, stream-based \cite{fujii2016budgeted} and pool-based \cite{wu2018pool}. In this paper, we focus on pool-based active learning, especially on regression problems, where the uncertainty-based sampling \cite{gal2017deep}, density-based querying \cite{sener2018active}, and hybrid strategies \cite{ash2019deep} are three key acquisition methods under such setting. For instance, the information-theoretic based Bayesian Active Learning by Disagreement (BALD) \cite{gal2017deep} leverages the epistemic uncertainty to acquire unlabelled samples. Core-Set \cite{sener2018active} selects the greatest distance to its nearest neighbour in the hidden space. ACS-FW \cite{pinsler2019bayesian} is a hybrid between Core-Set and Bayesian approaches which balances the sample diversity and uncertainty in batch-mode acquisition. Note, that although the general active learning methods are not designated for CEE, benchmarking on these methods provides insightful results. 

\textbf{Causal Effect Estimation with Active Learning.} Some noticeable advances have been made in this area of research. \cite{sundin2019active} approximates the decision-making reliability via the estimated S-type error rate (the probability of the model inferring the sign of the treatment effect wrong) of the prediction model,  which is then used as the querying criterion. However, \cite{sundin2019active} focuses on estimating the correct sign of the treatment effect, which is different from the risk metric in our setting. For works focusing on the same risk metric, QHTE \cite{qin2021budgeted} integrates the Core-Set concept \cite{tsang2005core,sener2018active} to form a theoretical framework, for which a theory-based optimization is proposed. However, the QHTE relaxes the tightness of the bound given by Shalit et. al. \cite{shalit2017estimating}, where a covering radius $r=0$ from the relaxed bound cannot even obtain the original
tightness, while, we propose a more informative theory which does not undermine the tightness of the original bound. More importantly, QHTE does not consider the distribution imbalance during sample acquisition, which is what our method can prominently distinguish from. To fix the acquisition imbalance issue, Causal-BALD \cite{jesson2021causal} cut into the problem from the information-theoretic perspective, its most representative criterion $\mu\rho$BALD accounts for the overlapping by especially scaling the criterion with the inverse of counterfactual variance, leaving the acquisition toward pairing up similar samples if its counterfactual were missing, which is a non-trivial improvement from its base - $\mu$BALD (an uncertainty-based softmax-BALD method \cite{kirsch2021stochastic}). Our proposed method is distinguished from \cite{jesson2021causal} in several points. First, Causal-BALD relies on model-dependent variance estimation, e.g., the deep kernel learning model \cite{wilson2016deep,van2021feature}, while our criterion is model agnostic. Second, we cut into the problem from an intriguing perspective to maximize the upper bound shrinkage at each query step instead of the mutual information perspective. Finally, taking the inverse of the counterfactual variance is undesirable and causes numerical instability, while our method leverages the simple but effective addition and subtraction operations to form the label acquisition criterion. It is also noted that some existing literature \cite{deng2011active, addanki2022sample} uses active learning to take the initiative for efficient experimental trials design, i.e., the pool set does not include the treatment information but enforcing treatment after sample acquisition, which is different from our setting.

\begin{table}[t!]
\centering
\caption{Summary of the Acquisition Setup and Testing}\label{table:acquisition_summary}
\vspace{-0.2cm}
\begin{tabularx}{.48\textwidth}{X *{6}{>{\centering\arraybackslash}X}} 
\toprule
\multicolumn{1}{l}{Dataset} & \multicolumn{1}{c}{Start S.} & \multicolumn{1}{c}{Step S.} & \multicolumn{1}{c}{Queries} & \multicolumn{1}{c}{Pool S.} & \multicolumn{1}{c}{Val S.} & \multicolumn{1}{c}{Test S.} \\
\midrule
IHDP &10&10&46&470&75&202\\
IBM &50&50&50&9540&3180&6250\\
CMNIST &50&50&50&31500&10500&18000\\
\bottomrule
\end{tabularx}
\vspace{-0.4cm}
\end{table}

\section{Experiments\label{section:experiments}}

Due to the unique nature of the CEE tasks, the counterfactual effect is hardly observed in the real world. Thus, in this paper, we take the common practice to utilise the fully-synthetic and semi-synthetic datasets for algorithm evaluations.

\textbf{Dataset:} \textbf{IHDP} \cite{hill2011bayesian} - an imbalanced dataset based on 747 samples (among them 139 with treatment status $t=1$ and 608 with status $t=0$) and 25 covariates, with 100 times simulated treatment outcomes by \cite{hill2011bayesian}. \textbf{IBM} \cite{shimoni2018benchmarking} - uses a cohort of 100k individuals from the publicly available Linked Births and Infant Deaths Database with 177 real-world covariates. Each original simulation randomly takes 25k out of such 100k samples and the potential outcomes are simulated 10 times according to \cite{shimoni2018benchmarking}, we create more imbalance by omitting additional samples from treatment group $t=1$. \textbf{CMNIST} \cite{jesson2021quantifying} - is of size 60k adapted from MINIST \cite{lecun1998mnist} dataset. The input from CMNIST is the handwritten digit of size 28$\times$28, which is distinct from the previous tabular datasets. The potential outcomes are simulated 10 times and generated by projecting the digits into a 1-dimensional latent manifold as described in \cite{jesson2021quantifying}.

\textbf{Metric:} We use precision in estimation of heterogeneous effect (PEHE) \cite{shalit2017estimating}, a well-established metric with the empirical formulation: $\sqrt{\epsilon_{\text{PEHE}}}=\sqrt{\Sigma_{i=1}^{N}((y^{t=1}_{i}-y^{t=0}_{i})-\tau_{i})^{2}/N}$ for measuring the accuracy of the treatment effect estimation at the individual level. The lower the value of $\sqrt{\epsilon_{\text{PEHE}}}$, the better the performance.

\textbf{Baselines:} We set the Random method as the benchmark acquisition function, as this is the most naive method that selects the data purely at random. We also compare our proposed method against many SOTA baselines from the general AL research field, that is, BADGE \cite{ash2019deep}, BAIT \cite{ash2021gone}, and LCMD \cite{holzmuller2023framework}. We argue that a good comparison to these methods from the broad AL research is indispensable and this paper also fills the blank for such comparisons. Moreover, the most related work - QHTE \cite{qin2021budgeted}, and especially Causal-Bald \cite{jesson2021causal} is the designated algorithm proposed to deal with the active causal effect estimation. Thus, we compare three representative variants of the Causal-Bald, namely $\mu$BALD, $\rho$BALD, and $\mu\rho$BALD.

\textbf{Prediction Backbone:} \textbf{DUE-DNN}\cite{van2021feature}. It is one of the SOTA deep kernel learning frameworks with the multi-layer perceptron as the common feature extractor and two sparse Gaussian process regressions defined over the extracted latent features as the downstream estimators for different treatment groups' effect estimations. \textbf{DUE-CNN}\cite{van2021feature}. It is a variant of the DUE model especially catering for the image-as-input experiment. It has a similar structure as DUE-DNN besides the latent feature extractor being replaced by the convolutional neural network (CNN), e.g., the ResNet \cite{he2016deep} is embedded. The computation resources and hyperparameter selection are described in \ref{appendix:hyperparameters}.

\begin{figure}[h!]
  \centering
  \includegraphics[scale=0.22]{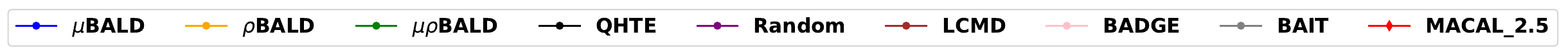}\\
  \subfigure[IHDP-CausalAL\label{fig:ihdp_causalal}]{\includegraphics[width=0.22\textwidth]{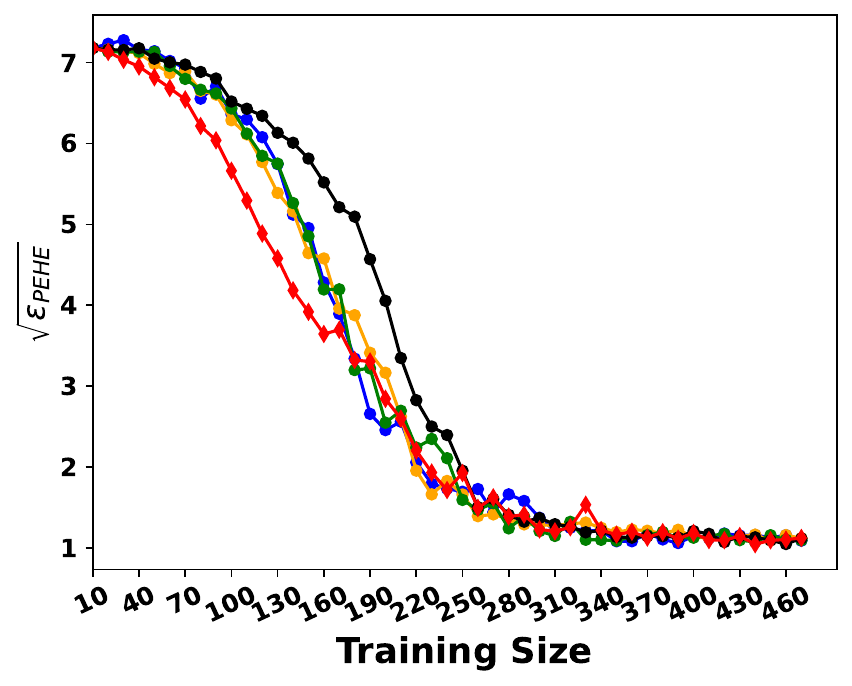}}
  \subfigure[IHDP-GeneralAL\label{fig:ihdp_generalal}]{\includegraphics[width=0.22\textwidth]{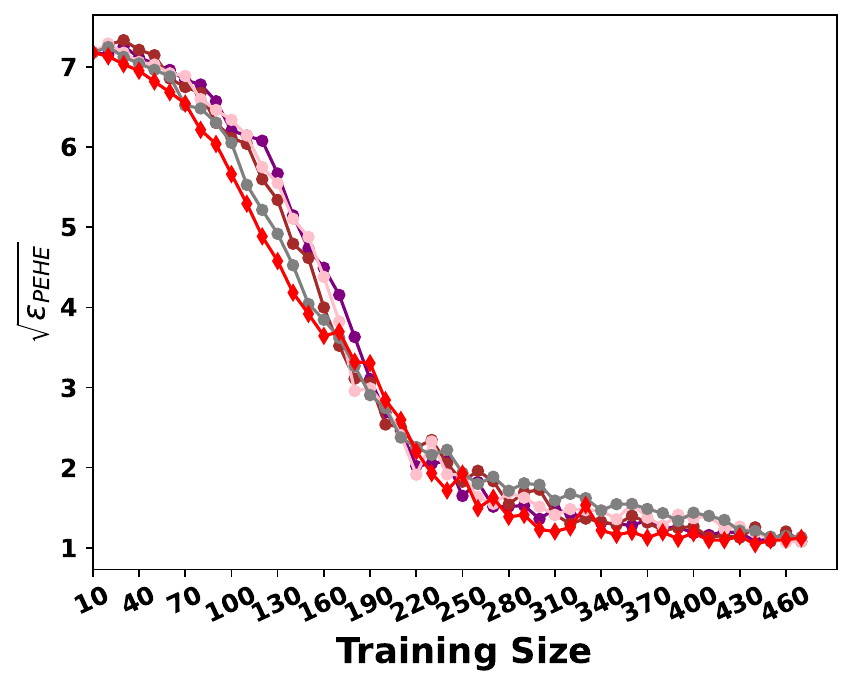}}
  \\
  \vspace{-0.2cm}
  \subfigure[IBM-CausalAL\label{fig:ibm_causalal}]{\includegraphics[width=0.22\textwidth]{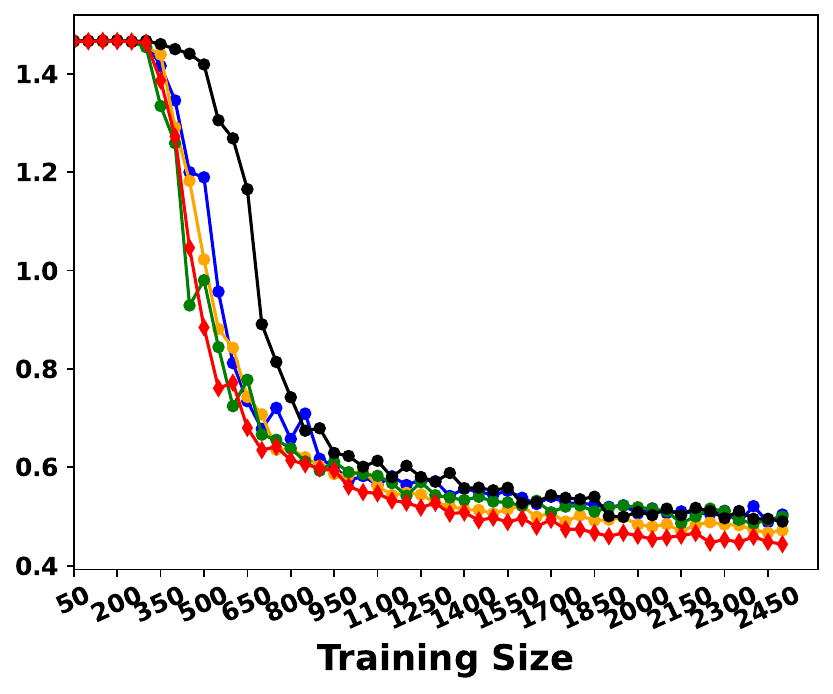}}
  \subfigure[IBM-GeneralAL\label{fig:ibm_generalal}]{\includegraphics[width=0.22\textwidth]{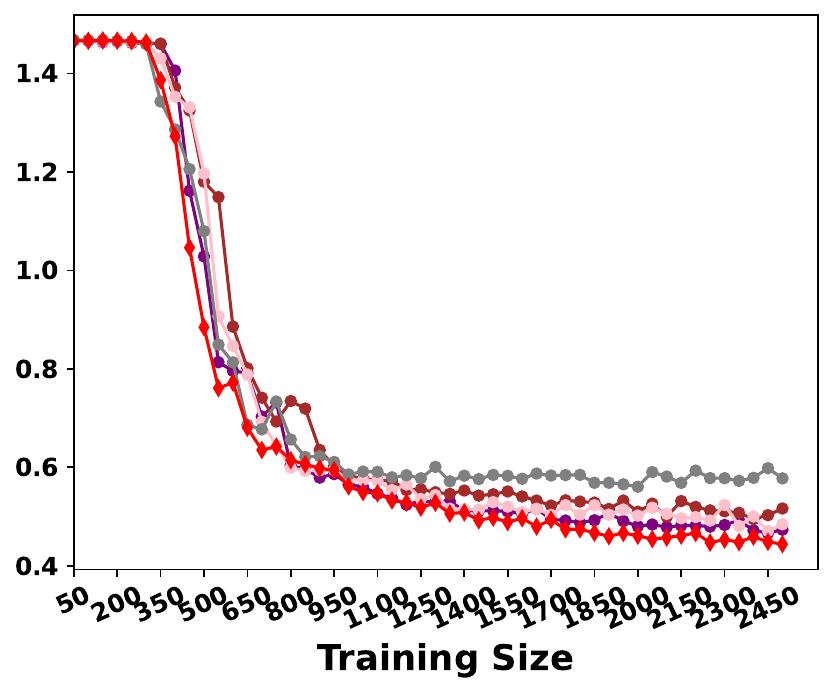}}\\
  \vspace{-0.2cm}
  \subfigure[CMNIST-CausalAL\label{fig:cmnist_cnn_causalal}]{\includegraphics[width=0.22\textwidth]{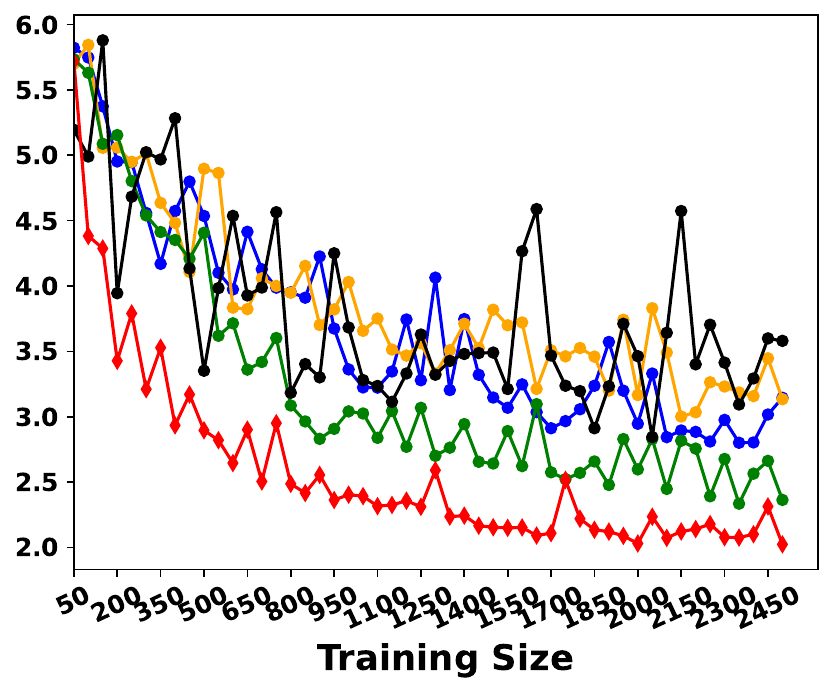}}
  \subfigure[CMNIST-GeneralAL\label{fig:cmnist_cnn_generalal}]{\includegraphics[width=0.22\textwidth]{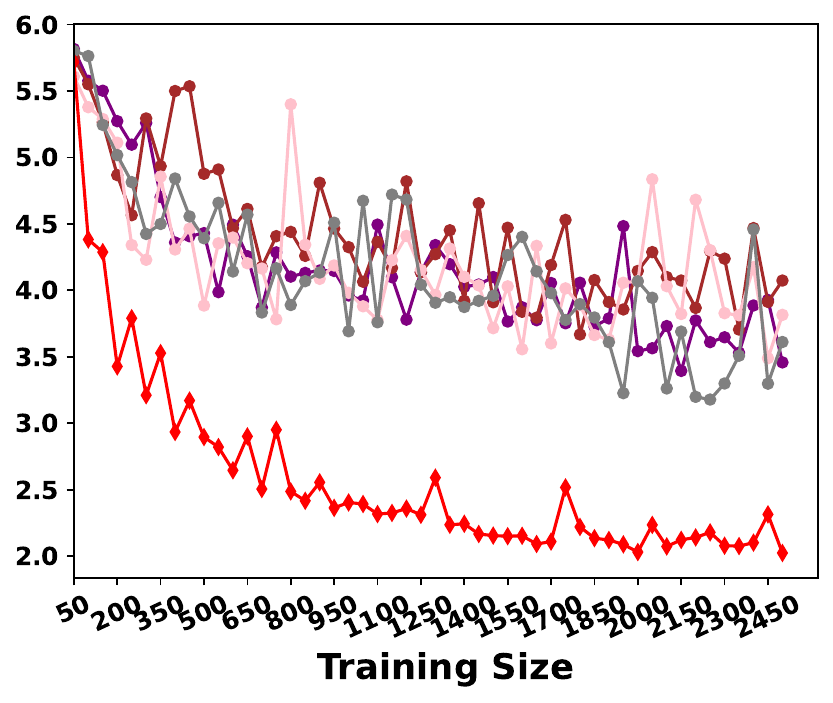}}
  \vspace{-0.2cm}
  \caption{Benchmarking with the available SOTAs on three datasets, i.e., IHDP (average with 100 simulations), IBM (average with 10 simulations) and CMNIST (average with 10 simulations). The first column concludes the comparisons against the baselines designated for active causal effect estimation. The second column shows the comparisons with the baselines from general active learning approaches. All of the results are given by the same downstream CEE model by DUE-DNN for IHDP \& IBM, and DUE-CNN for CMNIST.}
  \label{fig:pehe_evaluation}
  \vspace{-0.2cm}
\end{figure}    

\textbf{Acquisition setup:} We begin with a small Start Size (Start S.) to simulate the real-world scenario where only sparse labelled data can be accessed at the beginning. Then, a fixed Step Size (Step S.) is enforced at each query step, and the entire AL sessions (Queries), which consist of many query steps comes to an end when the model converges or the sample pool has been exhausted. The detail of the label acquisition setup is summarized in Table \ref{table:acquisition_summary}.

\begin{figure*}[ht!]
  \centering
  \subfigure[Random at Step 5 \label{fig:ihdp_random_5}]{\includegraphics[width=0.15\textwidth]{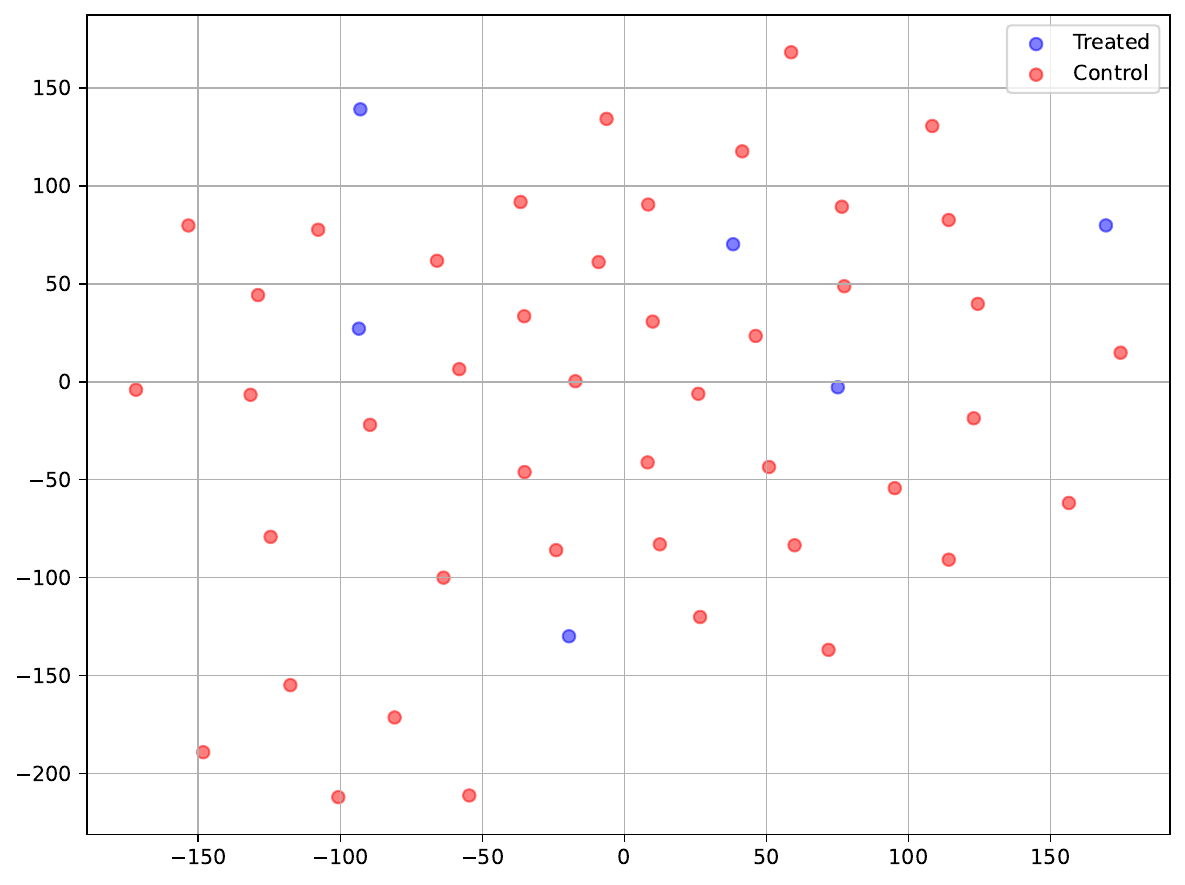}}
  \subfigure[Random at Step 15\label{fig:ihdp_random_15}]{\includegraphics[width=0.15\textwidth]{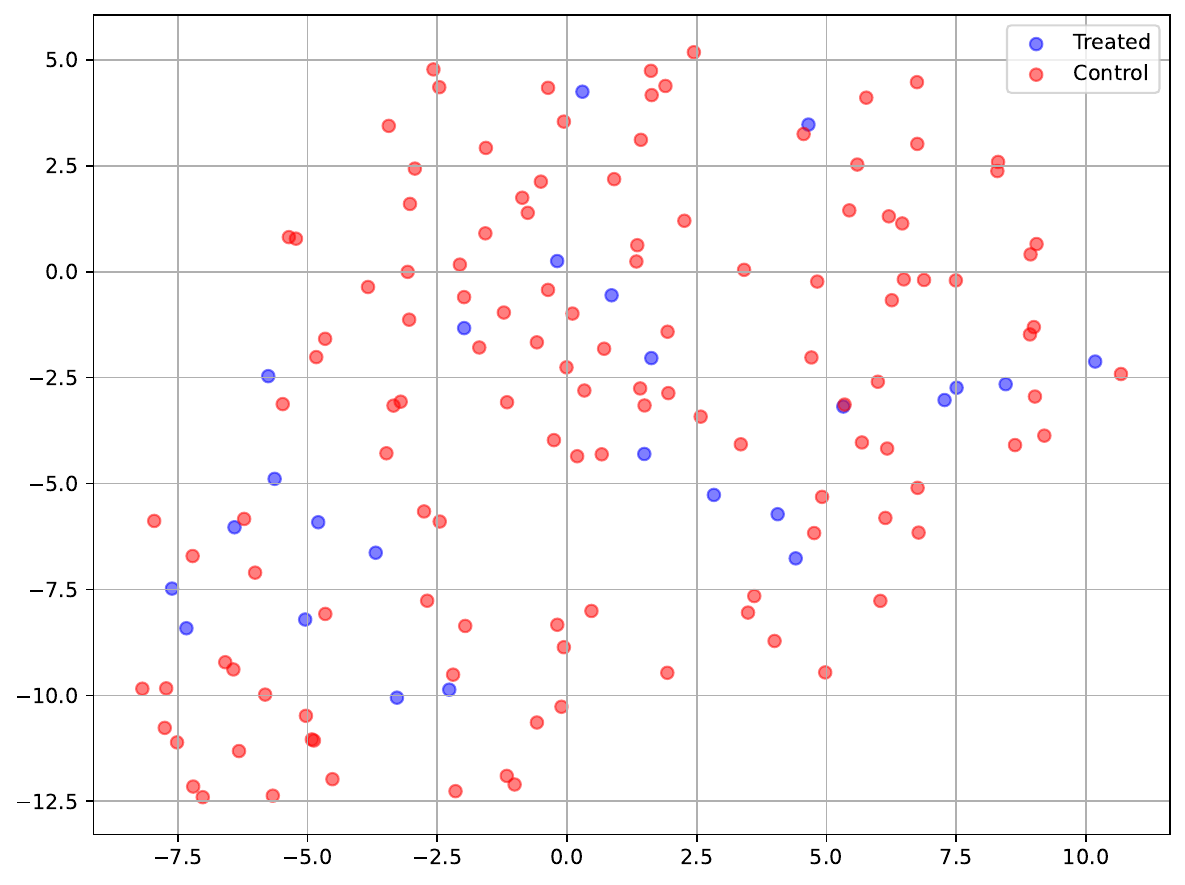}}
  \subfigure[Random at Step 10\label{fig:ibm_random_10}]{\includegraphics[width=0.15\textwidth]{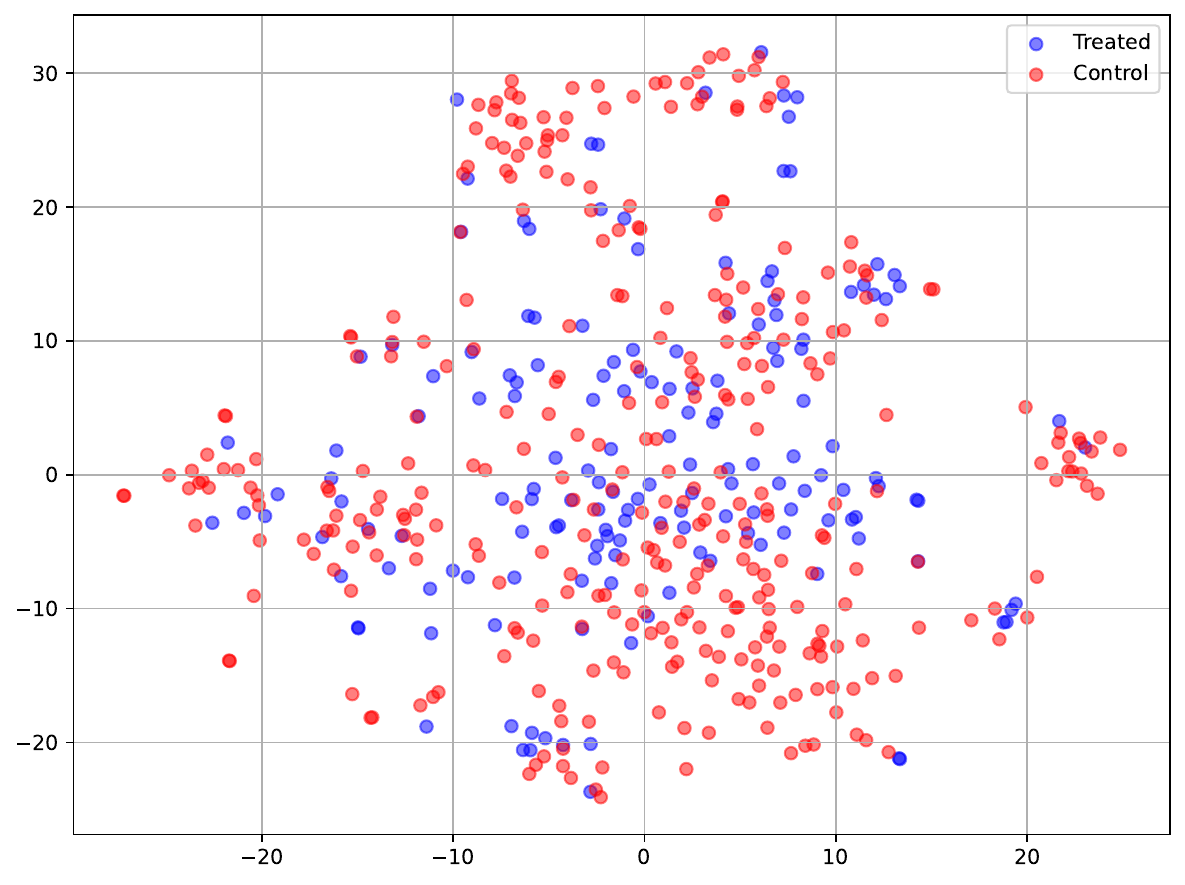}}
    \subfigure[Random at Step 30\label{fig:ibm_random_30}]{\includegraphics[width=0.15\textwidth]{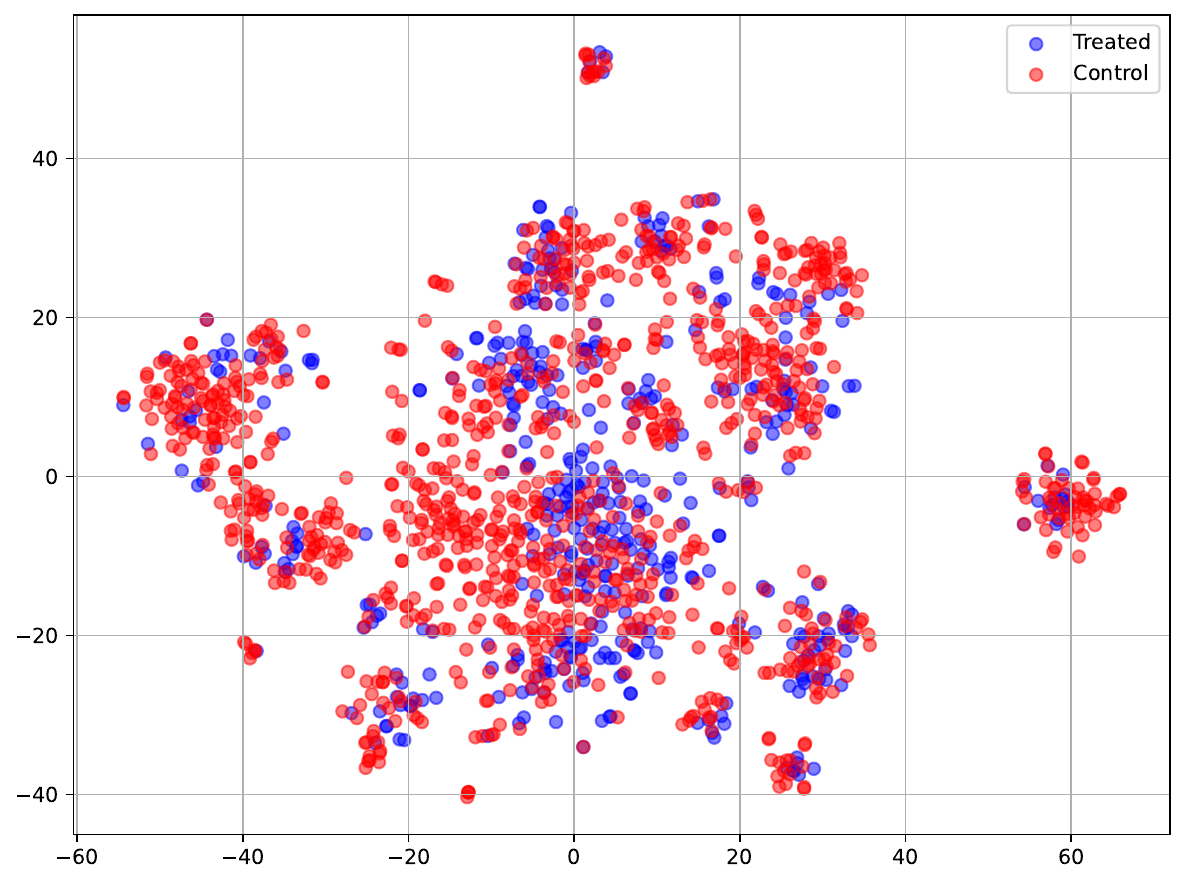}}
      \subfigure[Random at Step 10\label{fig:cmnist_random_10}]{\includegraphics[width=0.15\textwidth]{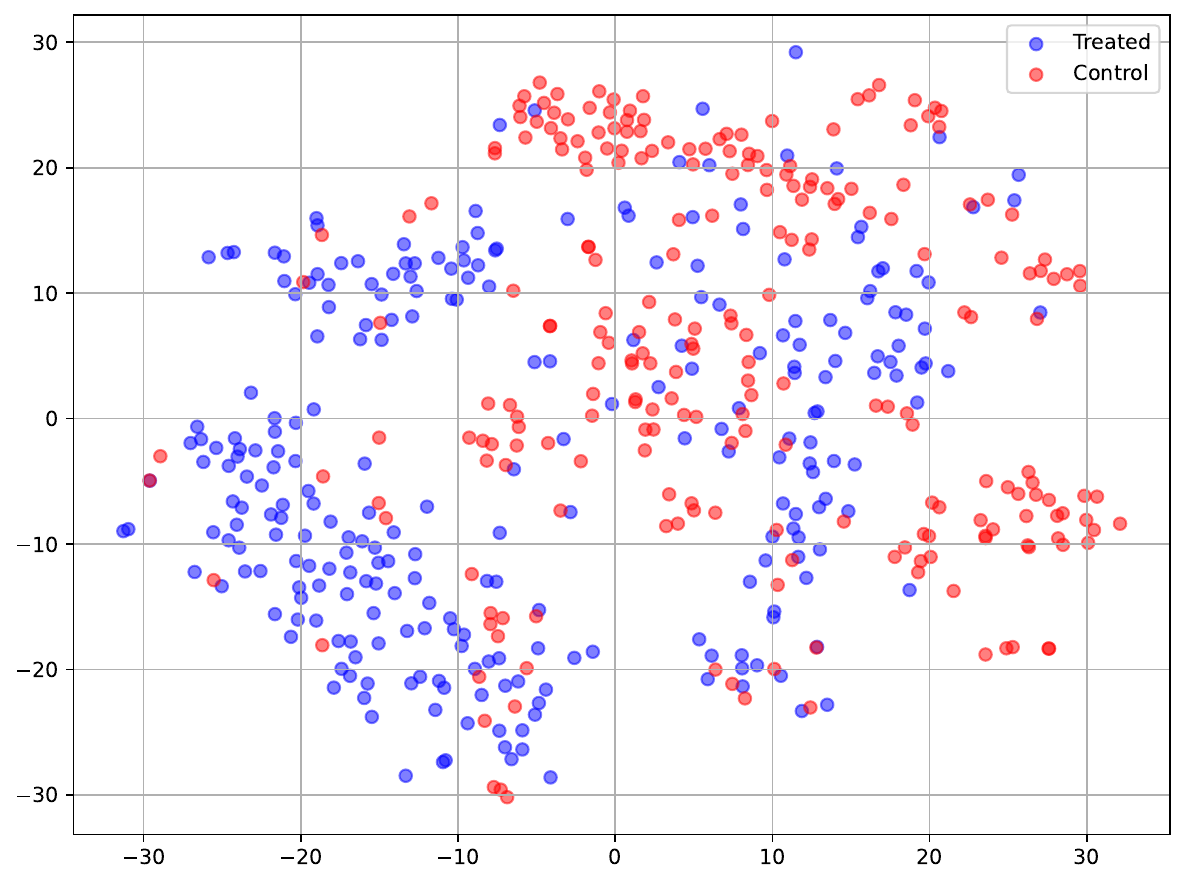}}
        \subfigure[Random at Step 30\label{fig:cmnist_random_30}]{\includegraphics[width=0.15\textwidth]{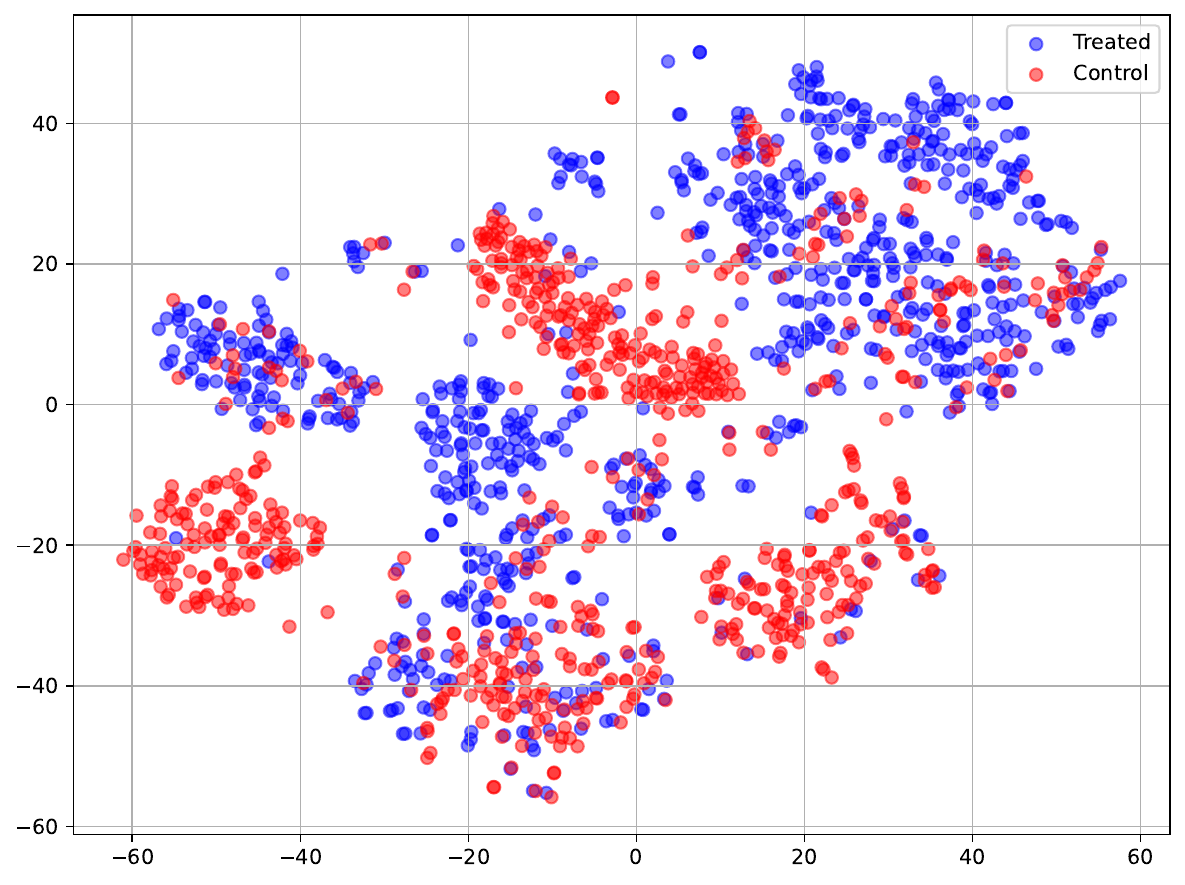}}\\
        \vspace{-0.2cm}
\subfigure[$\mu\rho$BALD at Step 5 \label{fig:ihdp_murho_5}]{\includegraphics[width=0.15\textwidth]{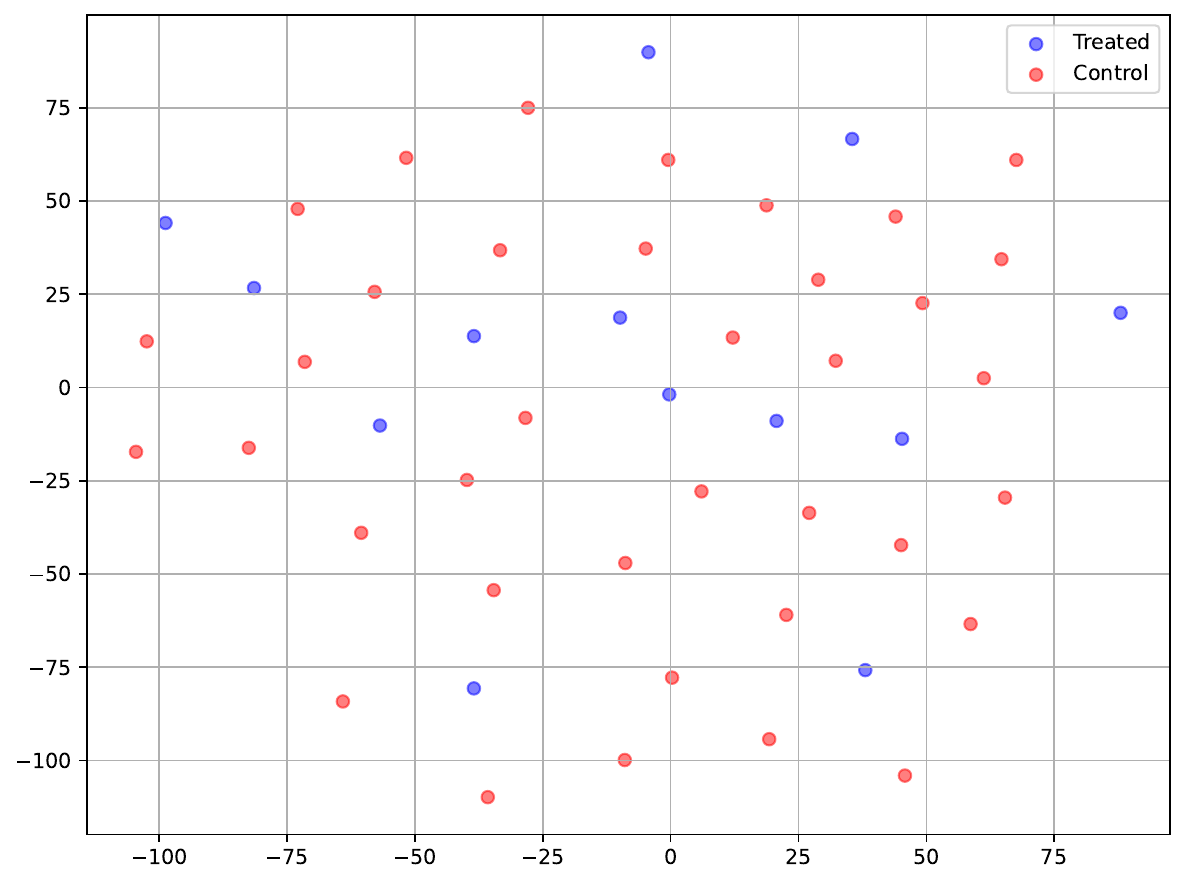}}
  \subfigure[$\mu\rho$BALD at Step 15\label{fig:ihdp_murho_15}]{\includegraphics[width=0.15\textwidth]{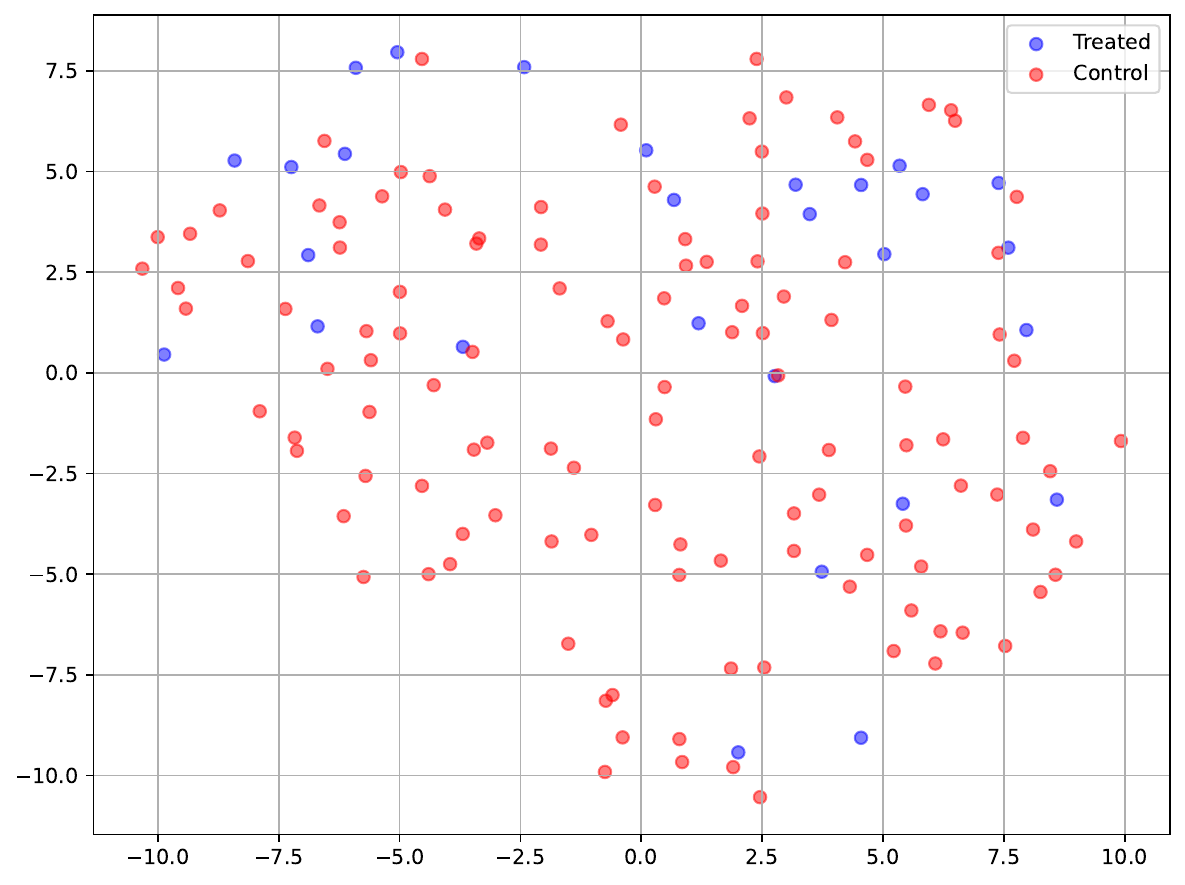}}
  \subfigure[$\mu\rho$BALD at Step 10\label{fig:ibm_murho_10}]{\includegraphics[width=0.15\textwidth]{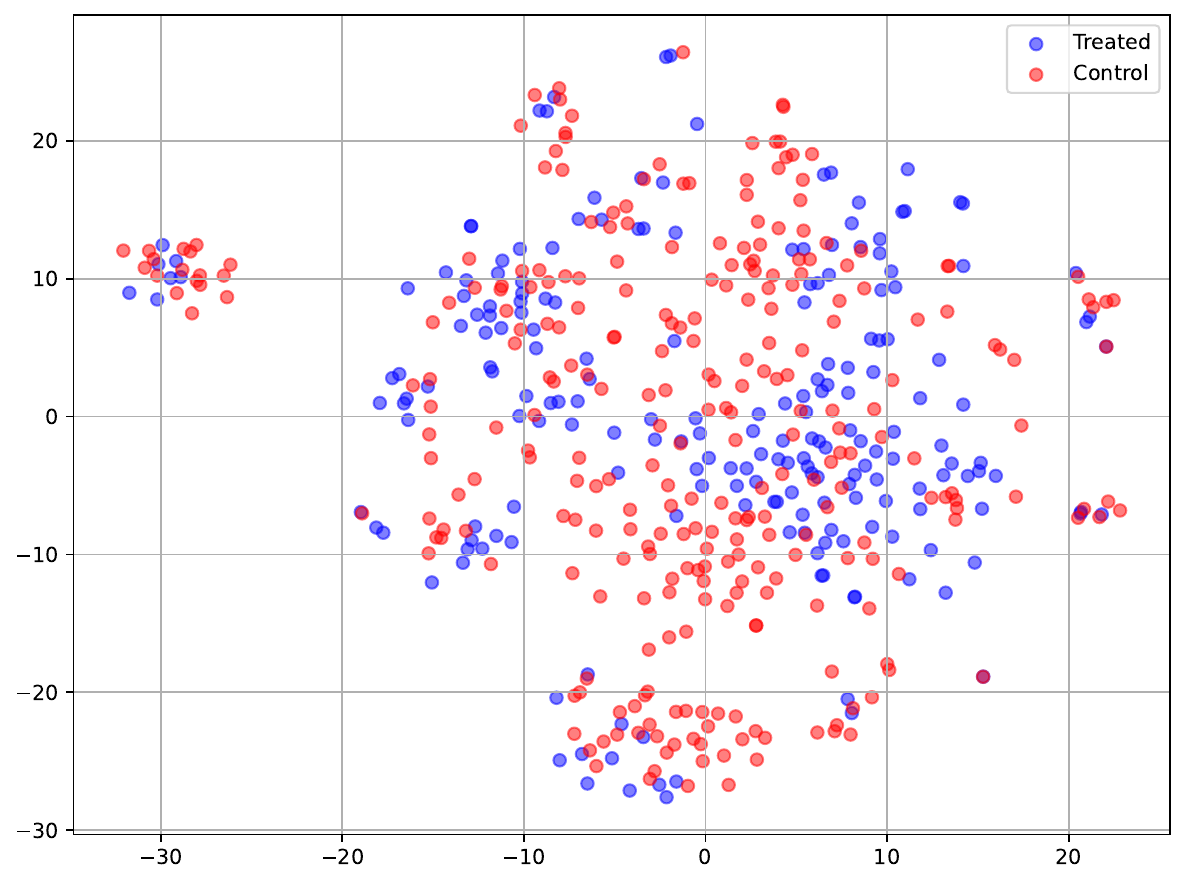}}
    \subfigure[$\mu\rho$BALD at Step 30\label{fig:ibm_murho_30}]{\includegraphics[width=0.15\textwidth]{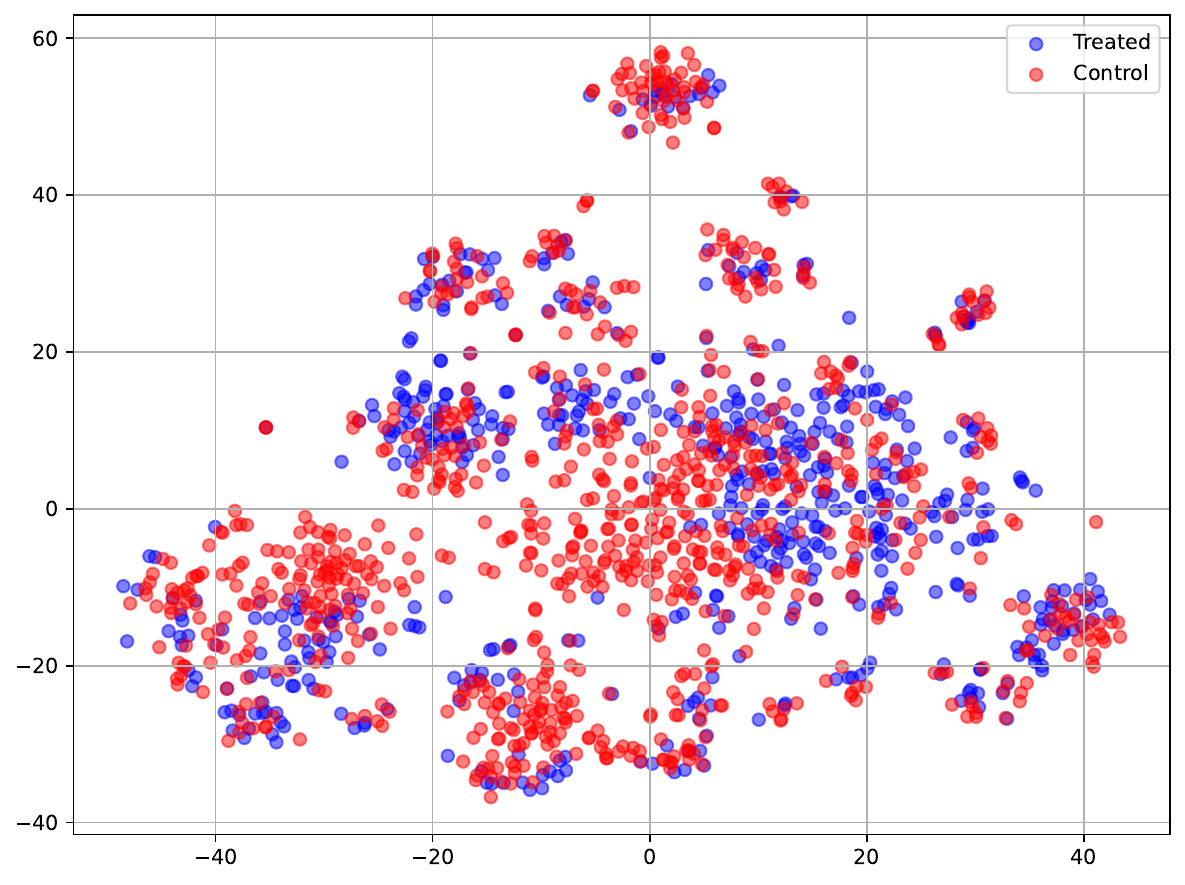}}
      \subfigure[$\mu\rho$BALD at Step 10\label{fig:cmnist_murho_10}]{\includegraphics[width=0.15\textwidth]{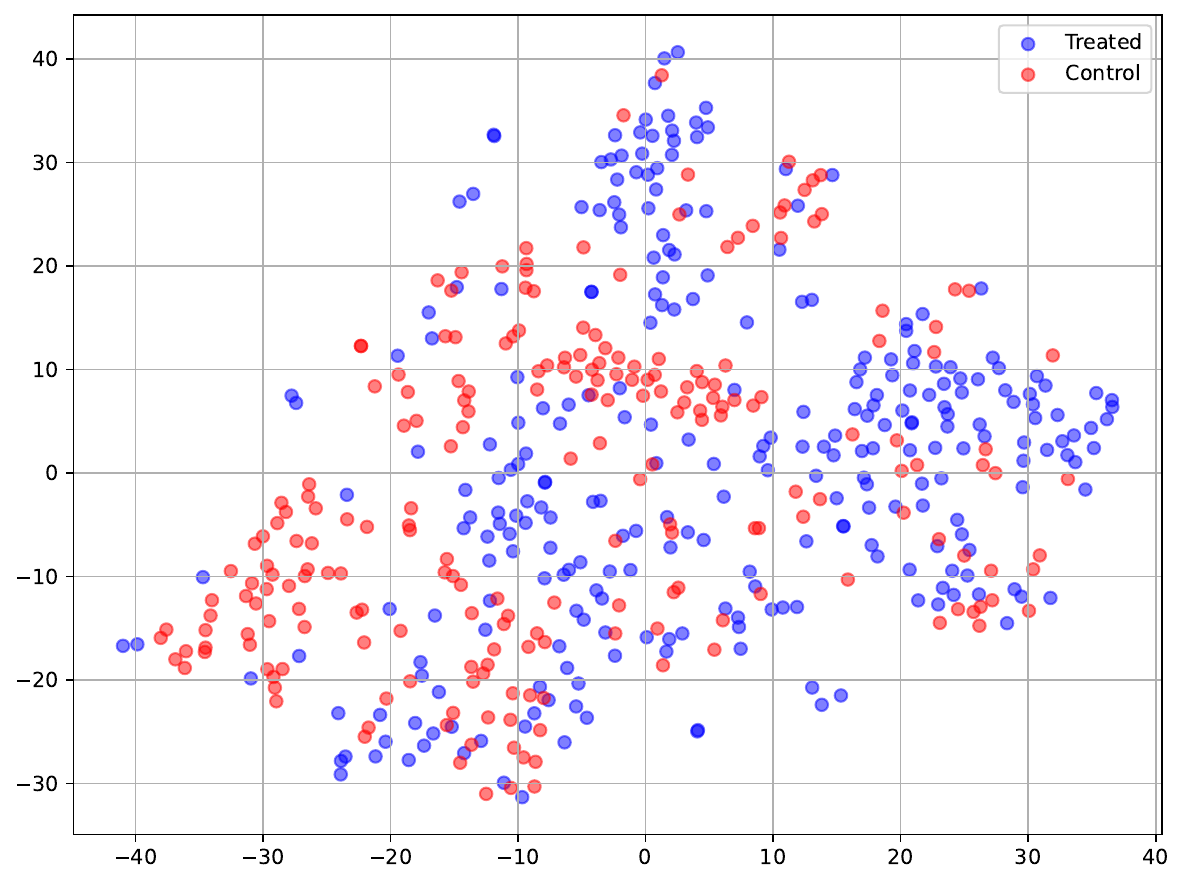}}
        \subfigure[$\mu\rho$BALD at Step 30\label{fig:cmnist_murho_30}]{\includegraphics[width=0.15\textwidth]{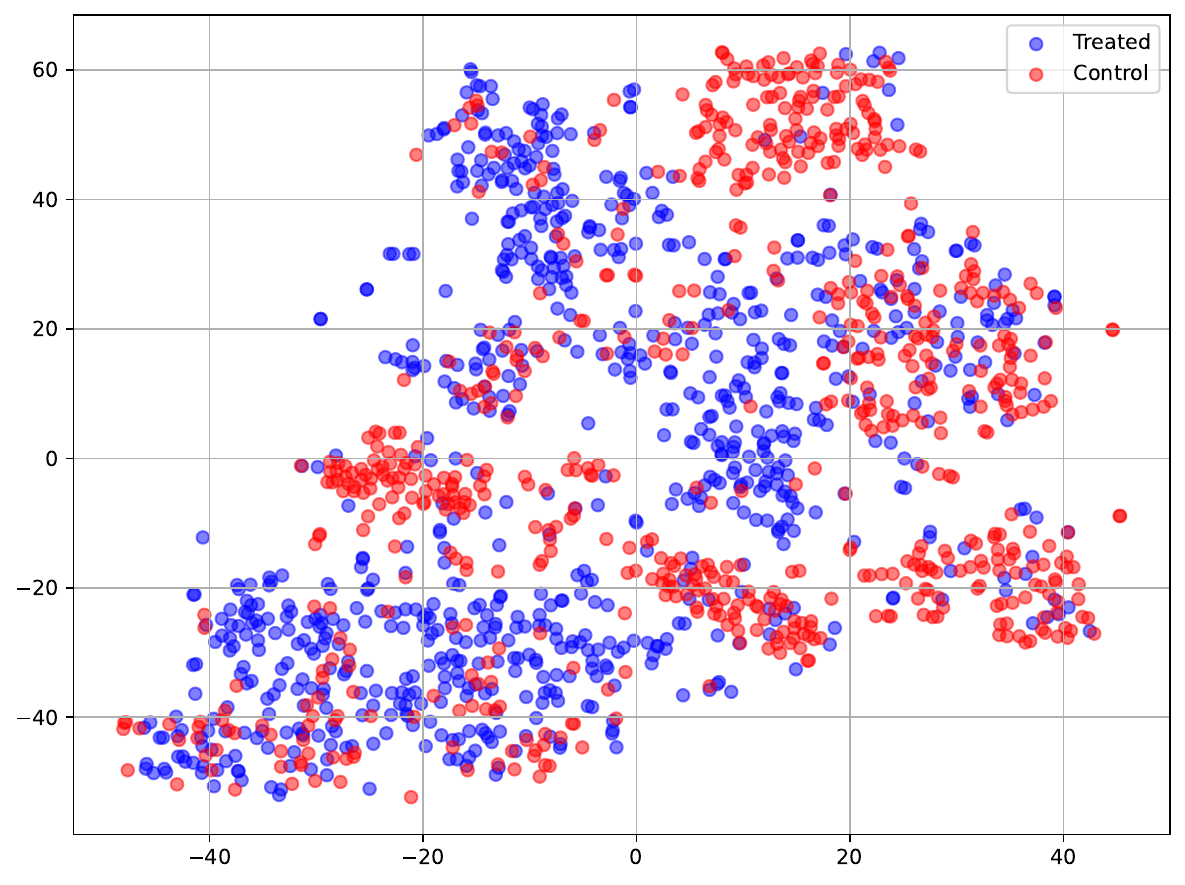}}\\
        \vspace{-0.2cm}
\subfigure[MACAL at Step 5 \label{fig:ihdp_macal_5}]{\includegraphics[width=0.15\textwidth]{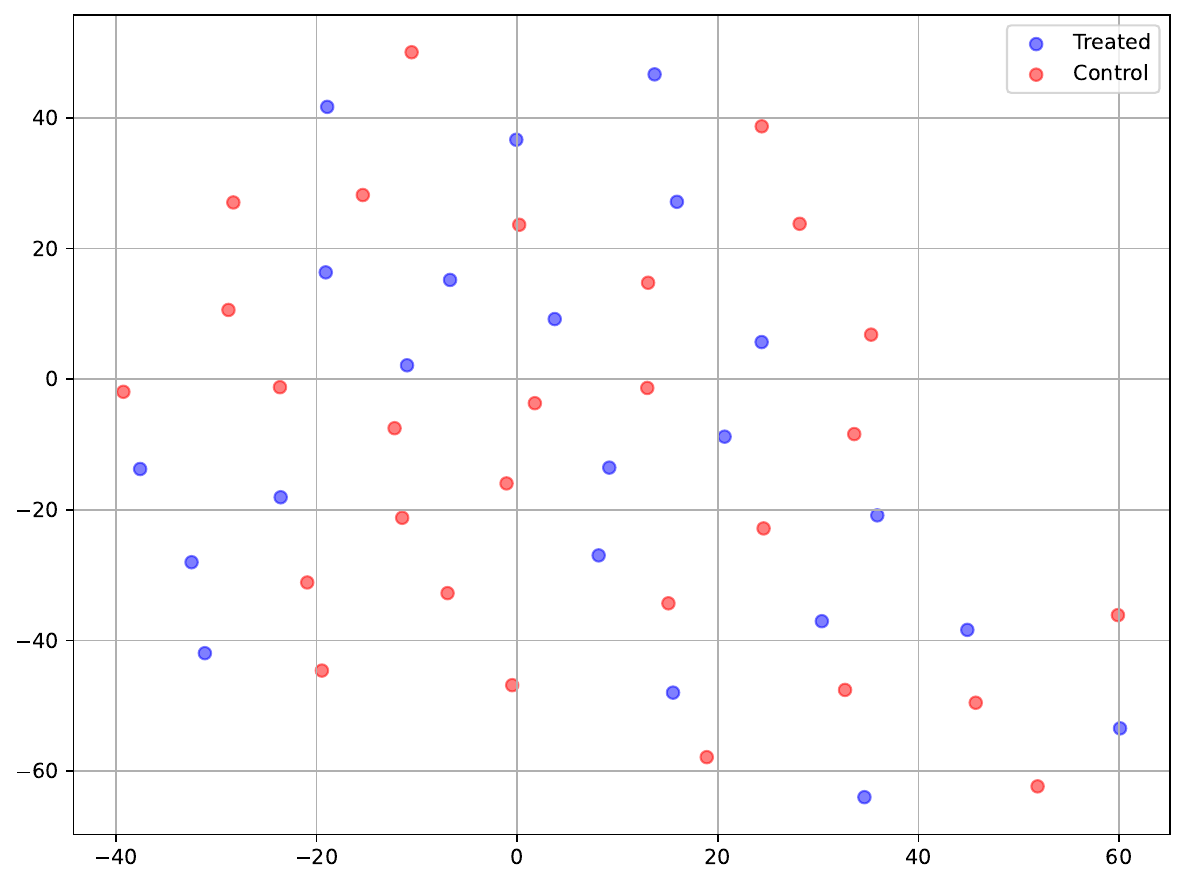}}
  \subfigure[MACAL at Step 15\label{fig:ihdp_macal_15}]{\includegraphics[width=0.15\textwidth]{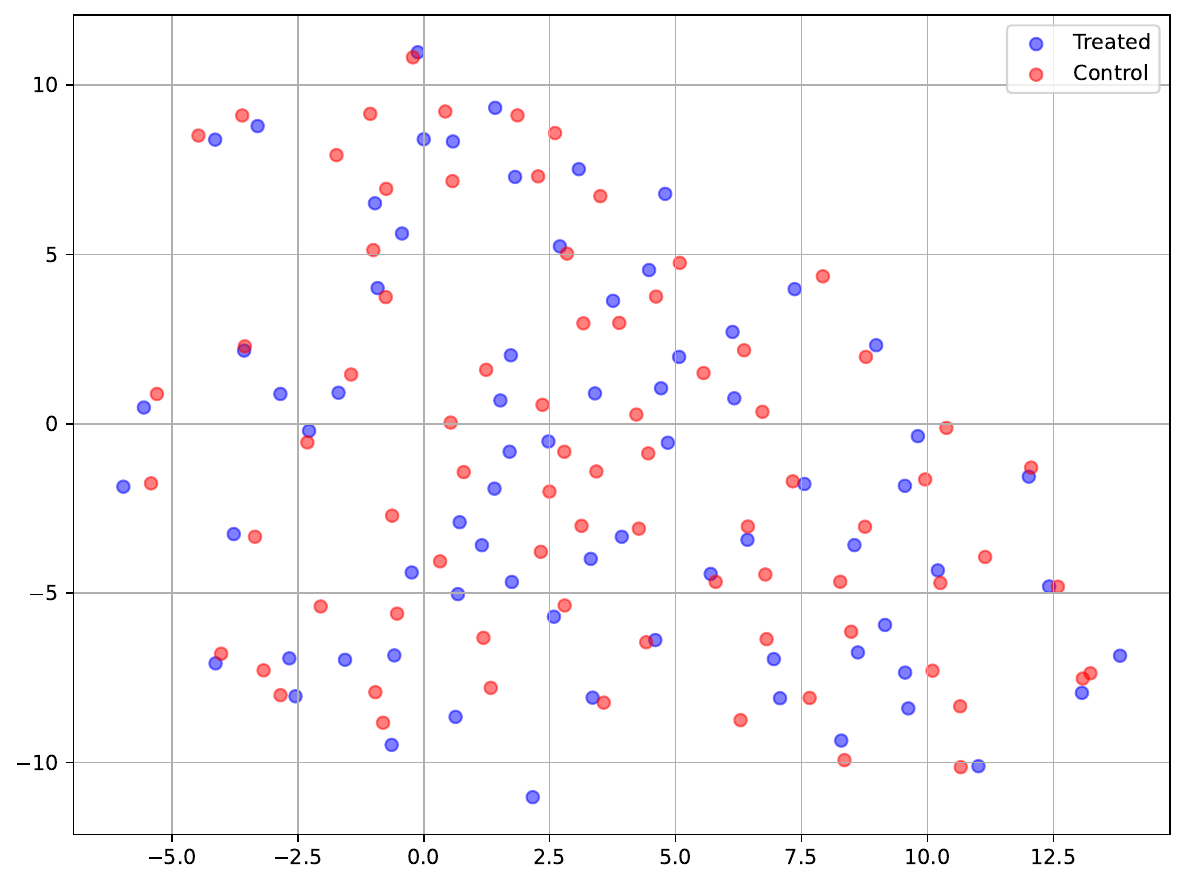}}
  \subfigure[MACAL at Step 10\label{fig:ibm_macal_10}]{\includegraphics[width=0.15\textwidth]{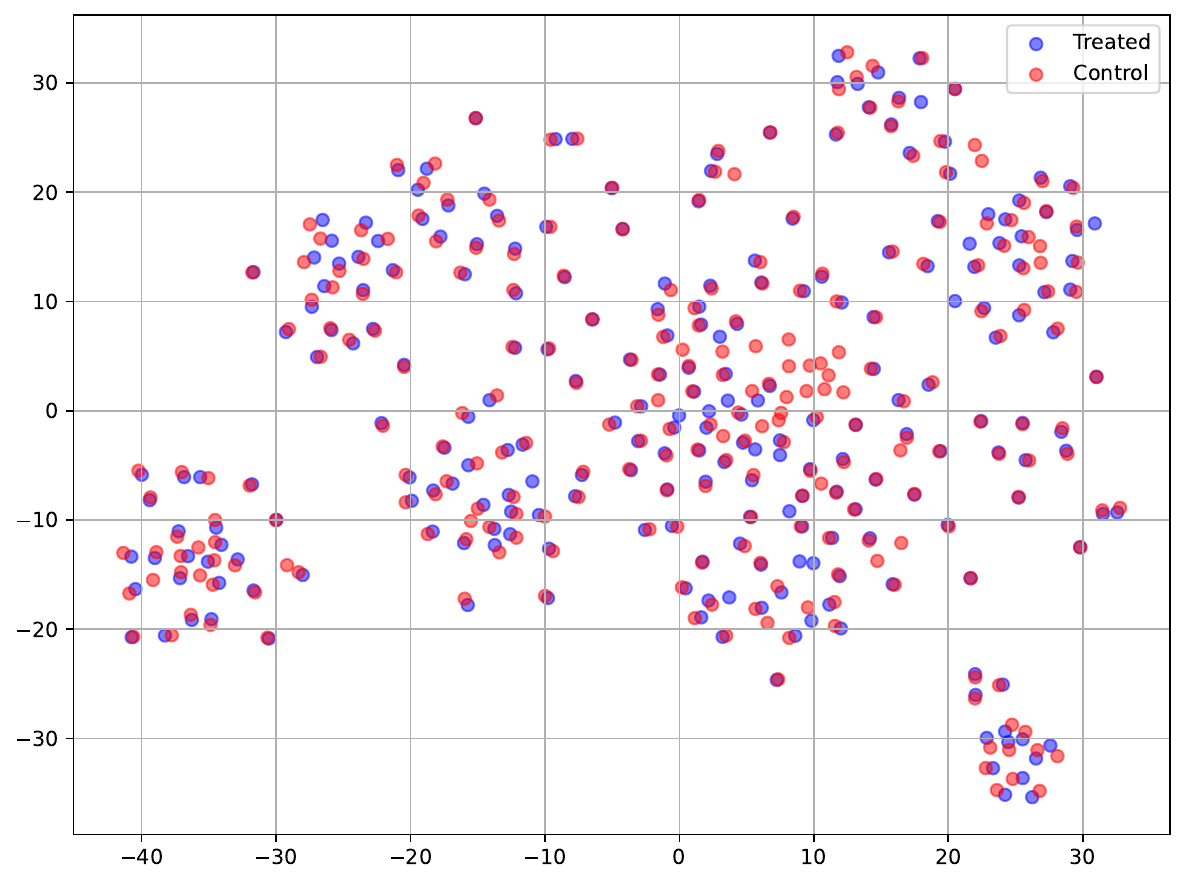}}
    \subfigure[MACAL at Step 30\label{fig:ibm_macal_30}]{\includegraphics[width=0.15\textwidth]{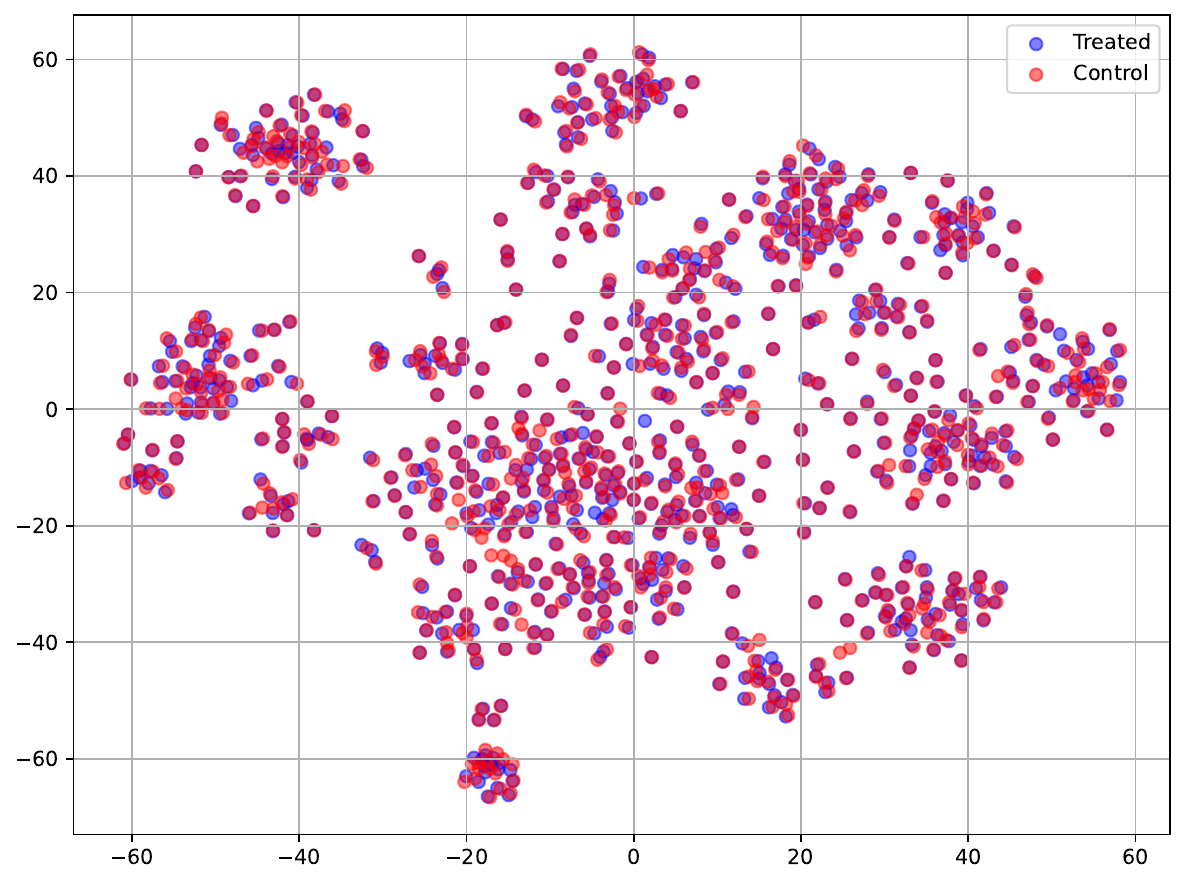}}
      \subfigure[MACAL at Step 10\label{fig:cmnist_macal_10}]{\includegraphics[width=0.15\textwidth]{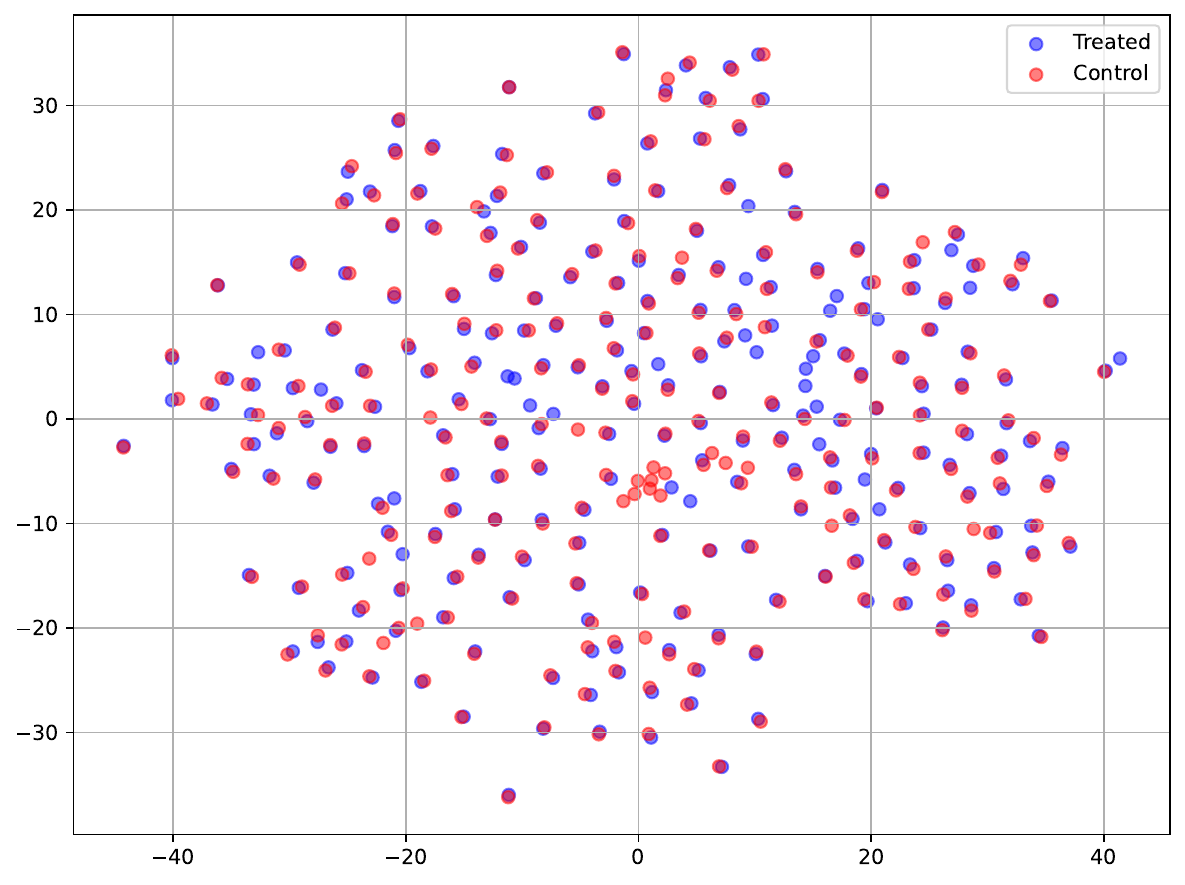}}
        \subfigure[MACAL at Step 30\label{fig:cmnist_macal_30}]{\includegraphics[width=0.15\textwidth]{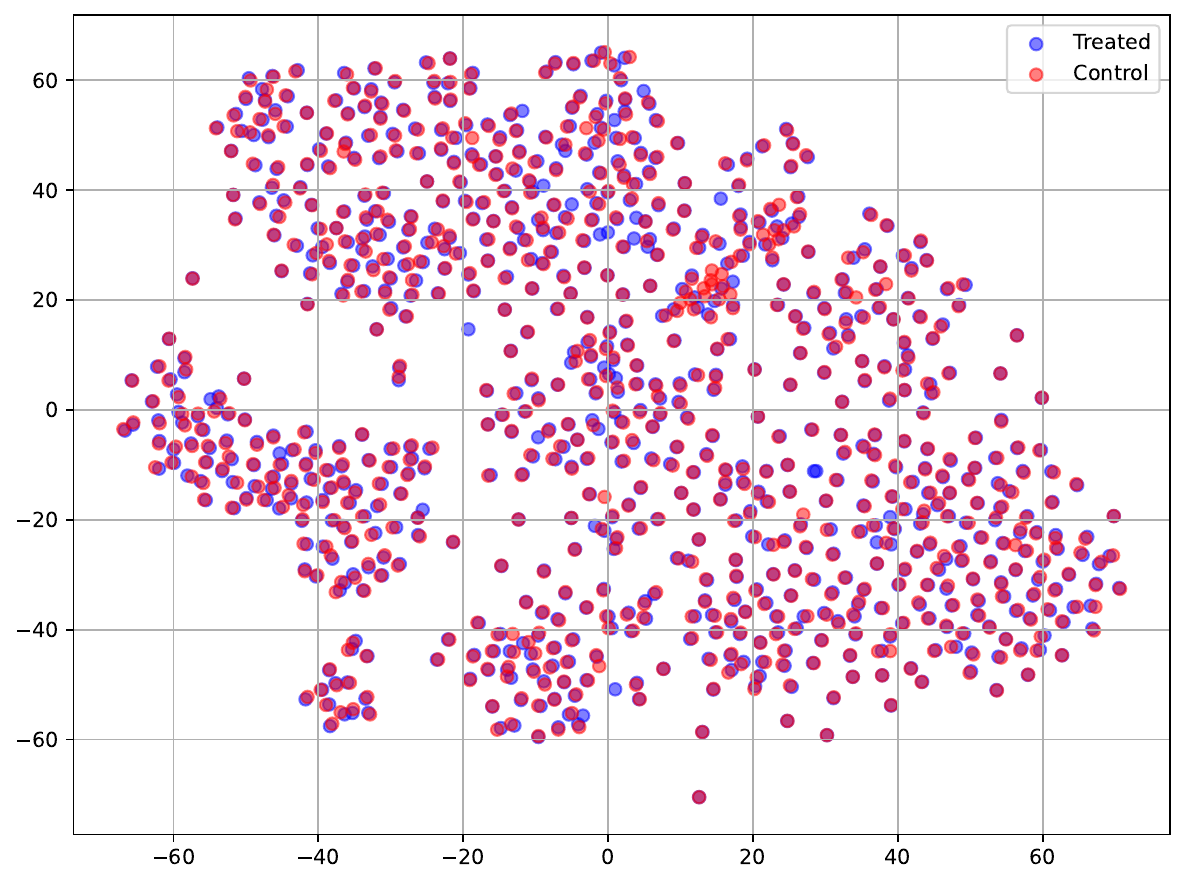}}
  \vspace{-0.3cm}
  \caption{Visualization of the post-acquisition dataset (IHDP: First two columns, IBM: Middle two columns, and CMNIST: Last two columns) via t-SNE for \textcolor{blue}{treatment group $t=1$}, \textcolor{red}{treatment group $t=0$}, and \textcolor{purple}{overlapping} for Random, $\mu\rho$BALD, and MACAL.}
  \label{fig:tsne}
  \vspace{-0.3cm}
\end{figure*}

\begin{figure*}[ht!]
  \centering
  \includegraphics[scale=0.3]{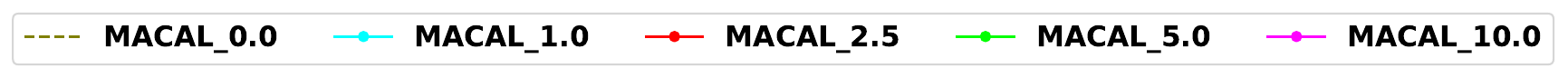}\\
  \vspace{-0.2cm}
  \subfigure[IHDP-DUE-DNN\label{fig:ihdp_alpha}]{\includegraphics[width=0.32\textwidth]{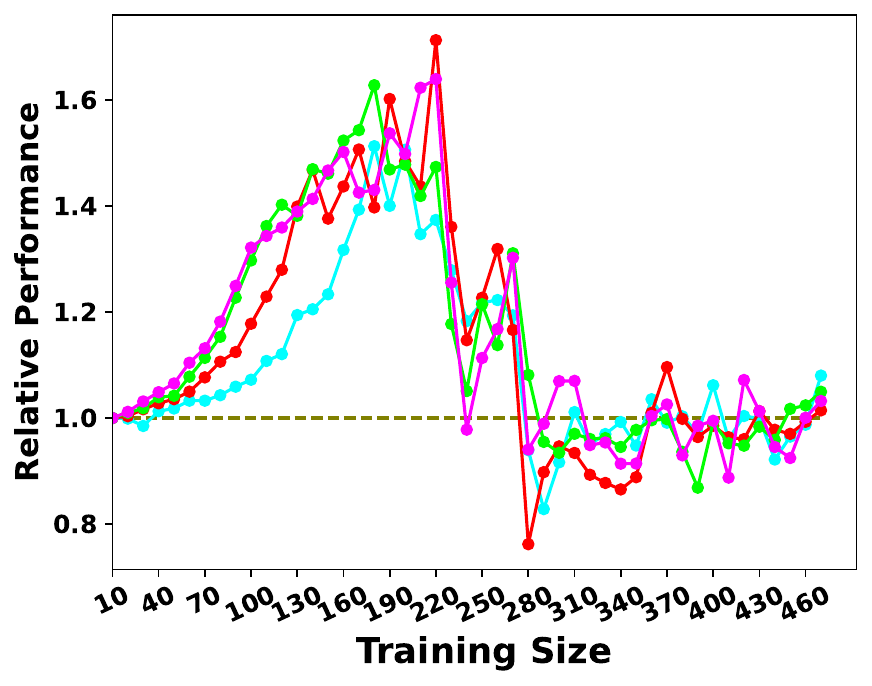}}
  \subfigure[IBM-DUE-DNN\label{fig:ibm_alpha}]{\includegraphics[width=0.3\textwidth]{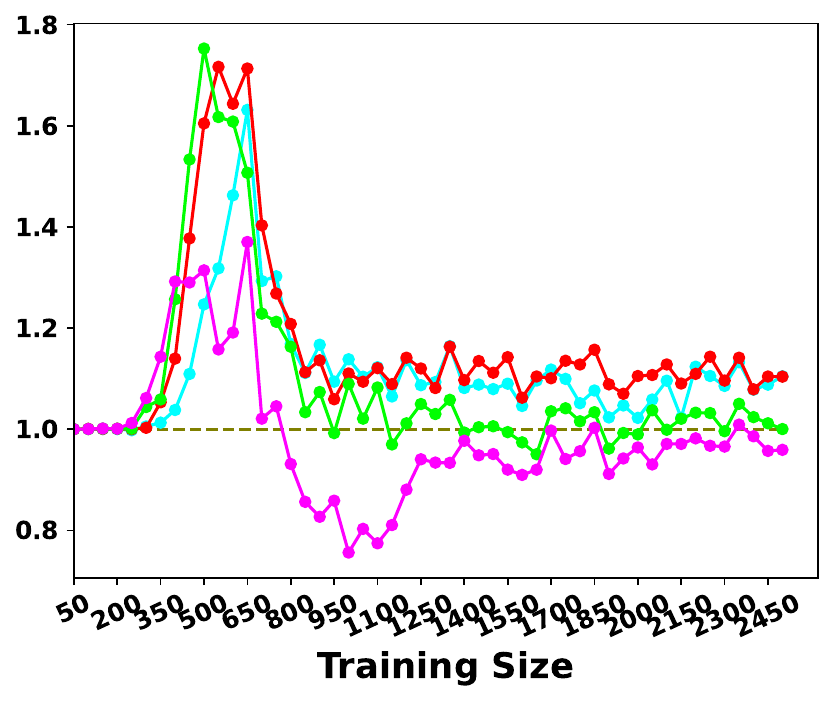}}
  \subfigure[CMNIST-DUE-CNN\label{fig:cmnist_cnn_alpha}]{\includegraphics[width=0.3\textwidth]{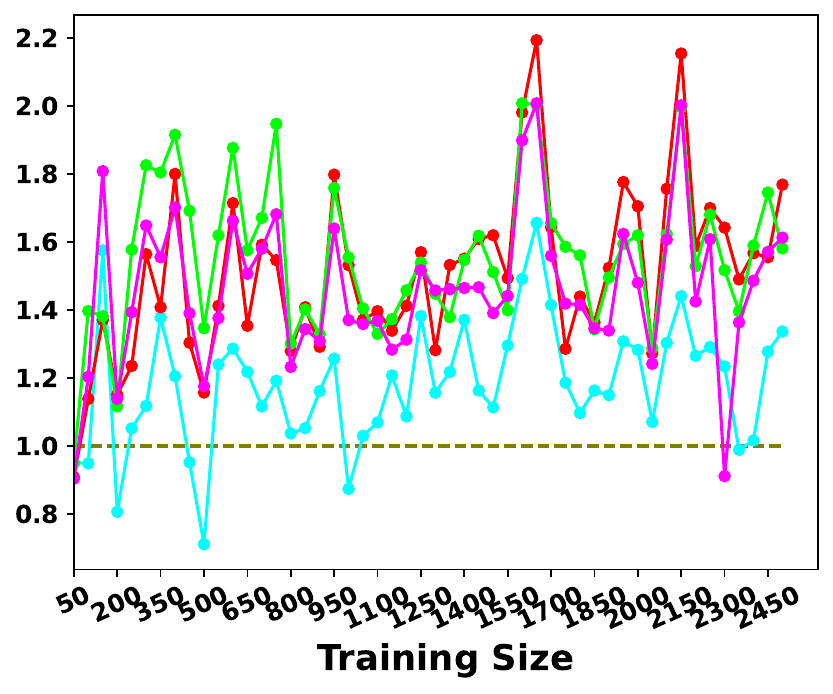}}
  \vspace{-0.2cm}
  \caption{Hyperparameter $\alpha$ representing various levels of symmetrical regularization for label acquisition. The relative performance of MACAL\_$\alpha_{i}$ is calculated as $\epsilon_{\text{PEHE},\alpha=0}$/$\epsilon_{\text{PEHE},\alpha=\alpha_{i}}$, the higher the better.}
  \label{fig:ablations_alpha}
  \vspace{-0.2cm}
\end{figure*} 

\subsection{Baseline Risk Evaluation}

Generally, across all figures, MACAL's performance set an empirical risk lower bound for all the other baselines. On the IHDP dataset, our proposed method obtains a lower risk till 160 training size (turning point). Then, it performs similarly to the other methods till the exhaustion of the pool set. This phenomenon is explainable due to the fact that the samples from treatment group $t=1$ get exhausted at the turning point, and MACAL can only acquire the samples from the treatment group ($t=0$) without benefiting from the reduction of distributional discrepancy by labelling similar pairs. It is also arguable that when deploying the general AL methods, e.g., BADGE, directly into active causal effect estimation, it is interesting to see that even the Random method can perform similarly to these SOTAs. We overall see a better performance of causal AL methods, e.g., Causal-BALD, and MACAL, over the general AL methods by additionally considering reducing the violation of positivity during label acquisition.

Moreover, none of the current SOTAs focusing on active causal effect estimation can consistently outperform the other methods from the general active learning research field across all the datasets. QHTE mostly underperforms because it only focuses on selecting the point that has the maximum distance from its closest neighbour in the current training set without meaningful constraints on post-acquisition imbalance on raw features. Also, even though $\mu\rho$BALD is the most representative method from \cite{jesson2021causal} incorporating the imbalance penalization in its query criterion, such indirect regularization via counterfactual uncertainty is not as optimal as ours. Because MACAL directly reduces the post-acquisition distributional imbalance by acquiring more similar pairs based on raw features. It is also noted that the division-form criterion of Causal-BALD can embed immense variation (the performance of all its variants fluctuates drastically) in estimations as shown in Figure \ref{fig:cmnist_cnn_causalal}, while our proposed simple addition-form criterion is significantly more stable. 

\subsection{Acquisition Visualization}

To give a direct comparison of acquisition quality, we visualize the results in Figure \ref{fig:tsne} by projecting the post-acquisition training set from three acquisition criteria, i.e., Random, $\mu\rho$BALD, and MACAL, on each dataset at two different query stages into the 2-dimensional latent space via t-SNE \cite{van2008visualizing}. From Figure \ref{fig:ihdp_random_5} to \ref{fig:cmnist_random_30} across three datasets, we observe that the Random draw from the original distribution inherently expresses a strong violation of positivity across three different datasets, such that we barely see large overlapping regions. While, $\mu\rho$BALD shown in Figure \ref{fig:ihdp_murho_5} to \ref{fig:cmnist_murho_30} looks slightly better than the Random method by being more spread out, but it still cannot well resolve the violation of the positivity issue at large scale, rendering a significantly imbalanced label acquisition for different treatment groups. As for MACAL, we observe an exceptional acquisition result from Figure \ref{fig:ihdp_macal_5} to \ref{fig:cmnist_macal_30}, each of the samples from both of the treatment groups can mostly find its (close) counterfactual such that the violation of positivity is significantly reduced. The acquisition by MACAL also shows high diversity instead of clustering. Hence, the remarkable performance gap shown in Figure \ref{fig:cmnist_cnn_causalal} can also be explained in essence by the acquisition results shown in \ref{fig:cmnist_macal_10} and \ref{fig:cmnist_macal_30}. Additional visualizations for the other baselines, and on different datasets are accessible in Appendix \ref{appendix:visualization}.

\subsection{Symmetrical Regularization Study\label{section:ablations}}
We conduct extensive ablation experiments for $\alpha\in\{0, 1, 2.5, 5, 10\}$, a clear observation is that, even though the benefit of setting the symmetrical regularization is non-trivial, there is no single hyperparameter $\alpha$ that can consistently outperform all the others throughout the entire label acquisition process. Also, the stronger regularization, e.g., $\alpha=10$, delivers better performance at the early stage of the acquisition, but such an advantage cannot be maintained across the whole acquisition process. Interestingly, during the course of the acquisition, a decreasing coefficient empirically grants an increasing relative performance, e.g., $\alpha=2.5$ underperforms $\alpha=10$ at the early stage, but it outperforms $\alpha=10$ in the later stage. This is explainable because when the key set of the overlapping samples is mostly collected, there is less information can be obtained from acquiring the repetitive samples even though these are from the overlapping region. The criterion should bias its acquisition toward the uncertain non-overlapping area to gain more information to reduce the risk of the model.

\section{Conclusion}

In this paper, we study the well-under-explored yet important and practical active causal effect estimation problem and construct a theoretical framework from a novel and intriguing perspective, i.e., decompose a more informative risk upper bound without loosening it and give mathematically guaranteed risk convergence analysis under certain conditions. Therefore, in theory we maximize the decomposed terms at each query step in order to minimize the generalization risk. Subsequently, we derive a theory-inspired simplified yet effective label acquisition algorithm, i.e., MACAL, which considers the joint reduction of the model's variance and post-acquisition distributional imbalance via a simplified yet effective label acquisition criterion. Moreover, reaching data-efficient labelling is never an NP-hard problem via MACAL, and thus the optimum can be obtained in polynomial time with $\mathcal{O}(N^{2})$. It is generally demonstrated that our proposed method consistently outperforms the other baselines across all the datasets with a non-trivial performance gain.

\begin{acks}
    This work is supported by the Australian Research Council under the streams of Future Fellowship (No. FT210100624), Discovery Early Career Researcher Award (No. DE230101033), Industrial Transformation Training Centre (No. IC200100022), Discovery Project (No. DP240101108 and No. DP240101814), and Linkage Project (No. LP230200892). Partial support from Health and Wellbeing Queensland is gratefully acknowledged. HW would like to additionally thank the invaluable support from Carrie Chen.
\end{acks}

\newpage
\bibliographystyle{ACM-Reference-Format}
\bibliography{sample-sigconf}

\appendix

\setcounter{assumption}{0}
\setcounter{proposition}{0}
\setcounter{theorem}{0}

\newpage
\section{Theory\label{appendix:theory}}

\subsection{Convergence Behaviour of Risk Upper Bound\label{appendix:theorem_1}}

\begin{theorem}\label{appendixtheorem:1}
    With budget $\mathcal{M}$, the maximum risk upper bound reduction $\Delta_{\mathcal{B}_{\text{overall}}}$ is achieved at the termination of the entire I data query steps given that the generalization risk upper bound shrinkage $\Delta_{\mathcal{B}_{i}}$ is maximized at each query step $\forall i$, i.e.:
    \begin{equation}
        \begin{split}
            \argmax_{\Tilde{\mathcal{D}}_{\text{overall}}}~&\Delta_{\mathcal{B}_{\text{overall}}}=\bigcup_{i=1}^{I}\argmax_{\Tilde{\mathcal{D}}_{i}}\Delta_{\mathcal{B}_{i}}\\
            &s.t.~ |\Tilde{\mathcal{D}}_{\text{overall}}|\leq\mathcal{M},
        \end{split}
    \end{equation}
    where $\Tilde{\mathcal{D}}_{\text{overall}}$ is the overall acquired data, $\Tilde{\mathcal{D}}_{i}$ is the acquired batch at $i$-th query step, and $\Delta_{\mathcal{B}_{i}}=\sum_{t\in\{0,1\}}\Delta_{\text{Var}_{i}}^{t}(\Tilde{\mathcal{D}}_{i})+C_{\phi}\Delta_{\text{IPM}_{i}}(\Tilde{\mathcal{D}}_{i})$. The convergence rate of the risk upper bound has the following guaranteed behaviours under certain circumstances:
    \begin{enumerate}[label=\roman*)] 
    \item When variance reduction $\sum_{t\in\{0,1\}}\Delta_{\text{Var}_{i}}^{t}(\Tilde{\mathcal{D}}_{i})$ becomes the dominant part of the risk upper bound, the risk convergence is lower-bounded by $\Omega(\beta^{i})$ with constant $\beta\in[0,1)$.
    \item While, with dominant constant $C_{\phi}$, the risk convergence is upper-bounded by $\mathcal{O}(\frac{1}{i+\gamma_{0}})$ with constant $\gamma_{0}\in\mathbb{R}^{+}$.
    \end{enumerate}
\end{theorem}

\begin{proof}
    Shalit et al. \cite{shalit2017estimating} provide the well-known upper-bound for the expected PEHE as:
    \begin{equation}
        \epsilon_{\text{PEHE}}
\leq2[\epsilon_{\text{F}}^{t=1}+\epsilon_{\text{F}}^{t=0}+C_{\phi}(\text{IPM}_{\mathcal{F}}(p^{t=1}_{\phi}, p^{t=0}_{\phi}))].
    \end{equation} 
    We denote such upper-bound at $i$-th query step as $\mathcal{B}_{i}$ with post-acquisition dataset $\mathcal{D}_{i}$ (which includes the labelled optimal batch $\Tilde{\mathcal{D}}^{*}_{i}$). By factual error decomposition in Proposition \ref{appendixprop:bias-variance}, we have:
    \begin{equation}
        \epsilon_{\text{F}}^{t}=\mathbb{E}_{\mathcal{X}}[(f^{t}(x) - \mathbb{E}[\hat{f}^{t}(x)])^{2}] + \mathbb{E}_{\mathcal{X}}[\mathbb{E}[(\hat{f}^{t}(x) - \mathbb{E}[\hat{f}^{t}(x)])^{2}]]+\sigma_{\xi^{t}}^{2}.
    \end{equation}
    
    We leave out the constant multiplier 2 in the original bound for notation simplicity during deduction (as the calculation for the shrinkage cancels off the constant), and derive the upper bound $\mathcal{B}_{i}$ at $i$-th query step in a brand-new form but with the same tightness as it is:
    \begin{subequations}
    \begin{align}
        \mathcal{B}_{i}=&\epsilon_{i,\text{F}}^{t=1}+\epsilon_{i,\text{F}}^{t=0}+C_{\phi}(\text{IPM}_{\mathcal{F}}(p^{t=1}_{i,\phi}, p^{t=0}_{i,\phi}))\\
        =&\sum_{t\in\{0,1\}}\mathbb{E}_{\mathcal{X}}\left[\text{Bias}^{t}_{i}[\hat{f}^{t}_{i}(x;\mathcal{D}_{i})]\right] + \\&\sum_{t\in\{0,1\}}\mathbb{E}_{\mathcal{X}}\left[\text{Var}^{t}_{i}[\hat{f}^{t}_{i}(x;\mathcal{D}_{i})]\right] +\\
        &\sum_{t\in\{0,1\}}\sigma^{2}_{\xi^{t}_{i}} + C_{\phi}\text{IPM}_{\mathcal{F}}(p^{t=1}_{i,\phi}, p^{t=0}_{i,\phi}),
    \end{align}
    \end{subequations} where $\text{Bias}^{t}_{i}[\hat{f}^{t}_{i}(x;\mathcal{D}_{i})]=(f^{t}_{i}(x) - \mathbb{E}[\hat{f}^{t}_{i}(x)])^{2}$, and $\text{Var}^{t}_{i}[\hat{f}^{t}_{i}(x;\mathcal{D}_{i})]=\mathbb{E}[(\hat{f}^{t}_{i} - \mathbb{E}[\hat{f}^{t}_{i}])^{2}]$. 
    
    Analogously, at ($i-$1)-th query step, namely $\mathcal{B}_{i-1}$ with dataset $\mathcal{D}_{i-1}$, we have:
    \begin{subequations}
    \begin{align}
        \mathcal{B}_{i-1}
        =&\sum_{t\in\{0,1\}}\mathbb{E}_{\mathcal{X}}\left[\text{Bias}^{t}_{i-1}[\hat{f}^{t}_{i-1}(x;\mathcal{D}_{i-1})]\right] + \\&\sum_{t\in\{0,1\}}\mathbb{E}_{\mathcal{X}}\left[\text{Var}^{t}_{i-1}[\hat{f}^{t}_{i-1}(x;\mathcal{D}_{i-1})]\right] + \\&\sum_{t\in\{0,1\}}\sigma^{2}_{\xi^{t}_{i-1}} + C_{\phi}\text{IPM}_{\mathcal{F}}(p^{t=1}_{i-1,\phi}, p^{t=0}_{i-1,\phi}).
    \end{align}
    \end{subequations}
    Subsequently, the shrinkage $\Delta_{\mathcal{B}_{i}}$ at $i$-th query step is defined as:
    \begin{subequations}
        \begin{align}
        &\Delta_{\mathcal{B}_{i}} = \mathcal{B}_{i-1} - \mathcal{B}_{i}\\
        =&\sum_{t\in\{0,1\}}\mathbb{E}_{\mathcal{X}}\left[\text{Bias}^{t}_{i-1}[\hat{f}^{t}_{i-1}(x;\mathcal{D}_{i-1})]-\text{Bias}^{t}_{i}[\hat{f}^{t}_{i}(x;\mathcal{D}_{i})])\right]+\\ 
        &\sum_{t\in\{0,1\}}\mathbb{E}_{\mathcal{X}}\left[\text{Var}^{t}_{i-1}[\hat{f}^{t}_{i-1}(x;\mathcal{D}_{i-1})]-\text{Var}^{t}_{i}[\hat{f}^{t}_{i}(x;\mathcal{D}_{i})]\right]+\\
        &\sum_{t\in\{0,1\}}(\sigma^{2}_{\xi^{t}_{i-1}}-\sigma^{2}_{\xi^{t}_{i}}) + C_{\phi}\left(\text{IPM}_{\mathcal{F}}(p^{t=1}_{i-1,\phi}, p^{t=0}_{i-1,\phi})-\text{IPM}_{\mathcal{F}}(p^{t=1}_{i,\phi}, p^{t=0}_{i,\phi})\right)\\
        =&0+\\
        &\sum_{t\in\{0,1\}}\underbrace{\mathbb{E}_{\mathcal{X}}\left[\text{Var}^{t}_{i-1}[\hat{f}^{t}_{i-1}(x;\mathcal{D}_{i-1})]-\text{Var}^{t}_{i}[\hat{f}^{t}_{i}(x;\mathcal{D}_{i})]\right]}_{\Delta_{\text{Var}_{i}}^{t}}+\\
        &0+C_{\phi}\underbrace{\left(\text{IPM}_{\mathcal{F}}(p^{t=1}_{i-1,\phi}, p^{t=0}_{i-1,\phi})-\text{IPM}_{\mathcal{F}}(p^{t=1}_{i,\phi}, p^{t=0}_{i,\phi})\right)}_{\Delta_{\text{IPM}_{i}}}\\
        =&\sum_{t\in\{0,1\}}\Delta_{\text{Var}_{i}}^{t}+C_{\phi}\Delta_{\text{IPM}_{i}}.\label{eq:shrinkage}
        \end{align}
    \end{subequations} 
    The second equality holds because the models' bias only depends on the selection of the model class \cite{cohn1996active}, or more empirically models' bias is negligible for models with enough complexity \cite{settles2009active}. Either way two bias terms cancelled off. Also, the data generation process has the same noise assumption, thus two noise variance terms cancelled off. 

    The overall bound shrinkage $\Delta_{\mathcal{B}}$ after the termination of the entire $I$ query steps is thus:
    \begin{subequations}
        \begin{align}
            \Delta_{\mathcal{B}_{\text{overall}}}
            &=\mathcal{B}_{0} - \mathcal{B}_{I}\\
            &=\mathcal{B}_{0} - \mathcal{B}_{1} + \mathcal{B}_{1} - \mathcal{B}_{2} + \dots + \mathcal{B}_{I-2} - \mathcal{B}_{I-1} + \mathcal{B}_{I-1} - \mathcal{B}_{I}\\
            &=\Delta_{\mathcal{B}_{0}} + \Delta_{\mathcal{B}_{1}} + \dots + \Delta_{\mathcal{B}_{I-1}} + \Delta_{\mathcal{B}_{I}}\\
            &=\sum_{i=1}^{I}\Delta_{\mathcal{B}_{i}}.
        \end{align}
    \end{subequations} 
    
    Therefore, to maximize the bound reduction and return the optimal set:
    \begin{equation}
        \argmax_{\Tilde{\mathcal{D}}_{\text{overall}}}\Delta_{\mathcal{B}_{\text{overall}}}=\argmax_{\Tilde{\mathcal{D}}_{\text{overall}}=\bigcup_{i=1}^{I}\Tilde{\mathcal{D}}_{i}}\sum_{i=1}^{I}\Delta_{\mathcal{B}_{i}}=\bigcup_{i=1}^{I}\argmax_{\Tilde{\mathcal{D}}_{i}}\Delta_{\mathcal{B}_{i}},
    \end{equation}
    
    where the entire optimal set $\Tilde{\mathcal{D}}^{*}_{\text{overall}}$ is a union of the optimal set $\Tilde{\mathcal{D}}^{*}_{i}$ which is acquired at every query step to maximize the shrinkage $\Delta_{\mathcal{B}_{i}}$ in (\ref{eq:shrinkage}), thus we can conclude that the maximum risk upper bound reduction is obtained after the termination of entire label acquisition process.

    Now that we define the shrinkage $A_{i}=\Delta_{\mathcal{B}_{i}}/\mathcal{B}_{i-1}$ at $i$-th step, and discuss two extreme contexts for the total risk upper bound since it is hardly to compute the exact value for the bounded constant $C_{\phi}$ as discussed in \cite{shalit2017estimating}.

    Scenario 1: Lower-bounded convergence rate with negligible $C_{\phi}$
    \begin{subequations}
    \begin{align}
        &A_{i}=\frac{\Delta_{\mathcal{B}_{i}}}{\mathcal{B}_{i-1}}\\
        &=\frac{\sum_{t\in\{0,1\}}\Delta_{\text{Var}_{i}}^{t}+C_{\phi}\Delta_{\text{IPM}_{i}}}{\sum_{t\in\{0,1\}}\mathbb{E}_{\mathcal{X}}\left[\text{Var}^{t}_{i-1}[\hat{f}^{t}(x;\mathcal{D}_{i-1})]\right] + C_{\phi}\text{IPM}_{\mathcal{F}}(p^{t=1}_{i-1,\phi}, p^{t=0}_{i-1,\phi})+\zeta} \\
        &\approx G_{i},
    \end{align}
    \end{subequations} where $\zeta=\sum_{t\in\{0,1\}}\sigma^{2}_{\xi^{t}_{i-1}}$, and the approximation $\approx$ is given by the a small enough $C_{\phi}$, and $G_{i}$ is the shrinkage in terms of the variance reduction at $i$-query step defined in Eq. \ref{eq:shrinkage_g}.

    When the variance term becomes the dominant part of the risk upper bound and leaving out the distributional discrepancy, we have the shrinkage difference $|A_{i}-G_{i}|<\epsilon$, for a small $\epsilon>0$. Under such circumstances, the convergence rate of the risk upper bound cannot go slower than $\Omega(\beta^{i})$ as stated in Lemma \ref{appendix_lemma:convergence_var} with proof given in \ref{proof:convergence_var}.

    Scenario 2: Upper-bounded convergence rate with dominant $C_{\phi}$
    \begin{subequations}
    \begin{align}
        &A_{i}=\frac{\Delta_{\mathcal{B}_{i}}}{\mathcal{B}_{i-1}}\\
        &=\frac{\sum_{t\in\{0,1\}}\Delta_{\text{Var}_{i}}^{t}+C_{\phi}\Delta_{\text{IPM}_{i}}}{\sum_{t\in\{0,1\}}\mathbb{E}_{\mathcal{X}}\left[\text{Var}^{t}_{i-1}[\hat{f}^{t}(x;\mathcal{D}_{i-1})]\right] + C_{\phi}\text{IPM}_{\mathcal{F}}(p^{t=1}_{i-1,\phi}, p^{t=0}_{i-1,\phi}) + \zeta}\\
        &\leq\frac{\sum_{t\in\{0,1\}}\Delta_{\text{Var}_{i}}^{t}+C_{\phi}\Delta_{\text{IPM}_{i}}}{C_{\phi}\text{IPM}_{\mathcal{F}}(p^{t=1}_{i-1,\phi}, p^{t=0}_{i-1,\phi})}\leq\frac{\sigma^{2}_{f^{t}}+C_{\phi}\Delta_{\text{IPM}_{i}}}{C_{\phi}\text{IPM}_{\mathcal{F}}(p^{t=1}_{i-1,\phi}, p^{t=0}_{i-1,\phi})}\\
        &=\frac{\frac{\sigma^{2}_{f^{t}}}{C_{\phi}}+\Delta_{\text{IPM}_{i}}}{\text{IPM}_{\mathcal{F}}(p^{t=1}_{i-1,\phi}, p^{t=0}_{i-1,\phi})}\approx S_{i},
    \end{align}
    \end{subequations} where $\zeta=\sum_{t\in\{0,1\}}\sigma^{2}_{\xi^{t}_{i-1}}$, and the approximation $\approx$ is given by the the dominant $C_{\phi}$, and $S_{i}$ is the shrinkage in terms of the discrepancy reduction at $i$-query step defined in Eq. \ref{eq:shrinkage_s}.

    When the discrepancy term becomes the dominant part of the risk upper bound and leaving out the variance, we have the shrinkage difference $|A_{i}-S_{i}|<\epsilon$, for a small $\epsilon>0$. Under such circumstances, the convergence rate of the risk upper bound cannot exceed $\mathcal{O}(\frac{1}{i+\gamma_{0}})$ as stated in Lemma \ref{appendix_lemma:convergence_wass} with proof given in \ref{proof:convergence_wass}

    Thus we conclude the proof for the two convergence behaviours under two extreme circumstances from the influence of the distributional discrepancy.
    
\end{proof}

\subsection{Convergence Rate of Model Variance\label{proof:convergence_var}}

\begin{lemma}\label{appendix_lemma:convergence_var}
    Given the variance counted by Gaussian process regression model $f^{t}$ on treatment group $t$, by acquiring the most uncertain samples that have the maximum predictive variance $\sigma^{2}_{f^{t}}$, the slowest convergence rate of the model variance is lower-bounded by $\Omega(\beta^{i})$, where $0\leq\beta<1$.
\end{lemma}

\begin{proof}

We denote the model variance $\mathbb{E}_{\mathcal{X}}\left[\text{Var}^{t}[\hat{f}^{t}(x;\mathcal{D}^{t}_{\text{train}})]\right]=\mathbb{E}_{\mathcal{X}}\left[(\hat{f}^{t}(x) - \mathbb{E}[\hat{f}^{t}(x)])^{2}\right]$ on treatment group $t$. With the empirical distribution on sample space $\mathcal{X}^{t}$, the empirical realization of the model variance with $N_{\text{pool}}$ samples is: 
\begin{equation}
    \mathbb{E}_{\mathcal{X}}\left[\text{Var}^{t}[\hat{f}^{t}(x;\mathcal{D}^{t}_{\text{train}})\right]=\frac{1}{N_{\text{pool}}}\sum_{i=1}^{N_{\text{pool}}}\sigma^{2}(f^{t}(x_{i})),
\end{equation} where the predictive variance for observation $x_{i}$ is denoted as $\sigma^{2}(f^{t}(x_{i}))$.

For the model that counts the gold standard variance, i.e., Gaussian process \cite{williams2006gaussian}, will cap the variance by the constant signal variance $\sigma^{2}_{f^{t}}$, e.g., $\sigma^{2}_{f^{t}}=1$. Intuitively for samples far away from the training set, the model's belief reverts back to the prior. Mathematically, for noiseless observations, we have the following bounded predictive variance $\sigma^{2}(f^{t}(x_{*}))$ for any $x_{*}$:
\begin{equation}
    0\leq\sigma^{2}(f^{t}(x_{*}))=k(x_{*}, x_{*}) - \textbf{k}_{*}^{T}K^{-1}\textbf{k}_{*}\leq\sigma_{f^{t}}^{2}
\end{equation} where the RBF kernel $k(x_i, x_j) = \sigma^{2}_{f^{t}}\exp\left(-\frac{1}{2\theta} \|x_i - x_j\|^2 \right)$. We have zero variance estimation if $x_{*}\in\mathcal{D}^{t}_{\text{train}}$, and maximally $\sigma_{f^{t}}^{2}$ if $x_{*}$ is far away since $\forall x_{i}\in\mathcal{D}^{t}_{\text{train}}, \exp\left(-\frac{1}{2\theta} \|x_i - x_{*}\|^2\right)\rightarrow0$.

\setcounter{theorem}{3}
\begin{lemma}\label{lemma:inequality}
    Denote the predictive variance of the Gaussian process regression model $f$, trained on the dataset of size $m$, for any test point $x_{*}$ as $\sigma^{2}_{m}(f(x_{*}))$, the predictive variance will not grow when the training set is expanding with size $m_{0}\geq0$, i.e.: \cite{williams2000upper}
    \begin{equation}
        \sigma^{2}_{m+m_{0}}(f(x_{*}))\leq\sigma^{2}_{m}(f(x_{*})).
    \end{equation}
\end{lemma}

Thus, assuming the training size $N_{i-1}=|\mathcal{D}^{t}_{i-1}|$, the fixed batch size $b_{0}=|\Tilde{\mathcal{D}}^{t}_{i}|$, and the post-acquisition training size $N_{i}=|\mathcal{D}^{t}_{i-1}|+|\Tilde{\mathcal{D}}^{t}_{i}|$. Let's denote $N_{\text{pool}}$ to be the pool set at $(i-1)$-th query step  (containing the acquired batch $\Tilde{\mathcal{D}}^{t}_{i}$), the variance reduction at $i$-th query step is:
\begin{subequations}
    \begin{align}\label{eq:variance_reduction}
        &\Delta_{\text{Var}_{i}}^{t}(\Tilde{\mathcal{D}}^{t}_{i})\\
        =&\mathbb{E}_{\mathcal{X}}\left[\text{Var}^{t}_{i-1}[\hat{f}^{t}(x;\mathcal{D}^{t}_{i-1})]-\text{Var}^{t}_{i}[\hat{f}^{t}(x;\mathcal{D}^{t}_{i-1}\cup\Tilde{\mathcal{D}}^{t}_{i})]\right]\\
        =&\frac{1}{N_{\text{pool}}}\sum_{k=1}^{N_{\text{pool}}}\sigma^{2}_{N_{i-1}}(\hat{f}^{t}(x_{k})) - \frac{1}{N_{\text{pool}}}\sum_{k=1}^{N_{\text{pool}}}\sigma^{2}_{N_{i}}(\hat{f}^{t}(x_{k}))\\
        =&\frac{1}{N_{\text{pool}}}\left(\sum_{k=1}^{N_{\text{pool}}-b_{0}}\sigma^{2}_{N_{i-1}}(\hat{f}^{t}(x_{k}))+\sum_{j=N_{\text{pool}}-b_{0}+1}^{N_{\text{pool}}}\sigma^{2}_{N_{i-1}}(\hat{f}^{t}(x_{j}))\right)-\\
        &~~~~~\frac{1}{N_{\text{pool}}}\left(\sum_{k=1}^{N_{\text{pool}}-b_{0}}\sigma^{2}_{N_{i}}(\hat{f}^{t}(x_{k}))+\sum_{j=N_{\text{pool}}-b_{0}+1}^{N_{\text{pool}}}\sigma^{2}_{N_{i}}(\hat{f}^{t}(x_{j}))\right)\\
        =&\frac{1}{N_{\text{pool}}}\sum_{k=1}^{N_{\text{pool}}-b_{0}}\underbrace{\left(\sigma^{2}_{N_{i-1}}(\hat{f}^{t}(x_{k}))-\sigma^{2}_{N_{i}}(\hat{f}^{t}(x_{k}))\right)}_{\text{None-Negative by \textbf{Lemma} \ref{lemma:inequality}}}+\\
        &\frac{1}{N_{\text{pool}}}\sum_{j=N_{\text{pool}}-b_{0}+1}^{N_{\text{pool}}}\left(\sigma^{2}_{N_{i-1}}(\hat{f}^{t}(x_{j}))-\sigma^{2}_{N_{i}}(\hat{f}^{t}(x_{j}))\right)\\
        \geq&\frac{1}{N_{\text{pool}}}\left(\sum_{j=N_{\text{pool}}-b_{0}+1}^{N_{\text{pool}}}\left(\sigma^{2}_{N_{i-1}}(\hat{f}^{t}(x_{j}))-\sigma^{2}_{N_{i}}(\hat{f}^{t}(x_{j}))\right)\right)\\
        =&\frac{1}{N_{\text{pool}}}\cdot b_{0}\cdot\sigma^{2}_{f^{t}},
    \end{align}
\end{subequations} where the first inequality is given by \textbf{Lemma} \ref{lemma:inequality}, and the last equality is given by acquiring the most uncertain samples at $i$-th query step with maximum predictive variance $\sigma^{2}_{f^{t}}$, and together with the fact that the observed sample has zero variance by the Gaussian process model.

Now that at $i$-th query step, we calculate the shrinkage $G_{i}$ as follows:
\begin{subequations}\label{eq:shrinkage_g}
\begin{align}
        G_{i}&=\frac{\Delta_{\text{Var}_{i}}^{t}(\Tilde{\mathcal{D}}^{t}_{i})}{\mathbb{E}_{\mathcal{X}}\left[\text{Var}^{t}_{i-1}[\hat{f}^{t}(x;\mathcal{D}^{t}_{i-1})]\right]}\\
        &\geq\frac{\frac{1}{N_{\text{pool}}}\cdot b_{0}\cdot\sigma^{2}_{f^{t}}}{\frac{1}{N_{\text{pool}}}\sum_{k=1}^{N_{\text{pool}}}\sigma^{2}_{N_{i-1}}(\hat{f}^{t}(x_{k}))}\\
        &=\frac{b_{0}\cdot\sigma^{2}_{f^{t}}}{\sum_{k=1}^{N_{\text{pool}}}\sigma^{2}_{N_{i-1}}(\hat{f}^{t}(x_{k}))}\\
        &\geq\frac{b_{0}\cdot\sigma^{2}_{f^{t}}}{N_{\text{pool}}\cdot\sigma^{2}_{f^{t}}}=\frac{b_{0}}{N_{\text{pool}}}=\frac{b_{0}}{\omega b_{0}}=\frac{1}{\omega},
\end{align}
\end{subequations} where $\omega\in\mathbb{R}^{+}$ as $N_{\text{pool}}$ can be arbitrary larger than $b_{0}$, i.e., $N_{\text{pool}}=\omega b_{0}$.

Thus, we derive the the total shrinking coefficient $G^{i}$ after $i$ iterations as follows:
\begin{subequations}
\begin{align}
        G^{i}&=\prod_{n=1}^{i}\left(1-G_{n}\right)\\
        &\leq\prod_{n=1}^{i}\left(1-\frac{1}{\omega}\right)\\
        &=(1-\frac{1}{\omega})^{i}=\beta^{i},
\end{align} 
\end{subequations} where the first inequality is straightforwardly by the inequality in (\ref{eq:shrinkage_g}), and furthermore we have the coefficient $0\leq\beta=(1-\frac{1}{\omega})<1$.

Subsequently, with the initial model variance $I^{'}_{0}$, and after $i$ iterations (at $i$-th query step literally means accumulated $i$ iterations for label acquisition), we have the upper bound for the variance as $I^{'}_{0}G^{i}\leq g(i)=I^{'}_{0}\beta^{i}$, where $g(i)$ obeys the asymptotic behaviour in the following:
\begin{equation}
    \lim_{i\rightarrow\infty} g(i) = \lim_{i\rightarrow\infty} \beta^{i} = 0.
\end{equation} 

Thus, we can conclude that, by acquiring the most uncertain samples that have the maximum predictive variance $\sigma^{2}_{f^{t}}$, the slowest convergence rate is lower-bounded by $\Omega(\beta^{i})$ where $0\leq\beta<1$ since $I^{'}_{0}G^{i}\leq g(i)$.

\end{proof}

\subsection{Convergence Rate of Distributional Discrepancy\label{proof:convergence_wass}}

\begin{definition}
    Let $\mathcal{J}(P,Q)$ be all the joint distribution J for (X,Y) that respectively have the marginal distribution P and Q. Then, the p-Wasserstein distance is defined as:
    \begin{equation}
        W_{p}(P,Q)=\left(\inf_{J\in\mathcal{J}(P,Q)}\int \|x-y\|^{p}dJ(x,y)\right)^{\frac{1}{p}}
    \end{equation}\label{definition:Wasserstein}
\end{definition}

\setcounter{theorem}{2}
\begin{lemma}\label{appendix_lemma:convergence_wass}
    Given two empirical distributions $p^{t=1}$ and $p^{t=0}$ for different treatment groups, the distributional discrepancy given by 1-Wasserstein distance $W_{1}(p^{t=1},p^{t=0})$ has a convergence rate of $\mathcal{O}(\frac{1}{i+\gamma_{0}})$ if the identical samples from two groups can always be found throughout the query steps.
\end{lemma}

\begin{proof}

Given the empirical distribution for P and Q with $N_{i}$ multi-dimensional observations at query step $i$, i.e., $p^{t=1}_{i}$ and $p^{t=0}_{i}$, the 1-Wasserstein distance is analogously by the Definition \ref{definition:Wasserstein} reduced to the following: 
    \begin{equation}
        W_{1}^{i}(p^{t=1}_{i},p^{t=0}_{i})=\inf_{\pi}\left(\sum_{k=1}^{N_{i}} \|x^{t=1}_{k}-x^{t=0}_{\pi(k)}\|\right),\label{eq:discrete_wass}
    \end{equation}

    where the infimum runs over all the possible permutations $\pi$. 

    Since there exists an optimal permutation $\pi^{*}_{i}$ at $i$-th query step, by plugging in $\pi^{*}_{i}$ to (\ref{eq:discrete_wass}), we obtain the 1-Wasserstein distance between two empirical distributions. For now we just denote this value by $W_{1}^{i}(p^{t=1}_{i},p^{t=0}_{i})$ without knowing what exactly the number is. Therefore, for the distributional difference at $i$-th query step, we have two optimal permutation $\pi^{*}_{i-1}$ and $\pi^{*}_{i}$ to help us calculate the difference $\Delta_{\text{IPM}_{i}}(\Tilde{\mathcal{D}}_{i})$:

\begin{subequations}
    \begin{align}
    &\Delta_{\text{IPM}_{i}}(\Tilde{\mathcal{D}}_{i})\\
    =&W^{i-1}_{1}(p^{t=1}_{i-1},p^{t=0}_{i-1}) - W^{i}_{1}(p^{t=1}_{i},p^{t=0}_{i})\\
    =&\frac{1}{N_{i-1}}\sum_{k=1}^{N_{i-1}}\|x^{t=1}_{k}-x^{t=0}_{\pi^{*}_{i-1}(k)}\| - \frac{1}{N_{i-1}+b_{0}}\sum_{k=1}^{N_{i-1}+b_{0}}\|x^{t=1}_{k}-x^{t=0}_{\pi^{*}_{i}(k)}\|\\
    =&\frac{1}{N_{i-1}}\sum_{k=1}^{N_{i-1}}\|x^{t=1}_{k}-x^{t=0}_{\pi^{*}_{i-1}(k)}\| - \\
    &\frac{1}{N_{i-1}+b_{0}}\left(\sum_{k=1}^{N_{i-1}}\|x^{t=1}_{k}-x^{t=0}_{\pi^{*}_{i}(k)}\|+\sum_{l=N_{i-1}+1}^{N_{i-1}+b_{0}}\underbrace{\|x^{t=1}_{l}-x^{t=0}_{\pi^{*}_{i}(l)}\|}_{0}\right)\\
    =&\frac{N_{i-1}+b_{0}}{N_{i-1}(N_{i-1}+b_{0})}\sum_{k=1}^{N_{i-1}}\|x^{t=1}_{k}-x^{t=0}_{\pi^{*}_{i-1}(k)}\| - \\
    &\frac{N_{i-1}}{N_{i-1}(N_{i-1}+b_{0})}\sum_{k=1}^{N_{i-1}}\|x^{t=1}_{k}-x^{t=0}_{\pi^{*}_{i}(k)}\|\\
    =&(\frac{N_{i-1}+b_{0}}{N_{i-1}(N_{i-1}+b_{0})}- \frac{N_{i-1}}{N_{i-1}(N_{i-1}+b_{0})})\sum_{k=1}^{N_{i-1}}\|x^{t=1}_{k}-x^{t=0}_{\pi^{*}_{i-1}(k)}\|\\
    =&\frac{b_{0}}{N_{i-1}+b_{0}}\cdot\frac{1}{N_{i-1}}\sum_{k=1}^{N_{i-1}}\|x^{t=1}_{k}-x^{t=0}_{\pi^{*}_{i-1}(k)}\|,
    \end{align}
\end{subequations} where $b_{0}=|\Tilde{\mathcal{D}}_{i}|/2$, i.e., half of the batch size at each query step. Given that the added identical samples, the optimal permutation $\pi^{*}_{i}$ at $i$-th query step will match these identical pairs due to the cost $\sum_{l=N_{i-1}+1}^{N_{i-1}+b_{0}}\|x^{t=1}_{l}-x^{t=0}_{\pi^{*}(l)}\|=0$ introduce zero distributional discrepancy, thus the forth equality holds. Thus, for the rest of the $N_{i-1}$ samples, the optimal permutation $\pi^{*}_{i}$ must have the same transportation strategy as $\pi^{*}_{i-1}$ does to obtain the lowest cost on the rest $N_{i-1}$ samples, such that the fifth equality holds. We believe the equality claim resonates with the rigorously proved triangular inequality nature of Wasserstein metric \cite{clement2008elementary}.

Therefore, at $i$-th query step, the discrepancy shrinkage $S_{i}$ is defined as:
\begin{equation}
    S_{i}=\frac{\Delta_{\text{IPM}_{i}}(\Tilde{\mathcal{D}}_{i})}{W^{i-1}_{1}(p^{t=1}_{i-1},p^{t=0}_{i-1})}=\frac{b_{0}}{N_{i-1}+b_{0}}
\end{equation}
Since the number of samples in one treatment group (training) at $(i-1)$-th query step, $N_{i-1}$, can be reformulated as $N_{i-1}=\gamma_{i-1}\cdot b_{0}$ with arbitrary $\gamma_{i-1}\in\mathbb{R}^{+}$. Subsequently $S_{i}$ is reduced to:
\begin{equation}
    S_{i}=\frac{b_{0}}{\gamma_{i-1}\cdot b_{0}+b_{0}}=\frac{b_{0}}{b_{0}(\gamma_{i-1} + 1)}=\frac{1}{\gamma_{i-1} + 1}\label{eq:shrinkage_s}
\end{equation}

Thus, we derive the the total shrinking coefficient $S^{i}$ after $i$ iterations as follows:
\begin{subequations}
\begin{align}
        S^{i}&=\prod_{n=1}^{i}\left(1-S_{n}\right)\\
        &=\prod_{n=1}^{i}\left(1-\frac{1}{\gamma_{n-1} + 1}\right)\\
        &=\left(1-\frac{1}{\gamma_{0} + 1}\right)\cdot\left(1-\frac{1}{\gamma_{1} + 1}\right)\cdots\left(1-\frac{1}{\gamma_{i-1} + 1}\right)\\
        &=\left(1-\frac{1}{\gamma_{0} + 1}\right)\cdot\left(1-\frac{1}{(\gamma_{0}+1) + 1}\right)\cdots\left(1-\frac{1}{(\gamma_{0} + i-1) + 1}\right)\\
        &=\frac{\gamma_{0}}{\gamma_{0} + 1}\cdot\frac{\gamma_{0}+1}{(\gamma_{0}+1) + 1}\cdots\frac{\gamma_{0}+i-2}{\gamma_{0} + i -1}\cdot\frac{\gamma_{0}+i-1}{\gamma_{0} + i}\\
        &=\frac{\gamma_{0}}{\gamma_{0} + i},
\end{align} 
\end{subequations} where the arbitrary constant $\gamma_{0}\in\mathbb{R}^{+}$ means the first ratio constant for $N_{0}/b_{0}$, such that the distributional discrepancy can be initialized as $I_{0}=W^{0}_{1}(p^{t=1}_{0},p^{t=0}_{0})$. 

Define $s(i)=I_{0}S^{i}=I_{0}\cdot\gamma_{0}/(\gamma_{0} + i)$, and for $s(i)$ we have the following asymptotic behaviour:
\begin{equation}
    \lim_{i\rightarrow\infty} s(i) = \lim_{i\rightarrow\infty} \frac{I_{0}\gamma_{0}}{\gamma_{0} + i} = 0.
\end{equation} With identical acquisition from both of the treatment groups, we can conclude that the convergence rate of discrepancy is as fast as $\mathcal{O}(\frac{1}{i+\gamma_{0}})$.

\end{proof}

\subsection{Causal Effect Identifiability\label{appendix:identifiability}}

\begin{proposition}[Identifiability\label{appendixproposition:identifiability}]
The causal effect is identifiable if and only if the SUTVA, the unconfoundedness, and the positivity assumptions hold.
\end{proposition}

\begin{proof}

Under SUTVA (\textbf{Assumption} \ref{assumption:sutva}) and unconfoundedness (\textbf{Assumption} \ref{assumption:unconf}), the ITE for instance $i$ with cocvariate $\textbf{x}_{i}$ is:
\begin{equation}
    \begin{split}
        \mathbb{E}[Y^{t=1}\!-\!Y^{t=0}|\textbf{x}_{i}] =& \mathbb{E}[Y^{t=1}|\textbf{x}_{i}] - \mathbb{E}[Y^{t=0}|\textbf{x}_{i}]\\
        =&\mathbb{E}[Y^{t=1}|\textbf{x}_{i},t_i\!=\!1]\!-\!\mathbb{E}[Y^{t=0}|\textbf{x}_{i},t_i\!=\!0]\\
        =&\mathbb{E}[y_i|\textbf{x}_{i},t_i=1]-\mathbb{E}[y_i|\textbf{x}_{i},t_i=0],
    \end{split}
\end{equation}
where $y_i$ denotes the observed outcomes after the intervention $t=1/0$ has been taken. The first equality is the rewritten expectation, the second equality is based on the unconfoundedness, and the third equality states that the expected values of the observed outcomes $\{y_{1}, y_{0}\}$ equal the unobserved potential outcomes. The last two terms are identifiable as we assume $0<p(t=1|\textbf{x})<1$ (\textbf{Assumption} \ref{assumption:positivity}).

\end{proof}

\subsection{Factual Error Decomposition\label{appendix:variance_bias}}

\begin{definition}
    The expected treatment risk with status $t$ in terms of the expected squared loss function $\ell$(x,t) with density $p^{t}(x)$ are defined respectively as follows:
    \begin{equation}
    \begin{split}
        \epsilon_{F}^{t=1} = \int_{\mathcal{X}} \ell(x,1)p^{t=1}(x)dx,~
        \epsilon_{F}^{t=0} = \int_{\mathcal{X}} \ell(x,0)p^{t=0}(x)dx
    \end{split}
    \end{equation}
\end{definition}

\begin{proposition} Assume the potential effect $y^{t}$ has the form $y^{t}=f^{t}(x)+\xi^{t}$, and the estimated effect is set to $\hat{y}^{t}=\hat{f}^{t}(x)$. With the expected square loss fucntion $\ell(x,t)=\mathbb{E}[(y-\hat{y})^{2}]$ The expected risk for either of the treatment groups can be decomposed into the following:
\begin{equation}
    \begin{split}
        \epsilon_{F}^{t}
        =\mathbb{E}_{\mathcal{X}}[(f^{t}(x) - \mathbb{E}[\hat{f}^{t}(x)])^{2}] + \mathbb{E}_{\mathcal{X}}[\mathbb{E}[(\hat{f}^{t}(x) - \mathbb{E}[\hat{f}^{t}(x)])^{2}]]+\sigma_{\xi^{t}}^{2}
    \end{split}
\end{equation} where $f^{t}(x)$ is the true function, $\xi^{t}$ is the noise with 0 mean and constant variance $\sigma_{\xi^{t}}^{2}$, and $\hat{f}^{t}(x)$ is the approximation.
\label{appendixprop:bias-variance}
\end{proposition}

\begin{proof} We start the proof toward a single data point $x$, then simply extending to the domain $\mathcal{X}$ can conclude the proof.
    \begin{subequations}
        \begin{align}
            &\ell(x,t)\\
            =&\mathbb{E}[(y^{t}-\hat{y}^{t})^{2}]\\
            =&\mathbb{E}[(f^{t}(x)+\xi^{t} - \hat{f}^{t}(x))^{2}]\\
            =&\mathbb{E}[(f^{t}(x) - \hat{f}^{t}(x) + \xi^{t})^{2}]\\
            =&\mathbb{E}[(f^{t}(x)-\hat{f}^{t}(x))^{2} + 2(f^{t}(x)-\hat{f}^{t}(x))\xi^{t} + (\xi^{t})^{2})] \\
            =&\mathbb{E}[(f^{t}(x)-\hat{f}^{t}(x))^{2}]+2\mathbb{E}[(f^{t}(x)-\hat{f}^{t}(x))]\underbrace{\mathbb{E}[\xi^{t}]}_{0}+\mathbb{E}[(\xi^{t})^{2}]\\
            =&\mathbb{E}[(f^{t}(x)-\mathbb{E}[\hat{f}^{t}(x)]+\mathbb{E}[\hat{f}^{t}(x)]-\hat{f}^{t}(x))^{2}]+ \sigma_{\xi^{t}}^{2}\\
            =&\mathbb{E}[(f^{t}(x) - \mathbb{E}[\hat{f}^{t}(x)])^{2}]+\mathbb{E}[(\mathbb{E}[\hat{f}^{t}(x)]-\hat{f}^{t}(x))^{2}]+\\
            &2\mathbb{E}[(\underbrace{f^{t}(x) - \mathbb{E}[\hat{f}^{t}(x)]}_{\text{Constant w.r.t x}})(\underbrace{\mathbb{E}[\hat{f}^{t}(x)]-\hat{f}^{t}(x)}_{\text{0 expectation}})]+\sigma_{\xi^{t}}^{2}\\
            =&\underbrace{(f^{t}(x) - \mathbb{E}[\hat{f}^{t}(x)])^{2}}_{\text{Bias}} + \underbrace{\mathbb{E}[(\hat{f}^{t}(x) - \mathbb{E}[\hat{f}^{t}(x)])^{2}]}_{\text{Variance}} + \sigma_{\xi^{t}}^{2}
        \end{align}
    \end{subequations}

    Then, the expected factual loss across the domain $\mathcal{X}$ is:
    \begin{equation}
    \begin{split}
        &\int_{X}\ell(x,t)p(x)dx\\
        =&\int_{\mathcal{X}}\left((f^{t}(x) - \mathbb{E}[\hat{f}^{t}(x)])^{2} + \mathbb{E}[(\hat{f}^{t}(x) - \mathbb{E}[\hat{f}^{t}(x)])^{2}] + \sigma_{\xi^{t}}^{2}\right)p(x)dx\\
        =&\underbrace{\int_{\mathcal{X}}(f^{t}(x) - \mathbb{E}[\hat{f}^{t}(x)])^{2}p^{t}(x)dx}_{\text{Expected Bias}}+\\
        &\underbrace{\int_{\mathcal{X}}\mathbb{E}[(\hat{f}^{t}(x) - \mathbb{E}[\hat{f}^{t}(x)])^{2}]p^{t}(x)dx}_{\text{Expected Variance}} + \int_{\mathcal{X}}\sigma_{\xi^{t}}^{2}p^{t}(x)dx\\
        &=\mathbb{E}_{\mathcal{X}}[(f^{t}(x) - \mathbb{E}[\hat{f}^{t}(x)])^{2}] + \mathbb{E}_{\mathcal{X}}[\mathbb{E}[(\hat{f}^{t}(x) - \mathbb{E}[\hat{f}^{t}(x)])^{2}]]+\sigma_{\xi^{t}}^{2}
    \end{split}
    \end{equation}
\end{proof}

\subsection{Algorithm\label{appendix:algorithm}}

The full algorithm of MACAL is concluded in Algorithm \ref{appendixalg:macal} with additional consideration to the scenario where one of the treatment group's pool set is exhausted, such that the active learning process down to one-sided label acquisition without the accessibility to the other side.

\begin{algorithm*}[h!]
\caption{MACAL\label{appendixalg:macal}}
\begin{algorithmic}[1]
    \State \textbf{Input:} Initializing the training set $\mathcal{D}_{\text{train}}=\mathcal{D}^{t=1}_{\text{train}}\cup\mathcal{D}^{t=0}_{\text{train}}$, pool set $\mathcal{D}_{\text{pool}}=\mathcal{D}^{t=1}_{\text{pool}}\cup\mathcal{D}^{t=0}_{\text{pool}}$, batch size $S$, initial query step $Q=1$, symmetrical penalization $\alpha$, and maximum query step $Q_{\text{max}}$.
    \For{$Q<Q_{\text{max}}$}
    \State $\Tilde{\mathcal{D}}=\varnothing$
    \While{$|\Tilde{\mathcal{D}}|<S$} \hfill \(\triangleright\) \textcolor{gray}{\textit{Keep acquiring the samples if the batch is not filled up in each query step}}
    \If{$\mathcal{D}^{t=1}_{\text{pool}}\neq\varnothing \textbf{ and } \mathcal{D}^{t=0}_{\text{pool}}\neq\varnothing$} \hfill \(\triangleright\) \textcolor{gray}{\textit{If both treatment groups still have available unlabelled samples}}
    \State $\Tilde{x}^{t=1},\Tilde{x}^{t=0} \leftarrow\argmax_{(\Tilde{x}^{t=1},\Tilde{x}^{t=0})} \Sigma_{t\in\{0,1\}}\text{min}_{x'\in\mathcal{D}^{t}_{\text{train}}}d(\Tilde{x}^{t}, x')-\alpha d(\Tilde{x}^{t=1},\Tilde{x}^{t=0})$ \hfill \(\triangleright\) \textcolor{gray}{\textit{Acquire the sample pairs}}
    \State $\mathcal{D}^{t=1}_{\text{train}}, \mathcal{D}^{t=0}_{\text{train}}\leftarrow\mathcal{D}^{t=1}_{\text{train}}\cup\{\Tilde{x}^{t=1}\}, \mathcal{D}^{t=0}_{\text{train}}\cup\{\Tilde{x}^{t=0}\}$ \hfill \(\triangleright\) \textcolor{gray}{\textit{Add the unlabelled samples into the training sets}}
    \Statex \hspace{1.35cm} $\mathcal{D}^{t=1}_{\text{pool}},\mathcal{D}^{t=0}_{\text{pool}}\leftarrow\mathcal{D}^{t=1}_{\text{pool}}\backslash\{\Tilde{x}^{t=1}\}, \mathcal{D}^{t=0}_{\text{pool}}\backslash\{\Tilde{x}^{t=0}\}$ \hfill \(\triangleright\) \textcolor{gray}{\textit{Exclude the acquired samples from the pool sets}}
    \Statex \hspace{1.35cm} $\Tilde{\mathcal{D}}\leftarrow\Tilde{\mathcal{D}}\cup\{\Tilde{x}^{t=1},\Tilde{x}^{t=0}\}$ \hfill \(\triangleright\) \textcolor{gray}{\textit{Update the acquired batch}}
    \Else{} \hfill \(\triangleright\) \textcolor{gray}{\textit{If one of the treatment groups' pool set is exhausted}}
    \If{$\mathcal{D}^{t=1}_{\text{pool}}\neq\varnothing$} \hfill \(\triangleright\) \textcolor{gray}{\textit{If treatment group with $t=1$ is not exhausted}}
    \State $\Tilde{x}^{t=1} \leftarrow\argmax_{\Tilde{x}^{t=1}\in\mathcal{D}^{t=1}} \text{min}_{x'\in\mathcal{D}^{t=1}_{\text{train}}}d(\Tilde{x}^{t}, x')$
    \State $\mathcal{D}^{t=1}_{\text{train}}, \mathcal{D}^{t=1}_{\text{pool}}, \Tilde{\mathcal{D}}\leftarrow\mathcal{D}^{t=1}_{\text{train}}\cup\{\Tilde{x}^{t=1}\}, \mathcal{D}^{t=1}_{\text{pool}}\backslash\{\Tilde{x}^{t=1}\},\Tilde{\mathcal{D}}\cup\{\Tilde{x}^{t=1}\}$
    \Else{} \hfill \(\triangleright\) \textcolor{gray}{\textit{If treatment group with $t=0$ is not exhausted}}
    \State $\Tilde{x}^{t=0} \leftarrow\argmax_{\Tilde{x}^{t=0}\in\mathcal{D}^{t=0}} \text{min}_{x'\in\mathcal{D}^{t=0}_{\text{train}}}d(\Tilde{x}^{t}, x')$  
    \State $\mathcal{D}^{t=0}_{\text{train}}, \mathcal{D}^{t=0}_{\text{pool}}, \Tilde{\mathcal{D}}\leftarrow\mathcal{D}^{t=1}_{\text{train}}\cup\{\Tilde{x}^{t=0}\}, \mathcal{D}^{t=0}_{\text{pool}}\backslash\{\Tilde{x}^{t=0}\},\Tilde{\mathcal{D}}\cup\{\Tilde{x}^{t=0}\}$
    \EndIf
    \EndIf
    \EndWhile
    \State Reveal the acquired unlabelled samples in the training set via the Oracle. \hfill \(\triangleright\) \textcolor{gray}{\textit{Label all the acquired samples all at once}}
    \State $Q\leftarrow Q+1$ \hfill \(\triangleright\) \textcolor{gray}{\textit{Move to next query step}}
    \EndFor
    \State \textbf{Output:} $\mathcal{D}_{\text{train}}$.
\end{algorithmic}
\end{algorithm*}

\section{Additional Experiments and Setup}

\subsection{Additional Visualizations of Post-Acquisition Dataset\label{appendix:visualization}}

We visualize the post-acquisition dataset distribution via t-SNE for the five most representative models: MACAL, Random, LCMD, QHTE, and $\mu\rho$BALD, on all three datasets, i.e., CMNIST, IBM, and IHDP. Across all Figure \ref{appendix:tsne_cmnist}, \ref{appendix:tsne_ibm}, and \ref{appendix:tsne_ihdp}, we consistently observe that MACAL can significantly outperform the other methods in terms of acquiring the pairs to avoid the violation of positivity, and also expanding the data boundary for not being clustering at a small area to avoid repetitive samples. Interestingly, we notice that in the IHDP dataset, the pair acquisition by MACAL terminates at Step 15 as shown in Figure \ref{appendix:tsne_ibm_macal_15} since samples from the treatment group with $t=1$ are exhausted, such that, MACAL can only label the other treatment samples and leave a quite imbalanced dataset at Step 35 as shown in Figure \ref{appendix:tsne_ibm_macal_35}.

\begin{figure*}[h]
  \centering
  \subfigure[MACAL at Step 10]{\includegraphics[width=0.25\textwidth]{figures/appendix_tsne/truesim_2.5_9.pdf}}
  \subfigure[MACAL at Step 30]{\includegraphics[width=0.25\textwidth]{figures/appendix_tsne/truesim_2.5_29.pdf}}
  \subfigure[MACAL at Step 50]{\includegraphics[width=0.255\textwidth]{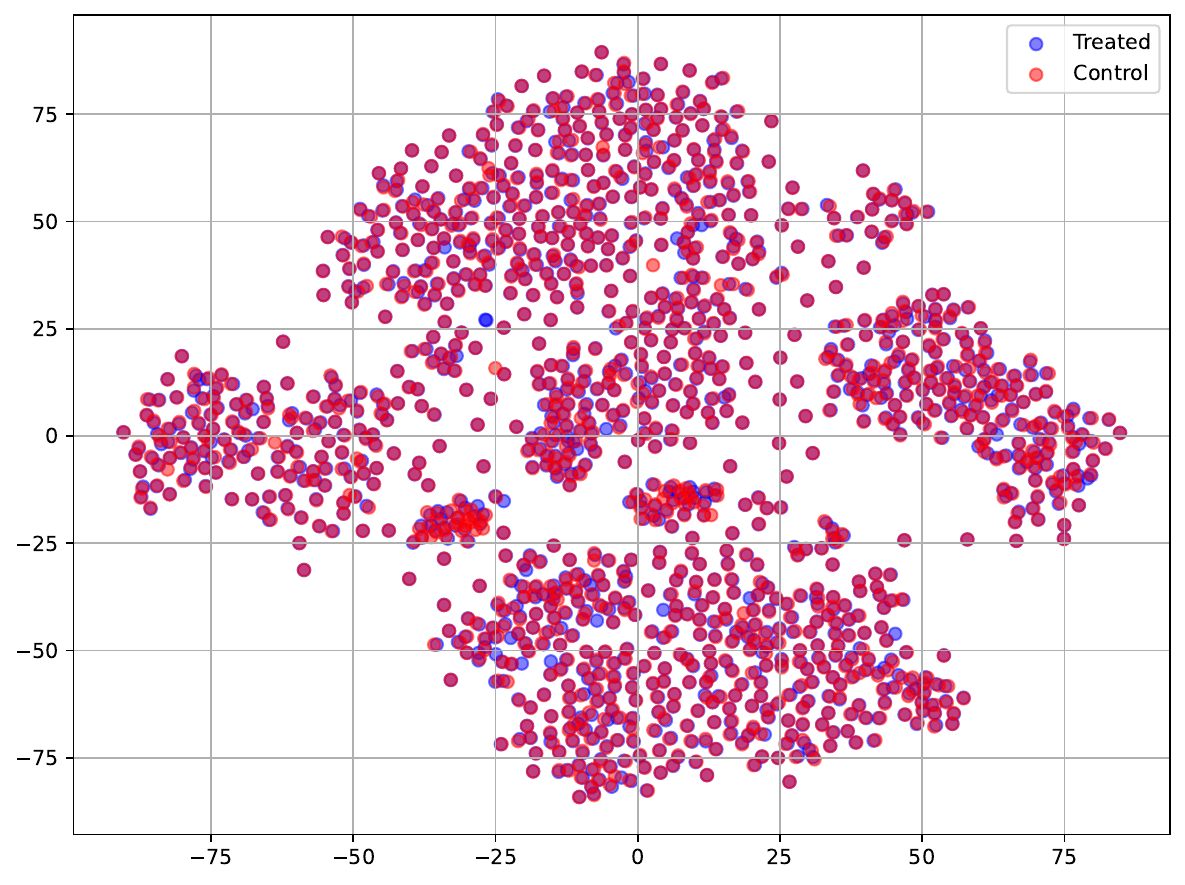}}\\
  \subfigure[Random at Step 10]{\includegraphics[width=0.25\textwidth]{figures/appendix_tsne/truerandom_9.pdf}}
  \subfigure[Random at Step 30]{\includegraphics[width=0.25\textwidth]{figures/appendix_tsne/truerandom_29.pdf}}
  \subfigure[Random at Step 50]{\includegraphics[width=0.255\textwidth]{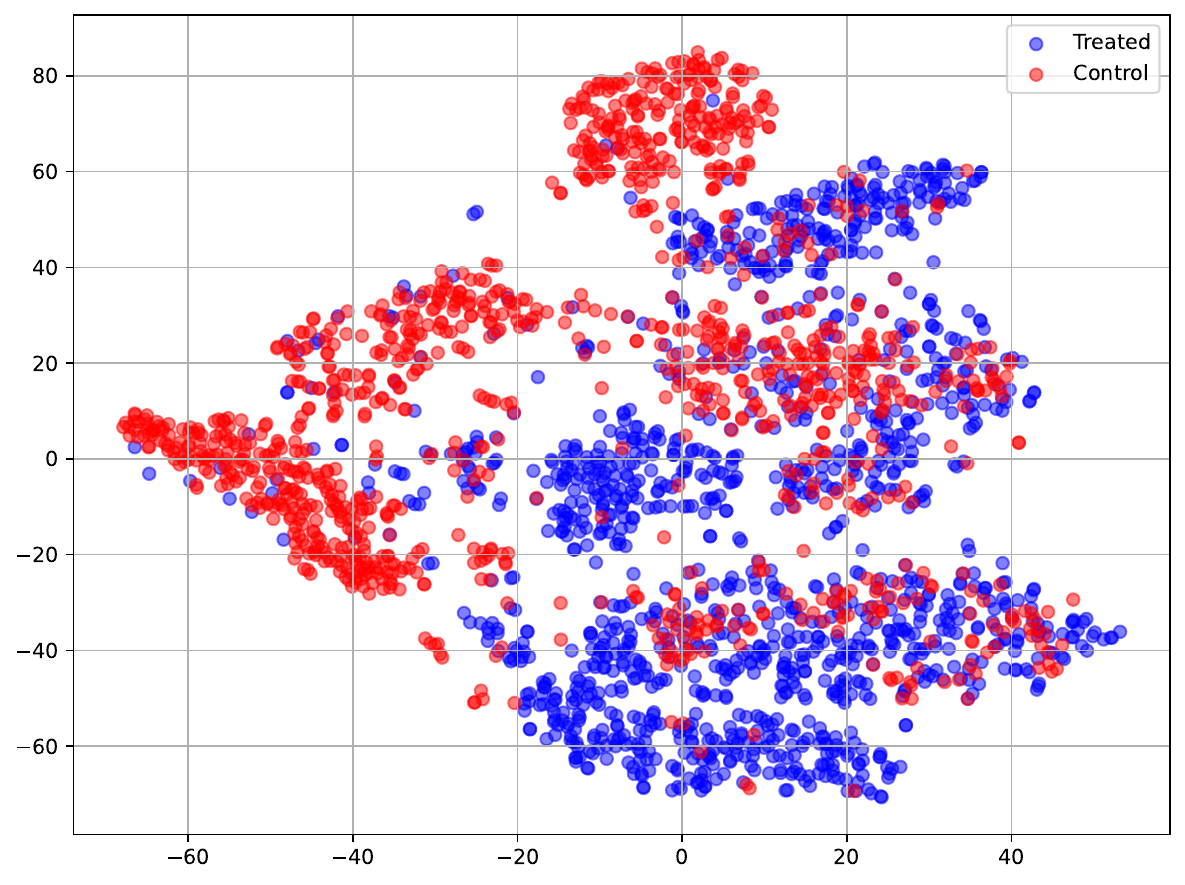}}\\
  \subfigure[LCMD at Step 10]{\includegraphics[width=0.25\textwidth]{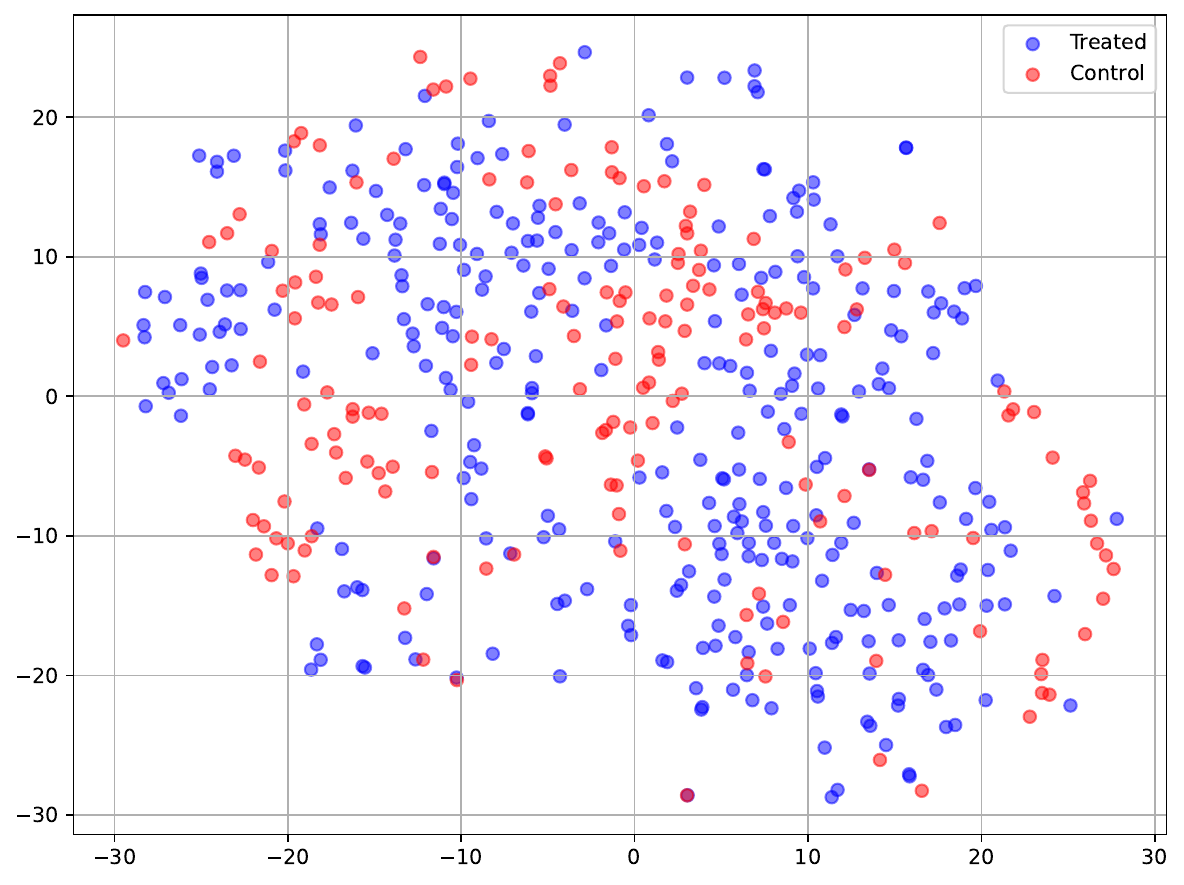}}
  \subfigure[LCMD at Step 30]{\includegraphics[width=0.25\textwidth]{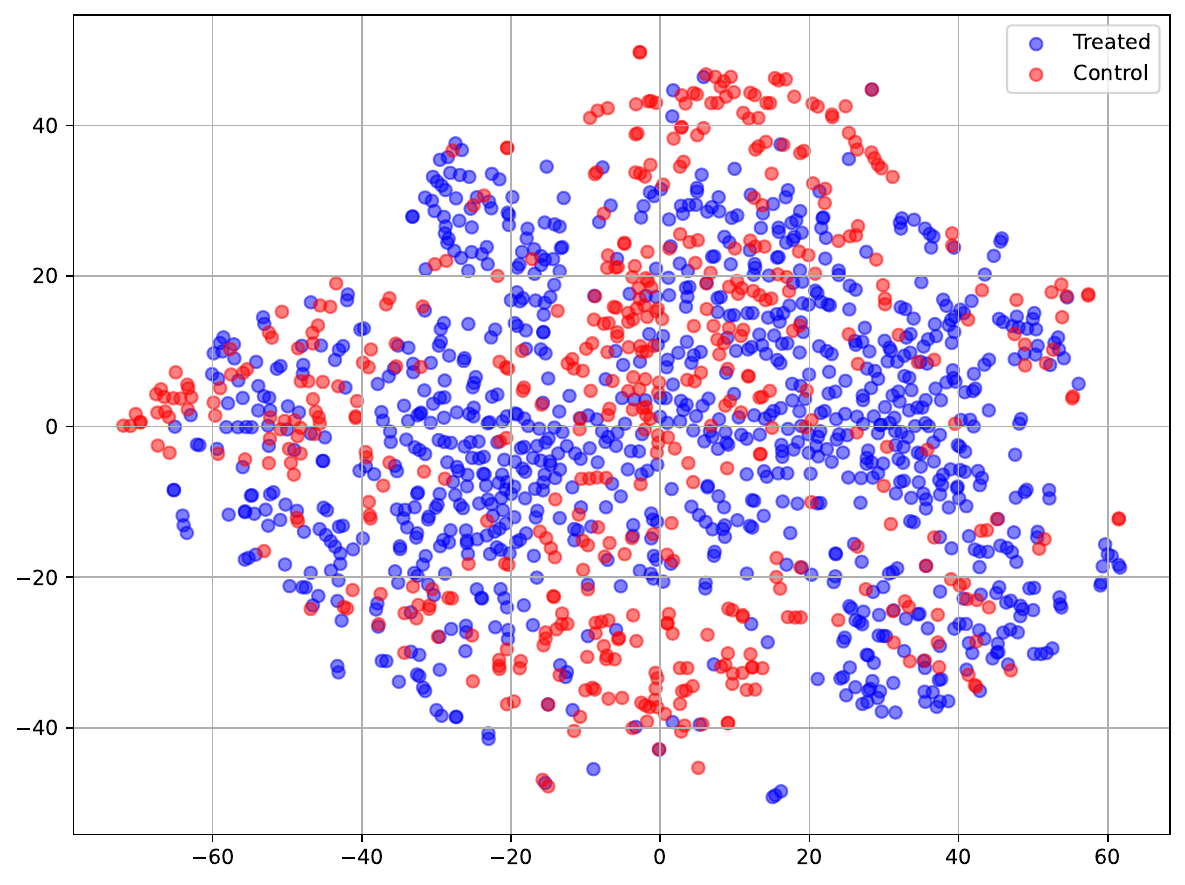}}
  \subfigure[LCMD at Step 50]{\includegraphics[width=0.255\textwidth]{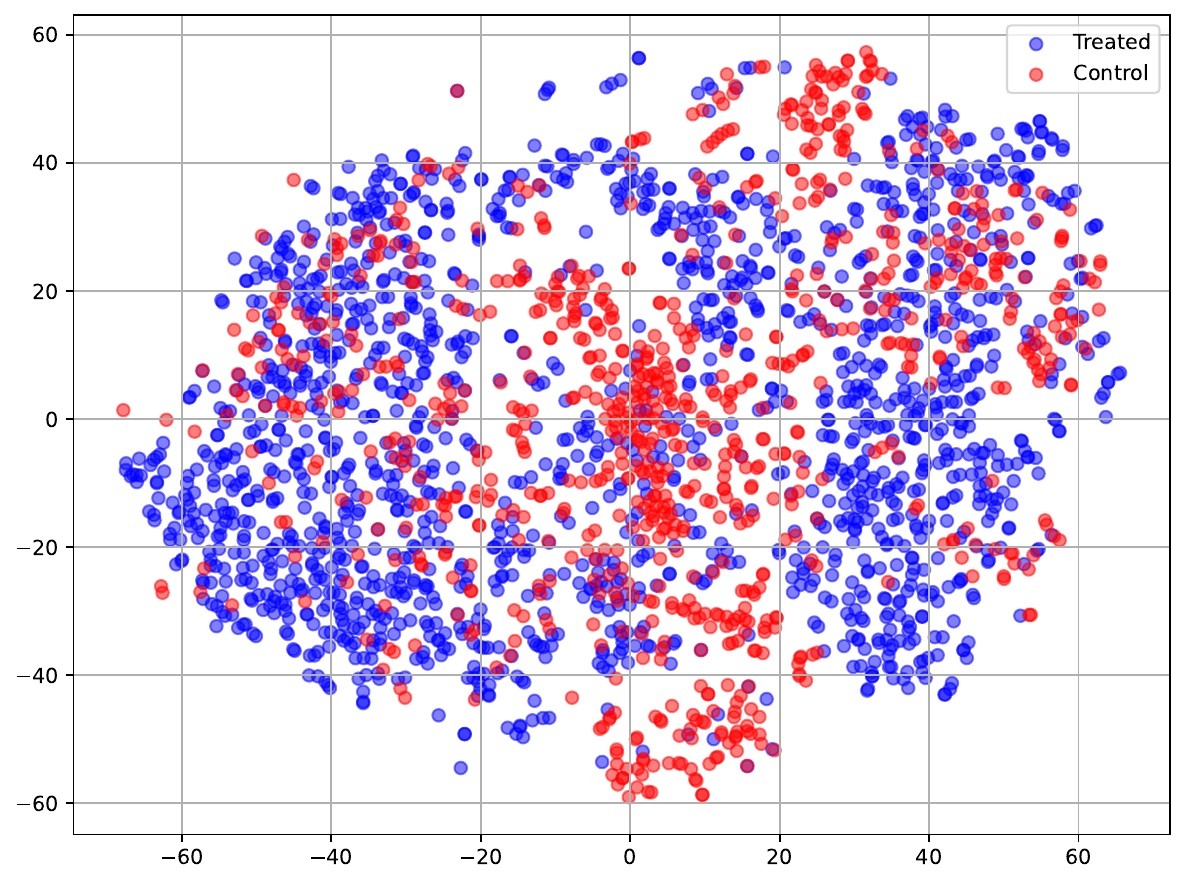}}\\
    \subfigure[QHTE at Step 10]{\includegraphics[width=0.25\textwidth]{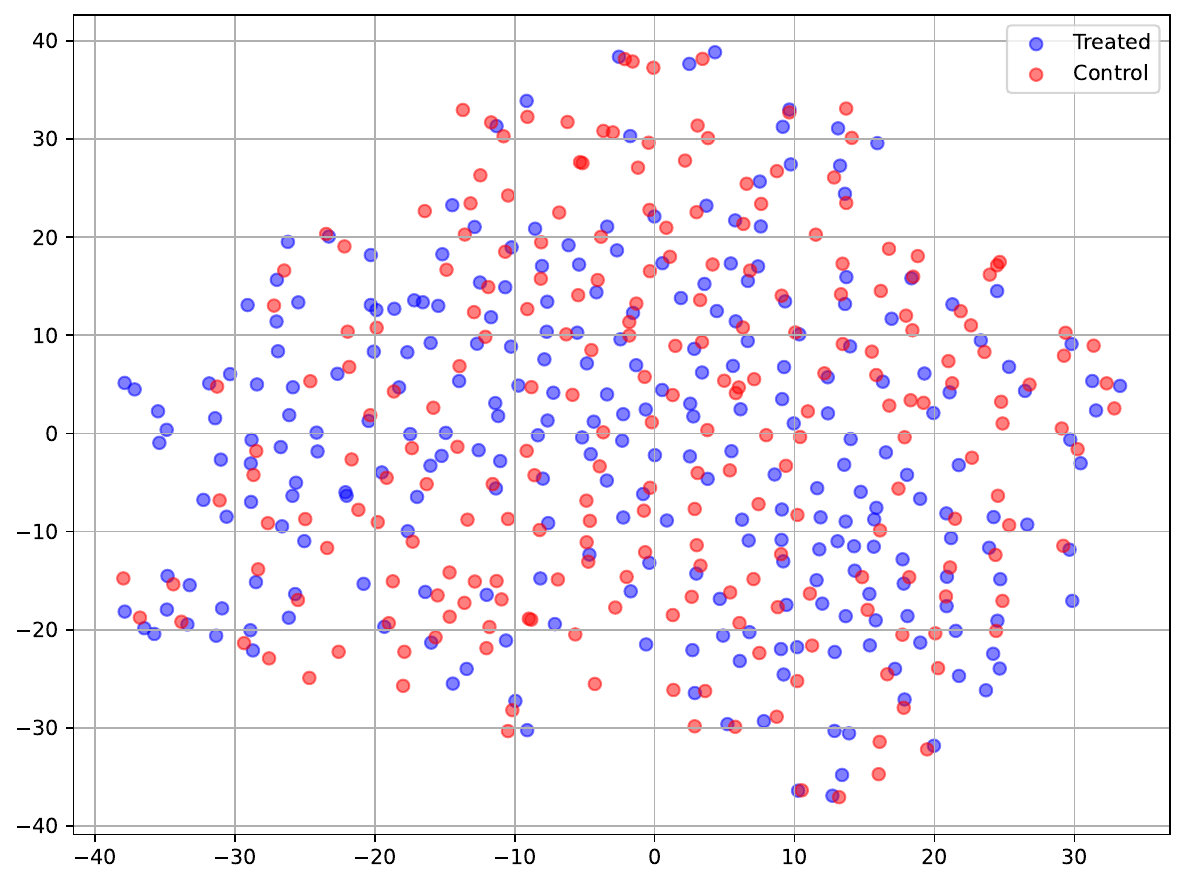}}
  \subfigure[QHTE at Step 30]{\includegraphics[width=0.25\textwidth]{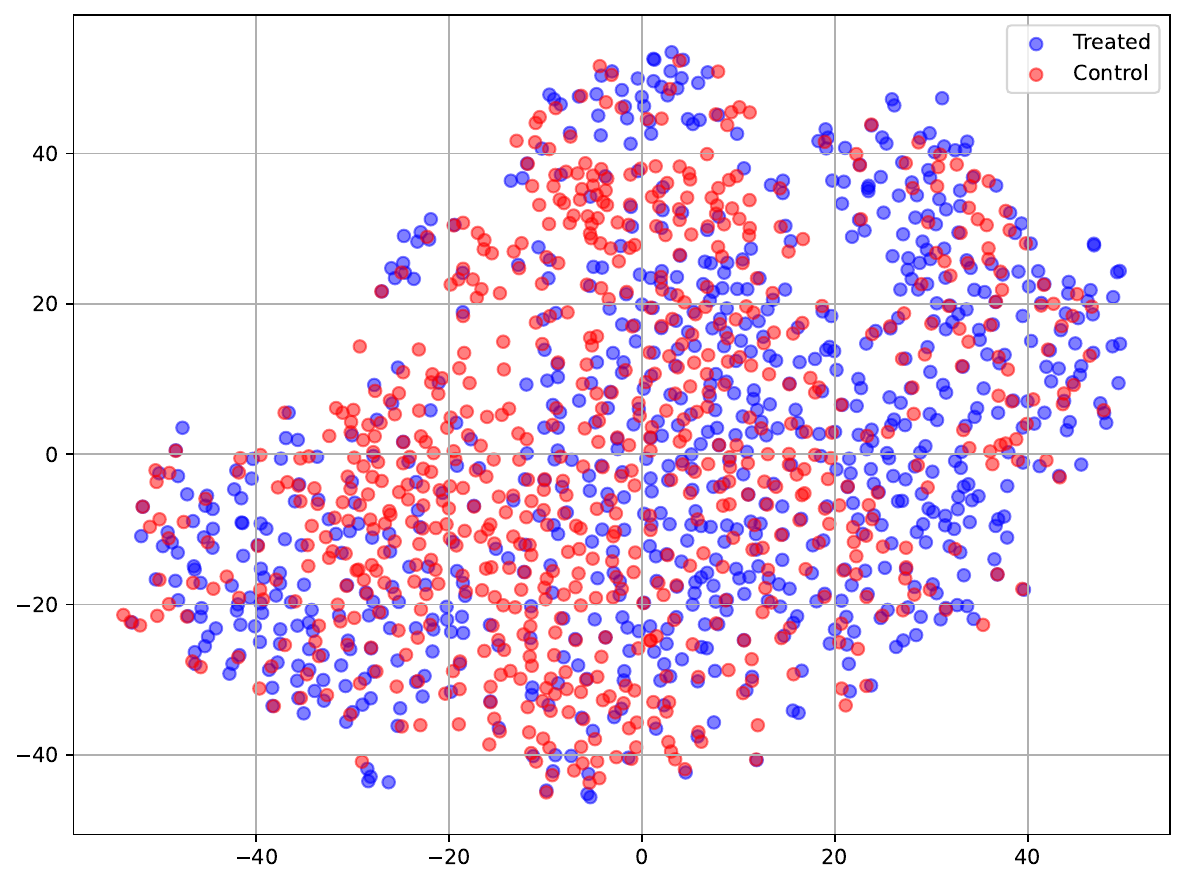}}
  \subfigure[QHTE at Step 50]{\includegraphics[width=0.255\textwidth]{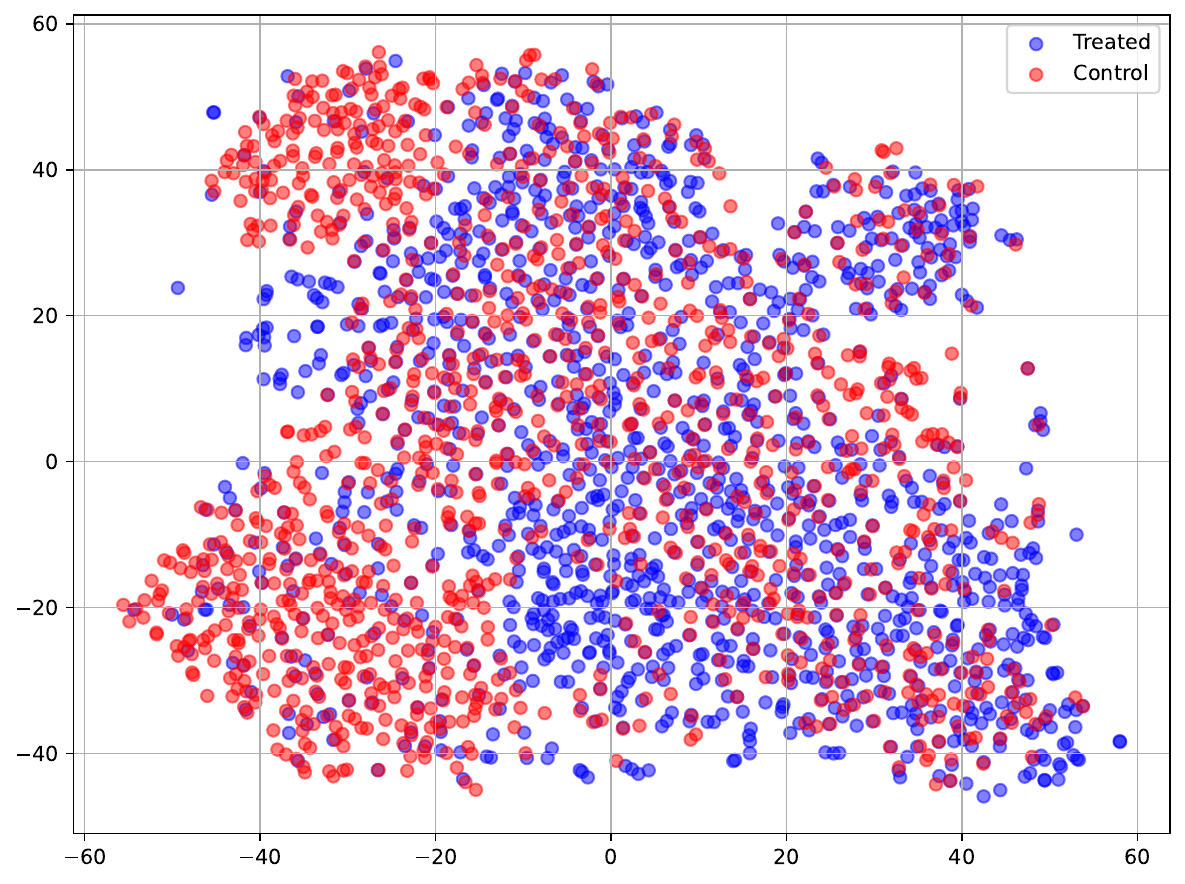}}\\
    \subfigure[$\mu\rho$BALD at Step 10]{\includegraphics[width=0.25\textwidth]{figures/appendix_tsne/truemurho_9.pdf}}
  \subfigure[$\mu\rho$BALD at Step 30]{\includegraphics[width=0.25\textwidth]{figures/appendix_tsne/truemurho_29.pdf}}
  \subfigure[$\mu\rho$BALD at Step 50]{\includegraphics[width=0.255\textwidth]{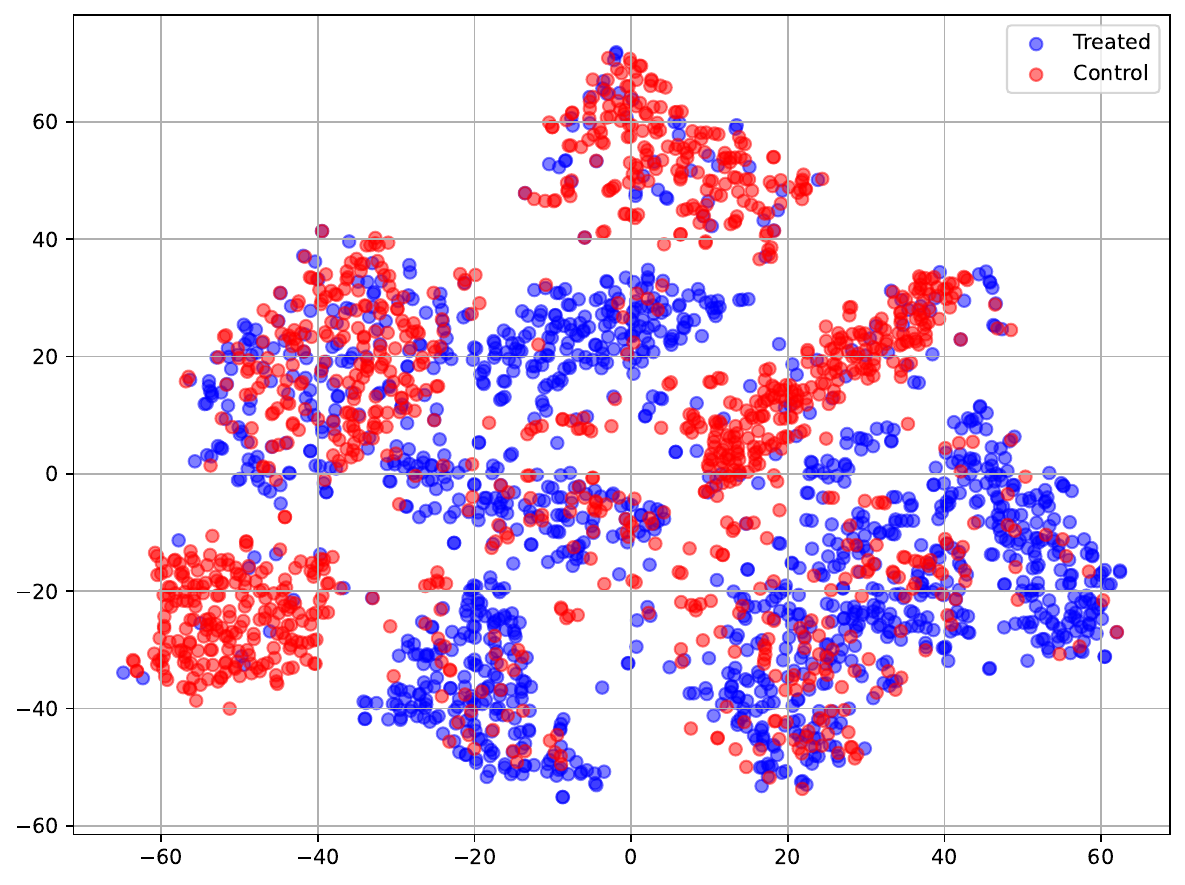}}
  
  \caption{Visualization of the post-acquisition training set at query step 10, 30, and 50 via t-SNE on CMNIST dataset.}.\label{appendix:tsne_cmnist}
\end{figure*}

\begin{figure*}[h]
  \centering
  \subfigure[MACAL at Step 10]{\includegraphics[width=0.25\textwidth]{figures/appendix_tsne/ibm_truesim_2.5_9.pdf}}
  \subfigure[MACAL at Step 30]{\includegraphics[width=0.25\textwidth]{figures/appendix_tsne/ibm_truesim_2.5_29.pdf}}
  \subfigure[MACAL at Step 50]{\includegraphics[width=0.255\textwidth]{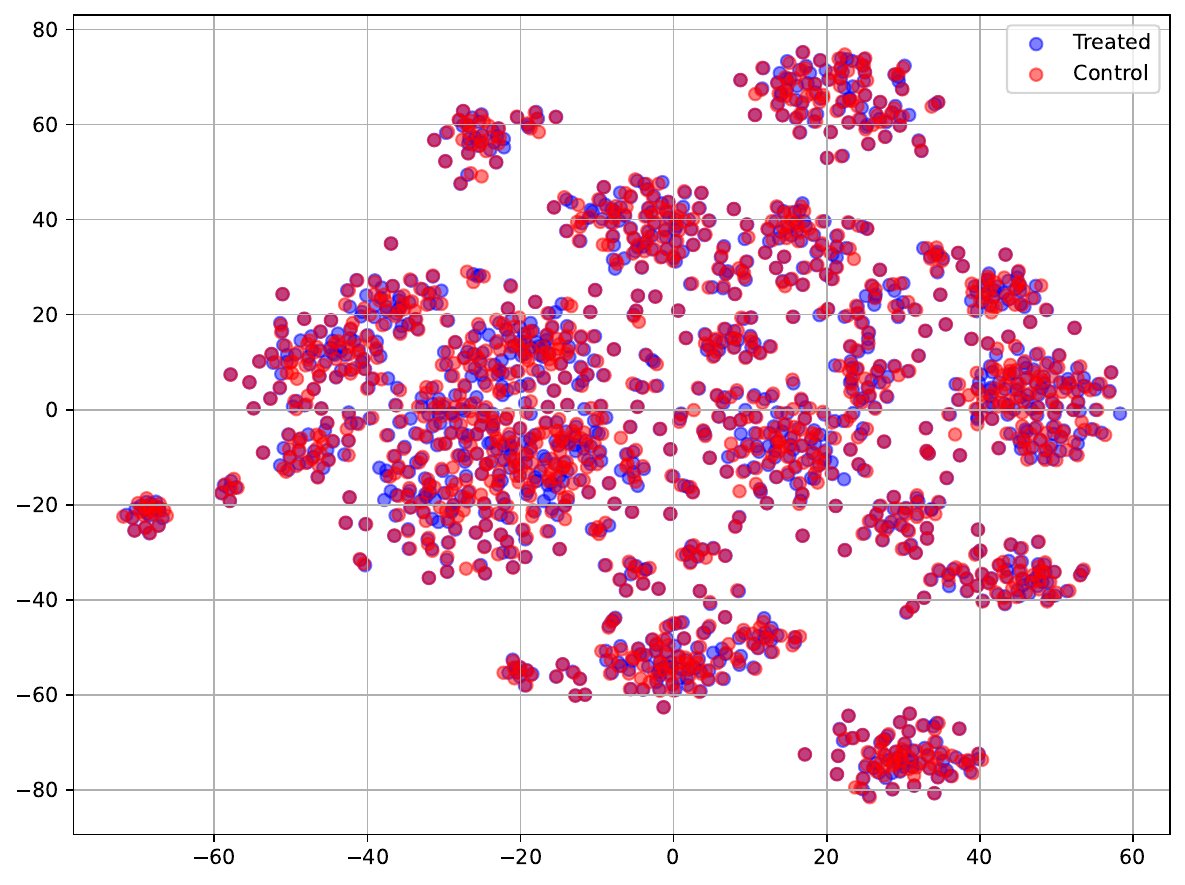}}\\
  \subfigure[Random at Step 10]{\includegraphics[width=0.25\textwidth]{figures/appendix_tsne/ibm_truerandom_9.pdf}}
  \subfigure[Random at Step 30]{\includegraphics[width=0.25\textwidth]{figures/appendix_tsne/ibm_truerandom_29.pdf}}
  \subfigure[Random at Step 50]{\includegraphics[width=0.255\textwidth]{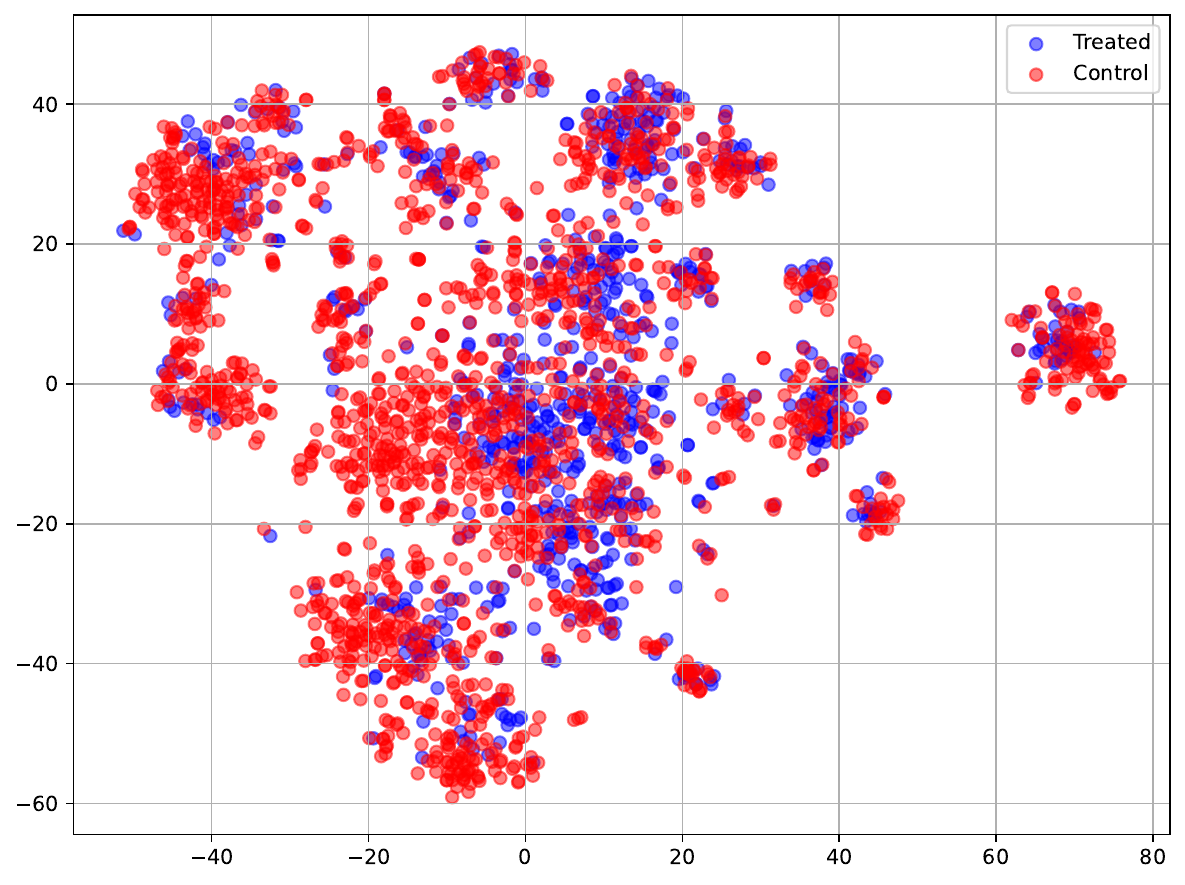}}\\
  \subfigure[LCMD at Step 10]{\includegraphics[width=0.25\textwidth]{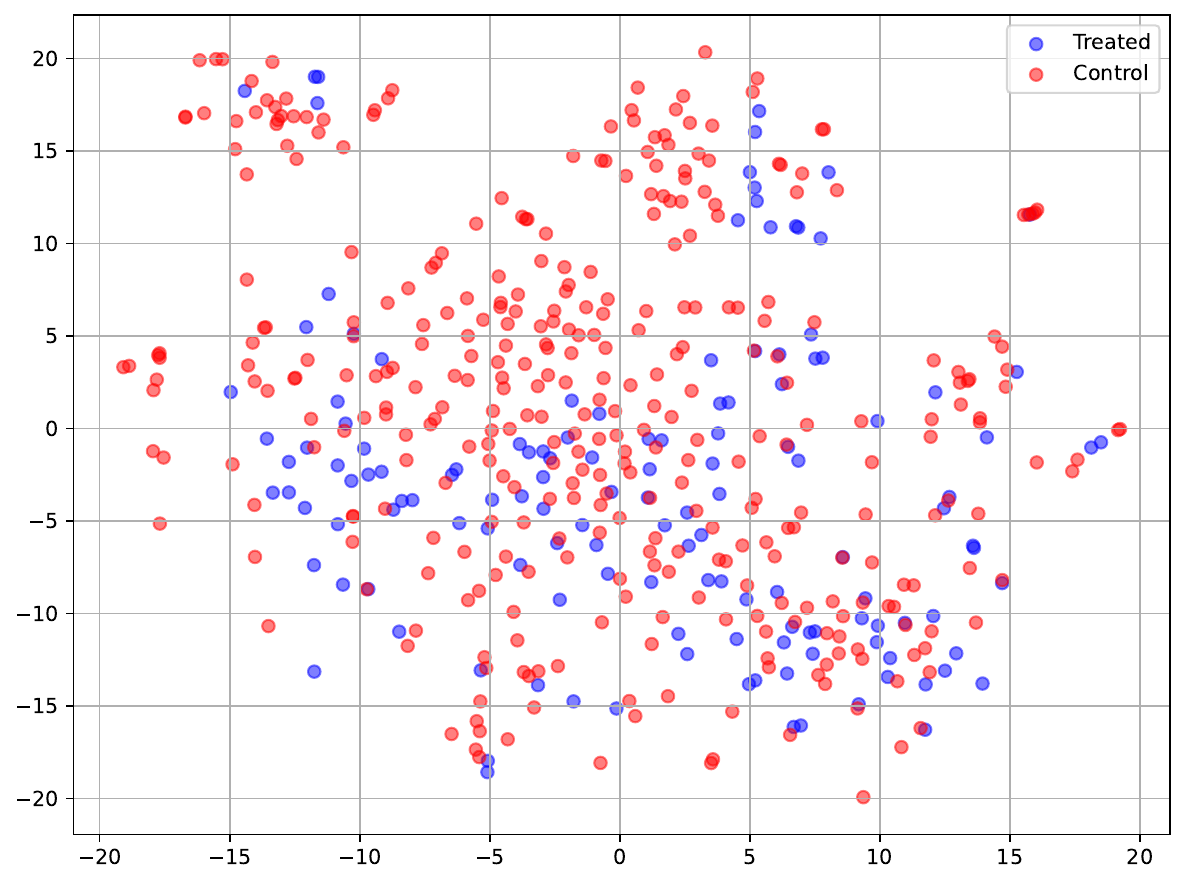}}
  \subfigure[LCMD at Step 30]{\includegraphics[width=0.25\textwidth]{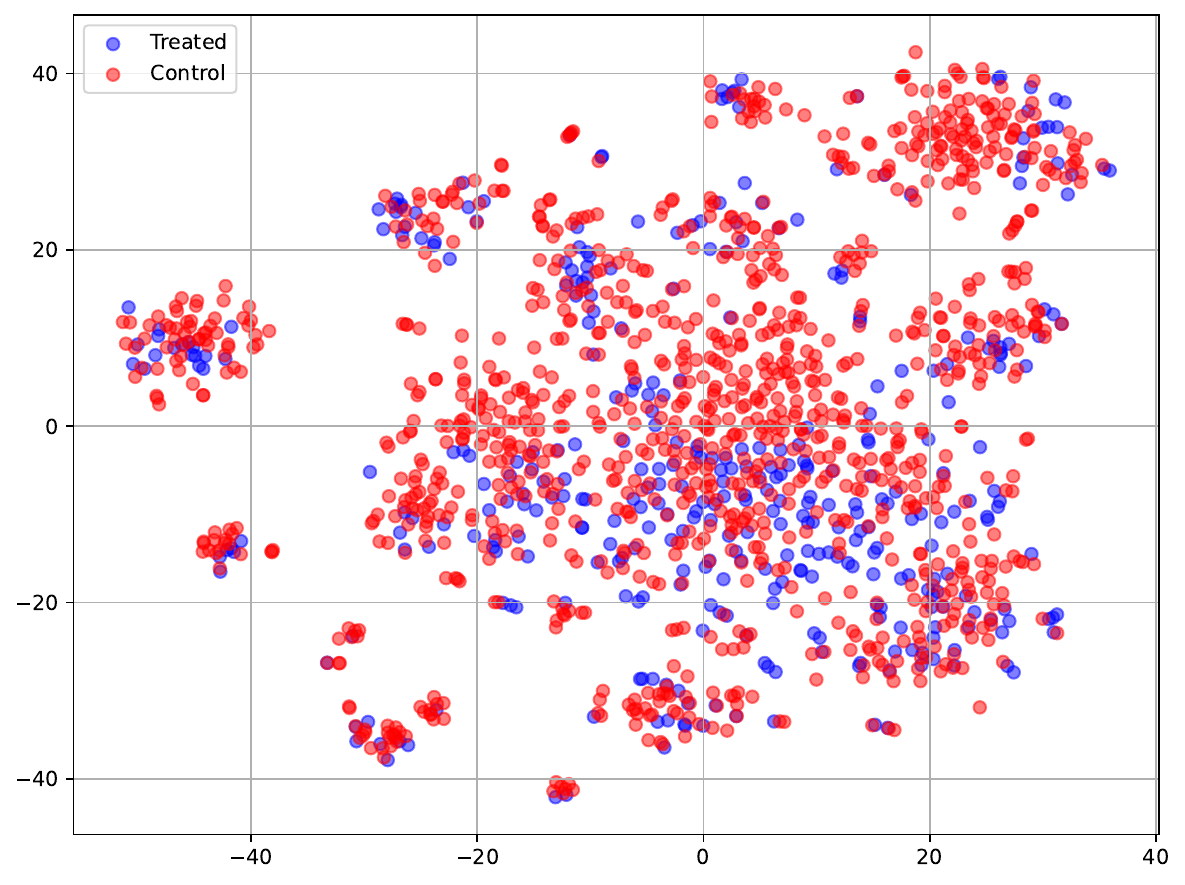}}
  \subfigure[LCMD at Step 50]{\includegraphics[width=0.255\textwidth]{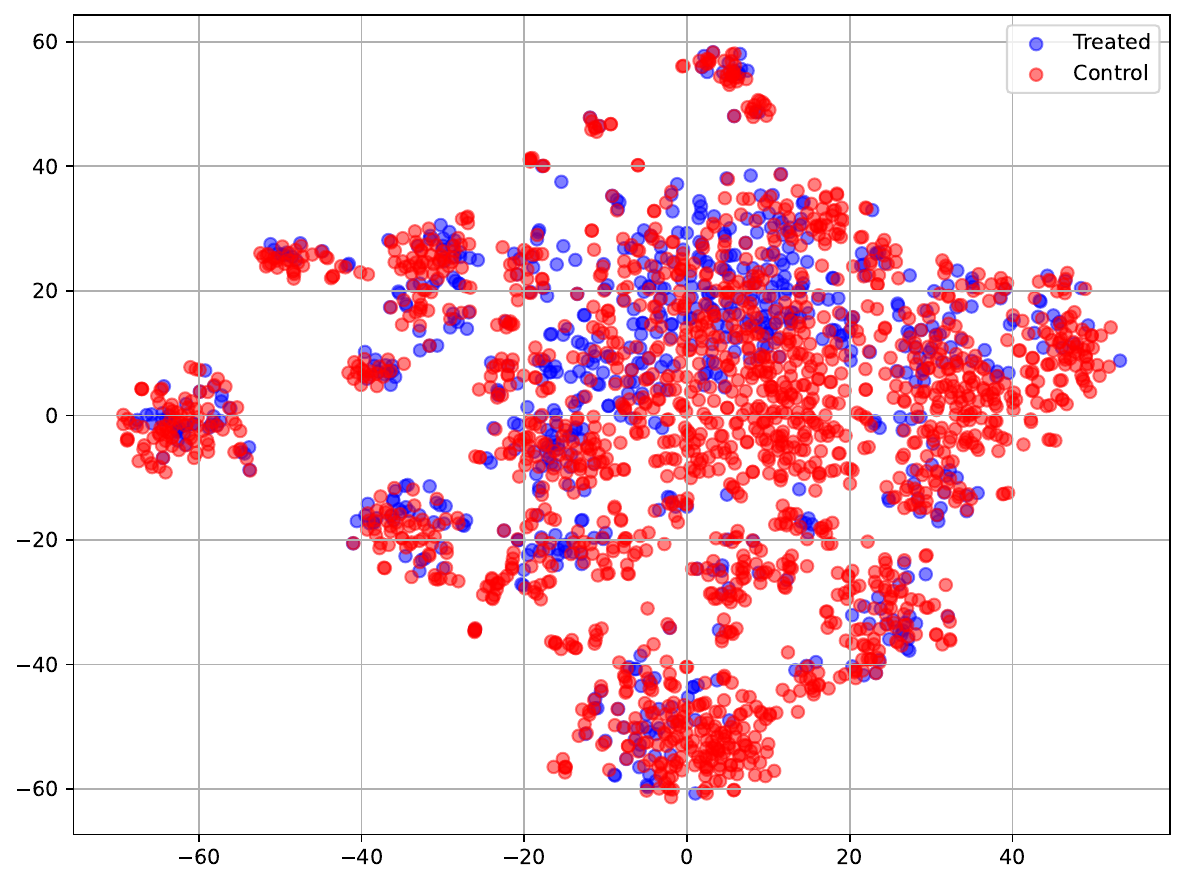}}\\
    \subfigure[QHTE at Step 10]{\includegraphics[width=0.25\textwidth]{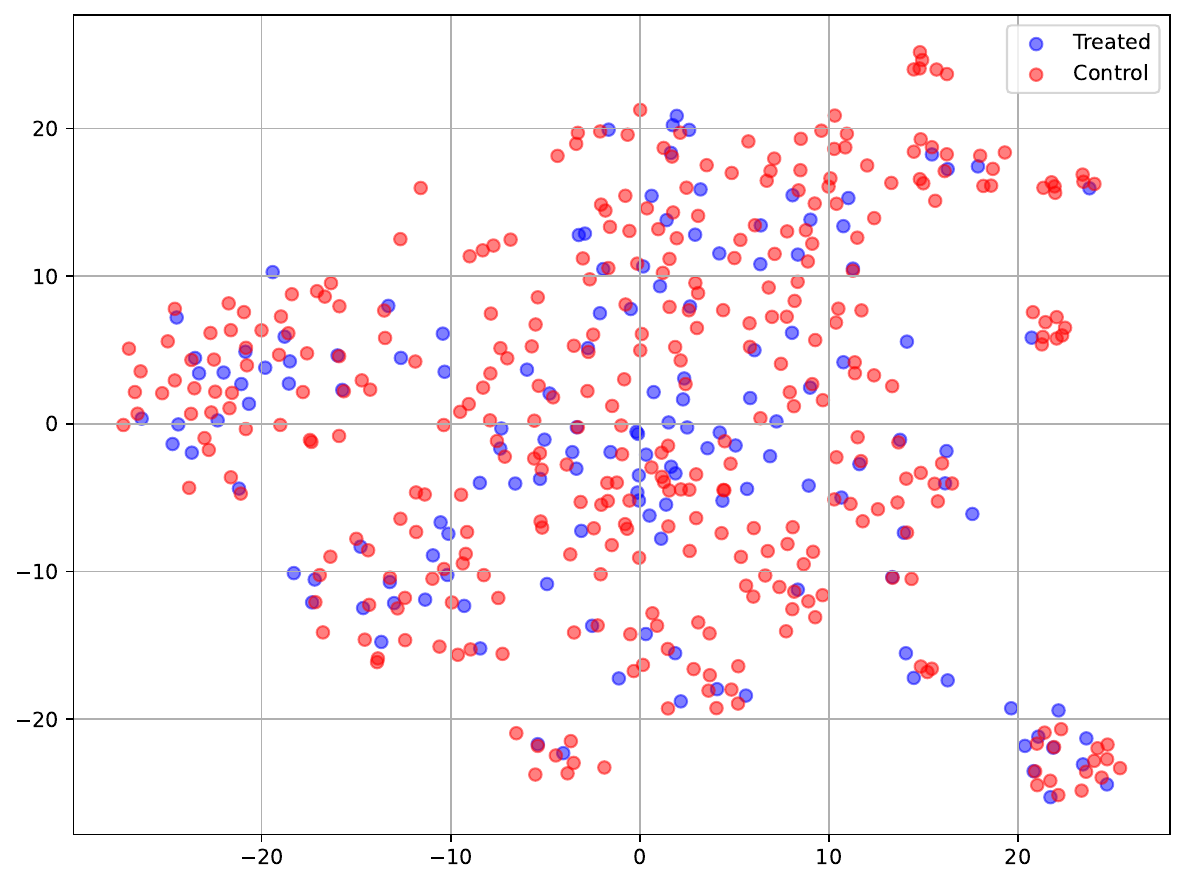}}
  \subfigure[QHTE at Step 30]{\includegraphics[width=0.25\textwidth]{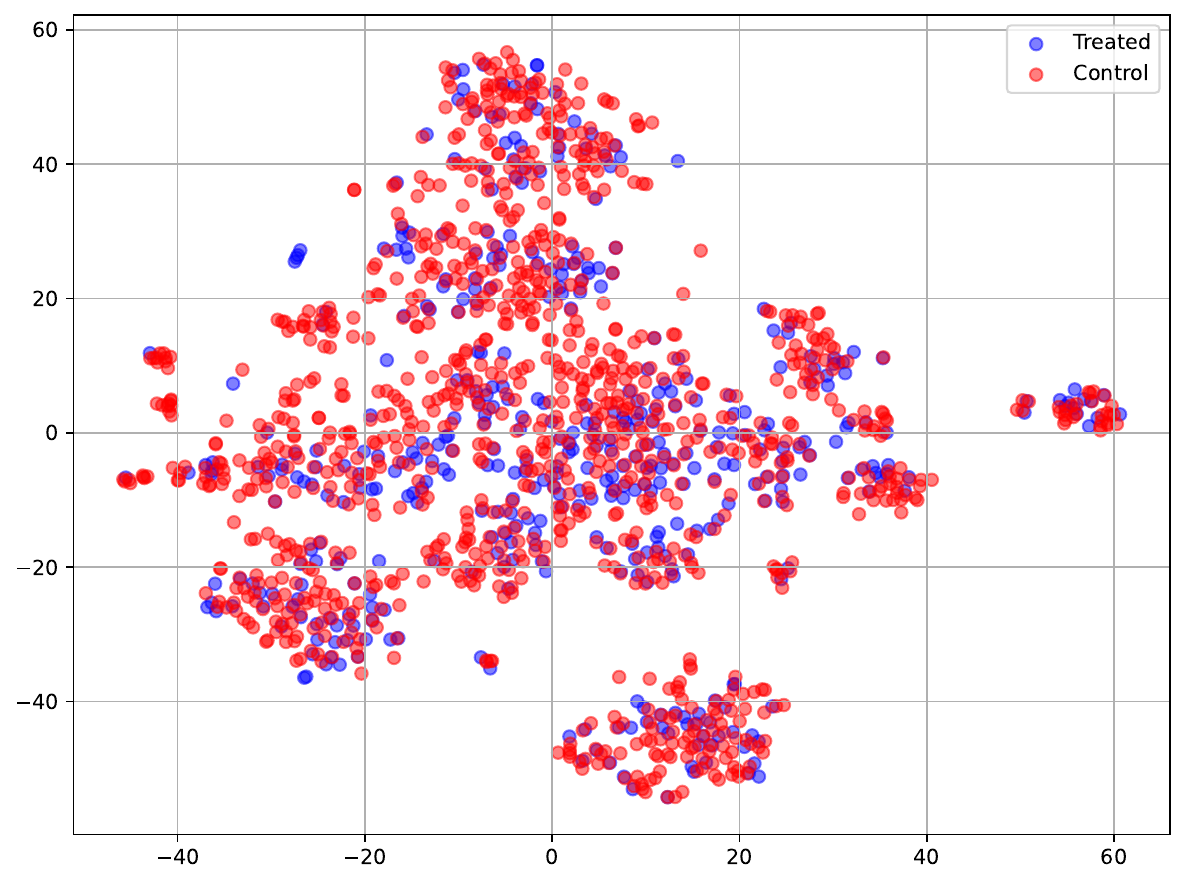}}
  \subfigure[QHTE at Step 50]{\includegraphics[width=0.255\textwidth]{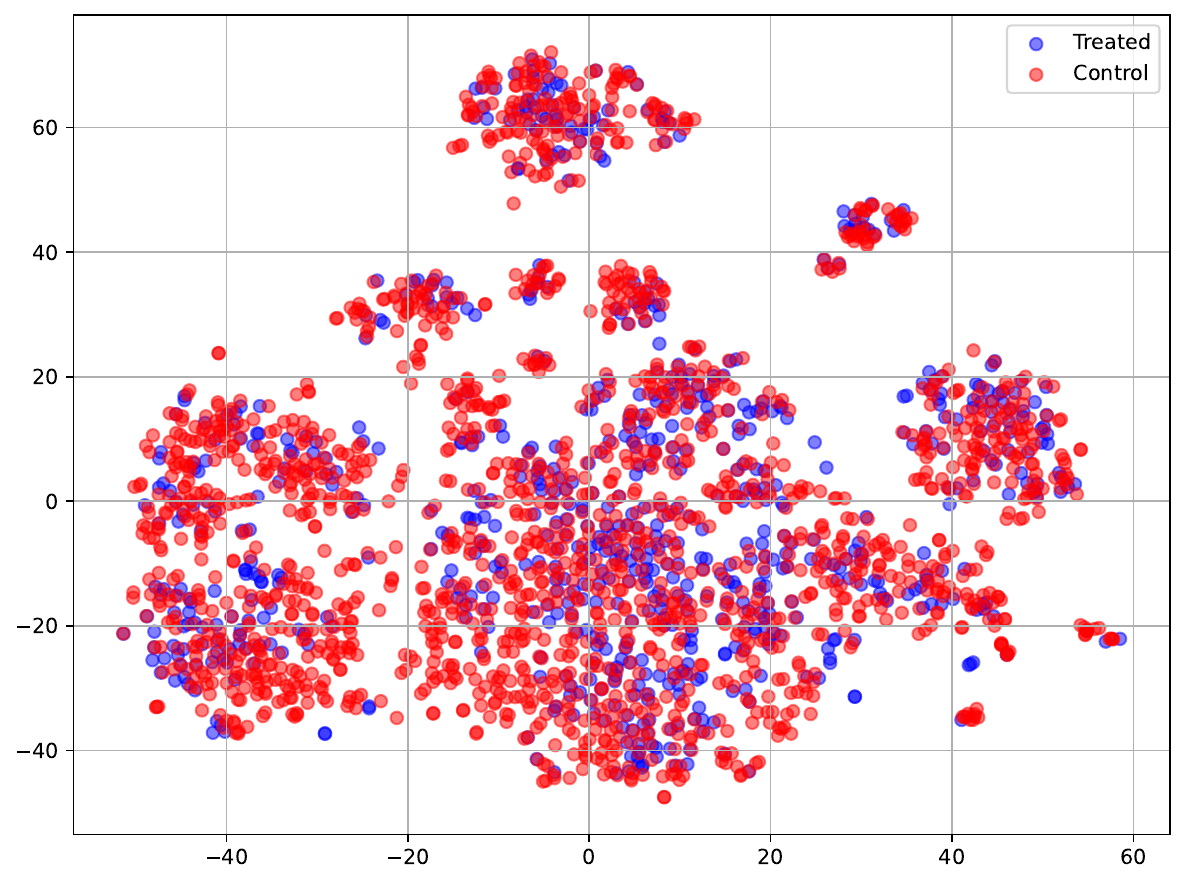}}\\
    \subfigure[$\mu\rho$BALD at Step 10]{\includegraphics[width=0.25\textwidth]{figures/appendix_tsne/ibm_truemurho_9.pdf}}
  \subfigure[$\mu\rho$BALD at Step 30]{\includegraphics[width=0.25\textwidth]{figures/appendix_tsne/ibm_truemurho_29.pdf}}
  \subfigure[$\mu\rho$BALD at Step 50]{\includegraphics[width=0.255\textwidth]{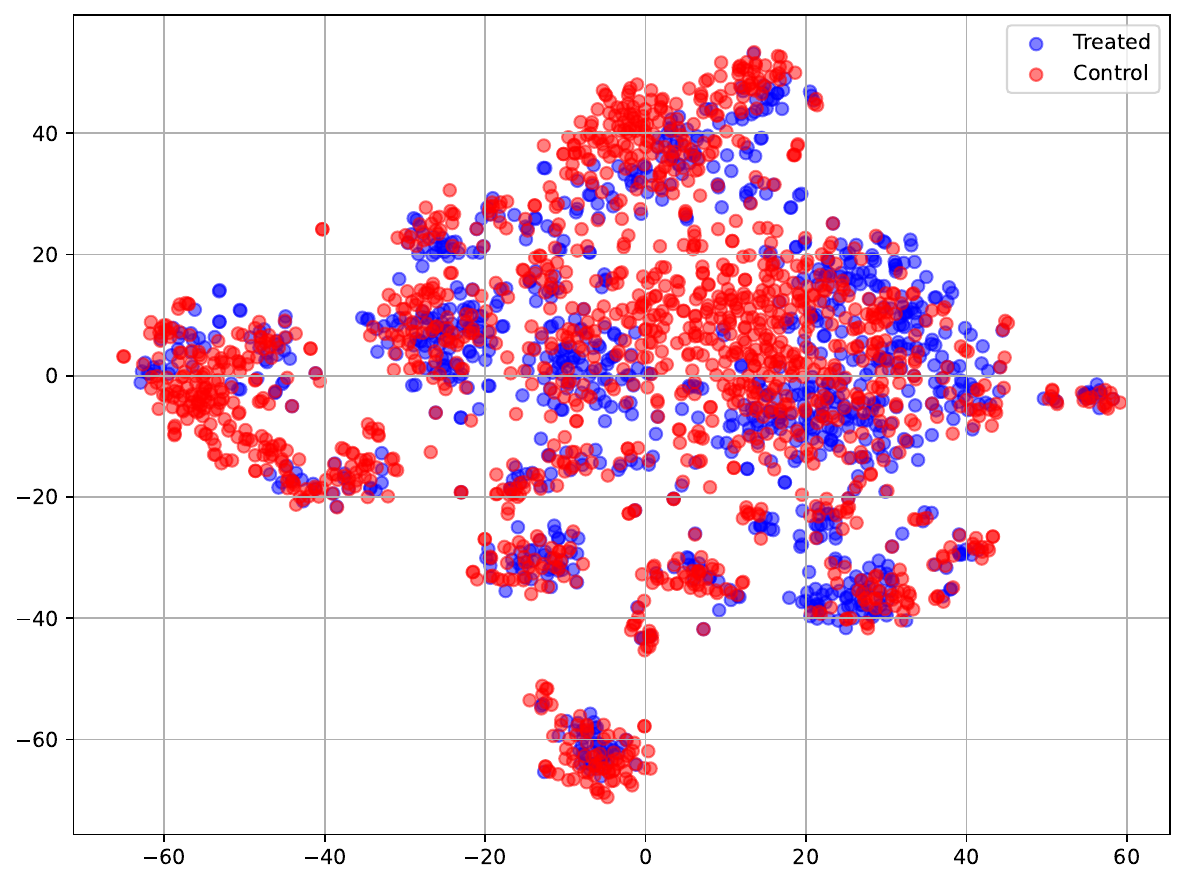}}
  
  \caption{Visualization of the post-acquisition training set at query step 10, 30, and 50 via t-SNE on IBM dataset.}\label{appendix:tsne_ibm}
\end{figure*} 

\begin{figure*}[h]
  \centering
  \subfigure[MACAL at Step 10]{\includegraphics[width=0.25\textwidth]{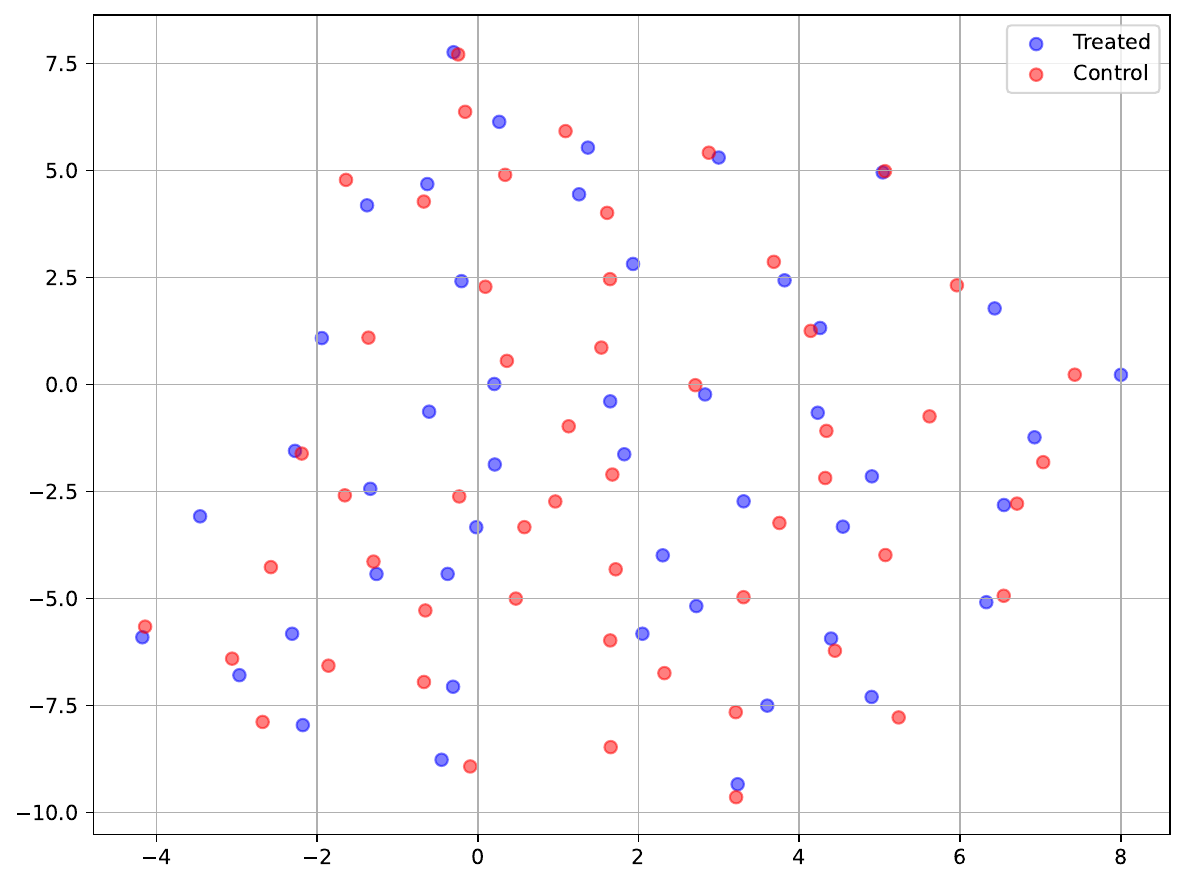}}
  \subfigure[MACAL at Step 15]{\includegraphics[width=0.25\textwidth]{figures/appendix_tsne/ihdp_truesim_2.5_14.pdf}\label{appendix:tsne_ibm_macal_15}}
  \subfigure[MACAL at Step 35]{\includegraphics[width=0.255\textwidth]{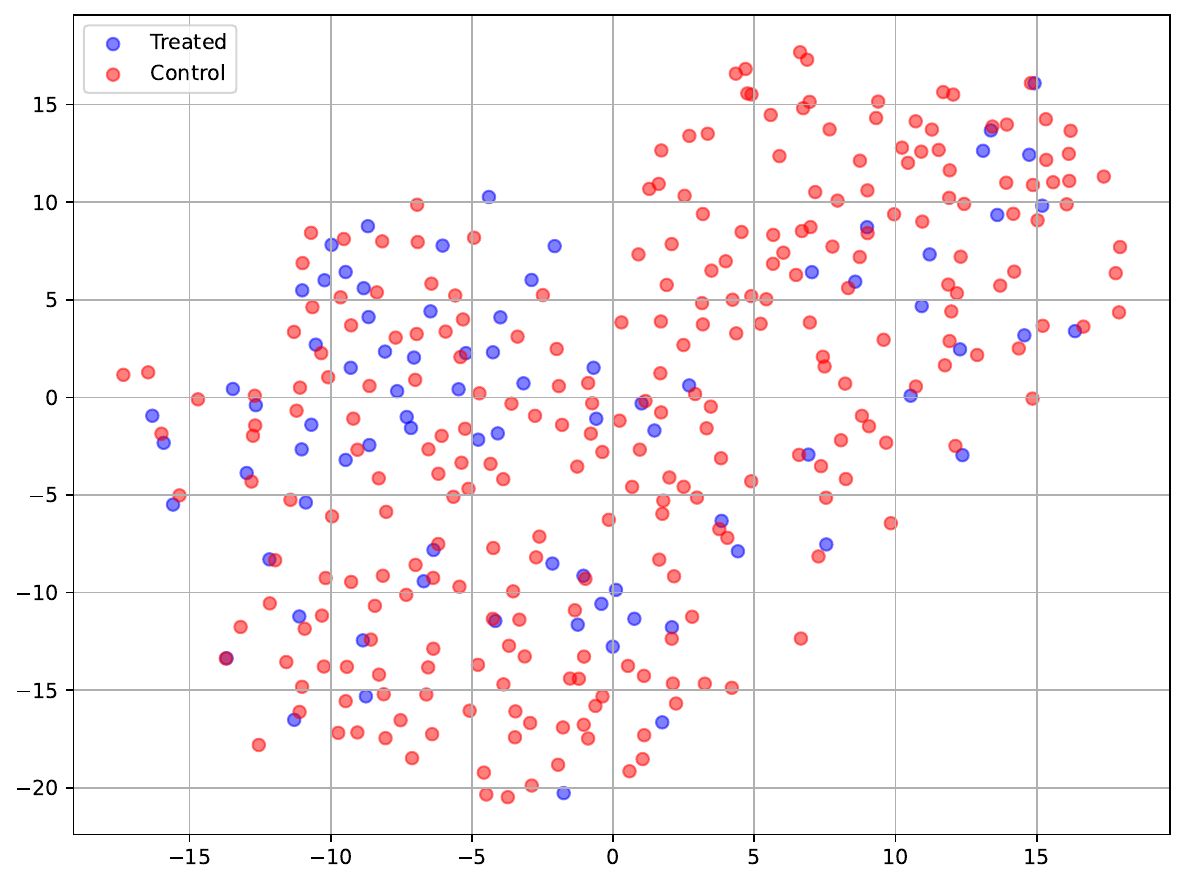}\label{appendix:tsne_ibm_macal_35}}\\
  \subfigure[Random at Step 10]{\includegraphics[width=0.25\textwidth]{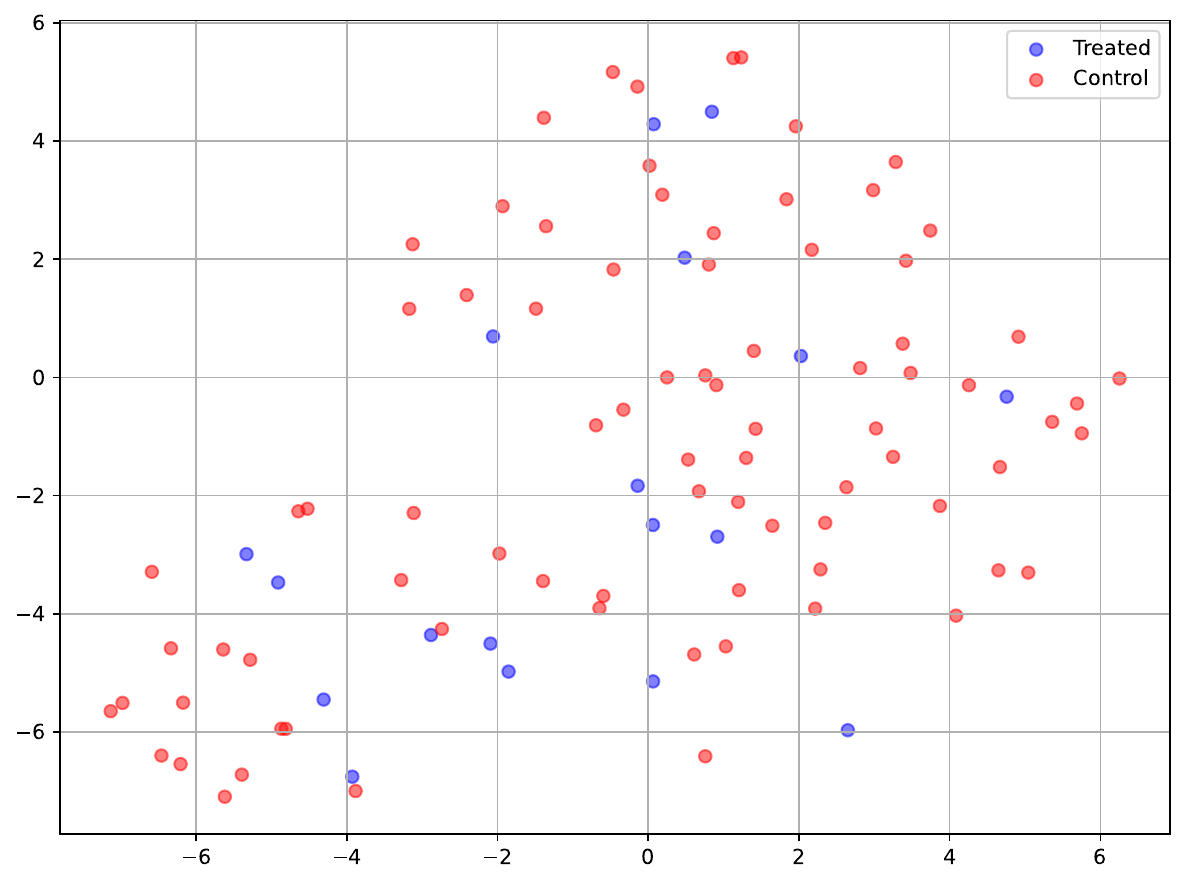}}
  \subfigure[Random at Step 15]{\includegraphics[width=0.25\textwidth]{figures/appendix_tsne/ihdp_truerandom_14.pdf}}
  \subfigure[Random at Step 35]{\includegraphics[width=0.255\textwidth]{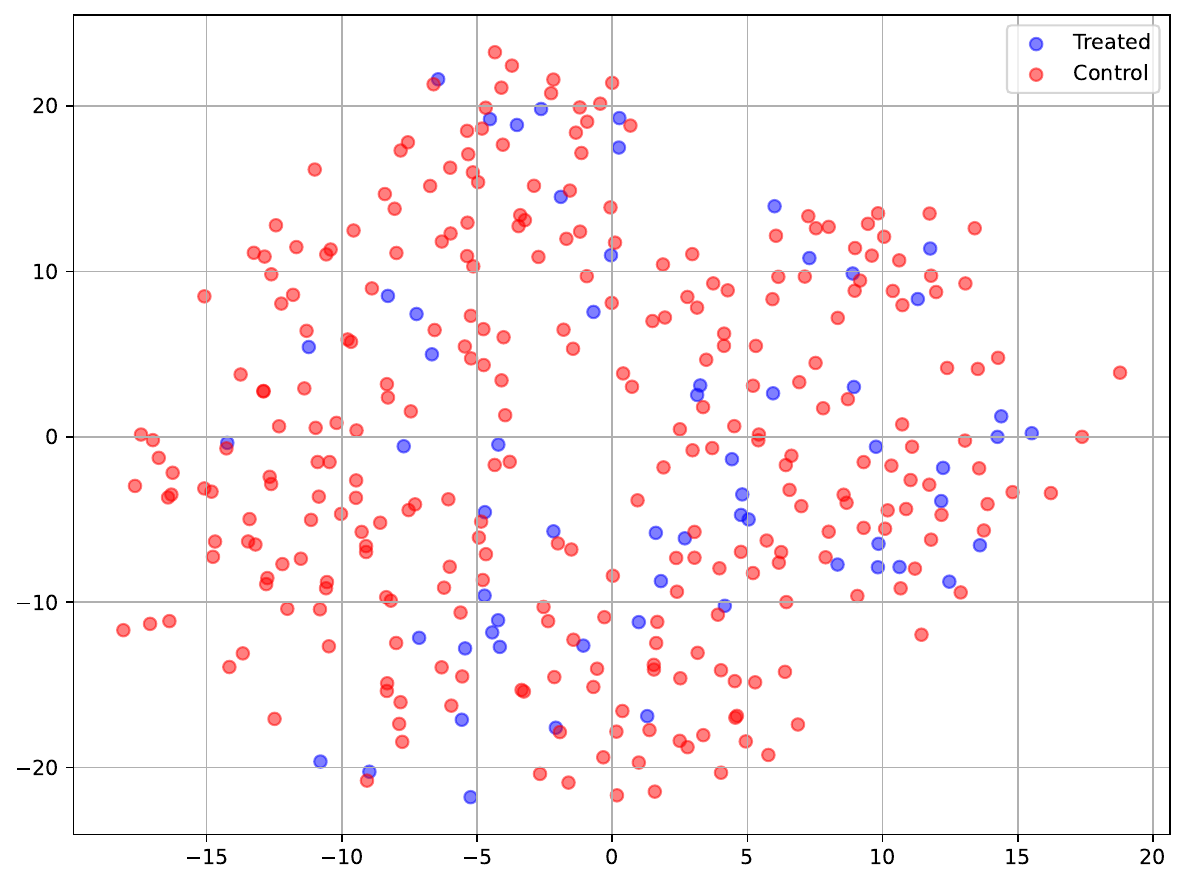}}\\
  \subfigure[LCMD at Step 10]{\includegraphics[width=0.25\textwidth]{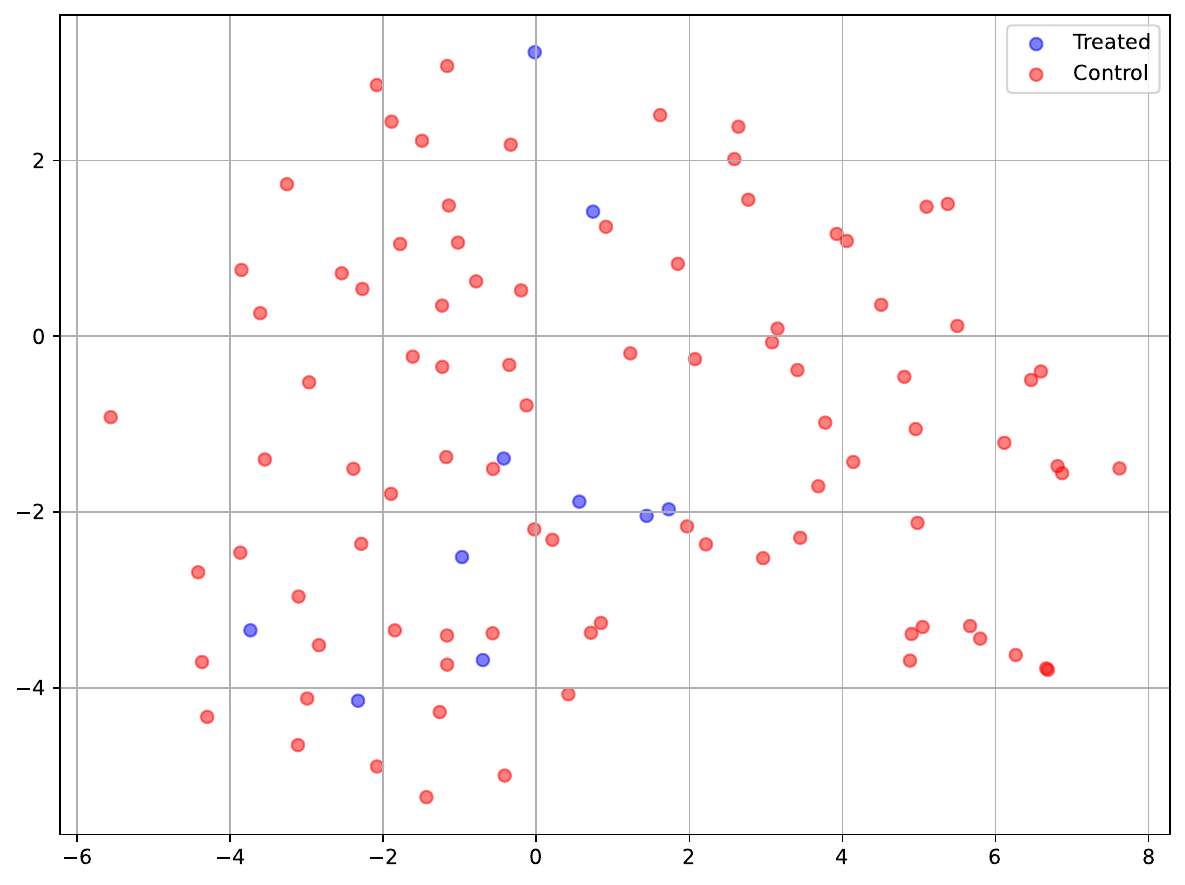}}
  \subfigure[LCMD at Step 15]{\includegraphics[width=0.25\textwidth]{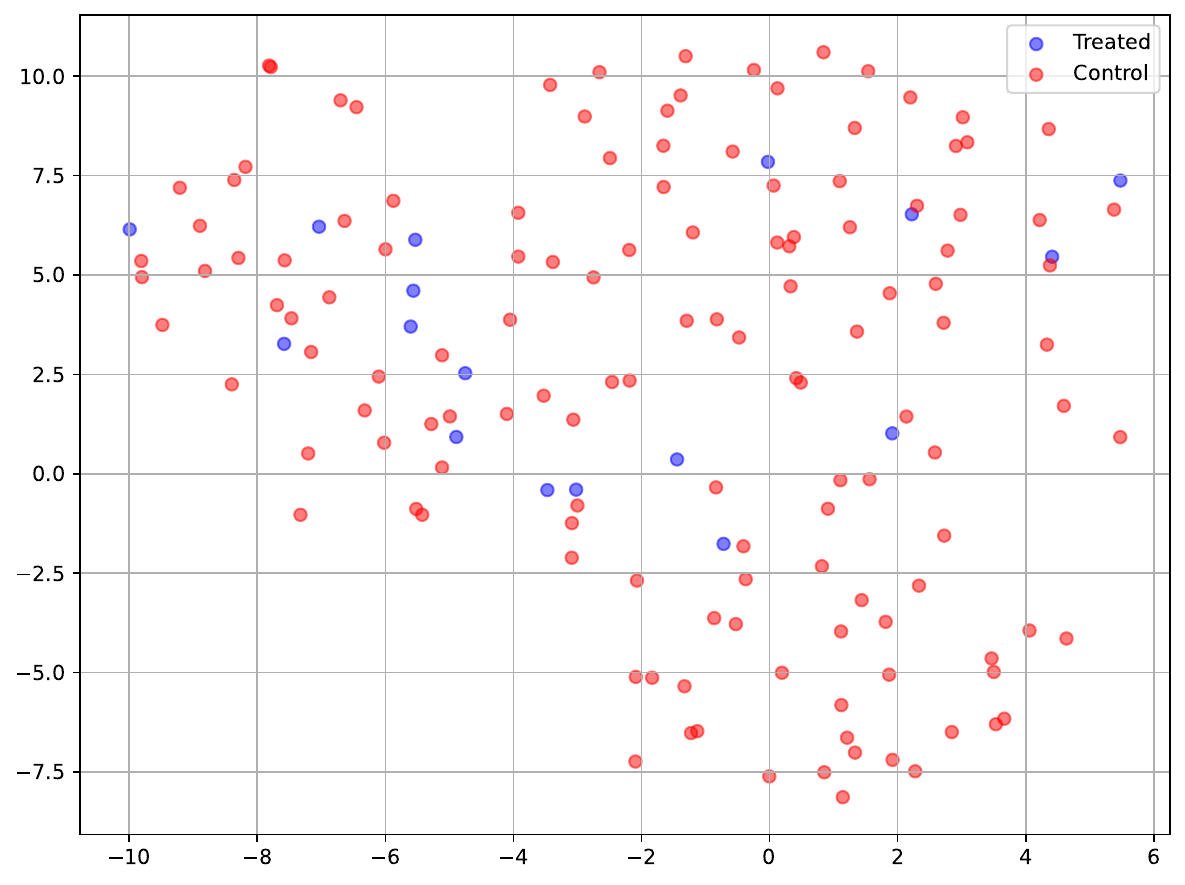}}
  \subfigure[LCMD at Step 35]{\includegraphics[width=0.255\textwidth]{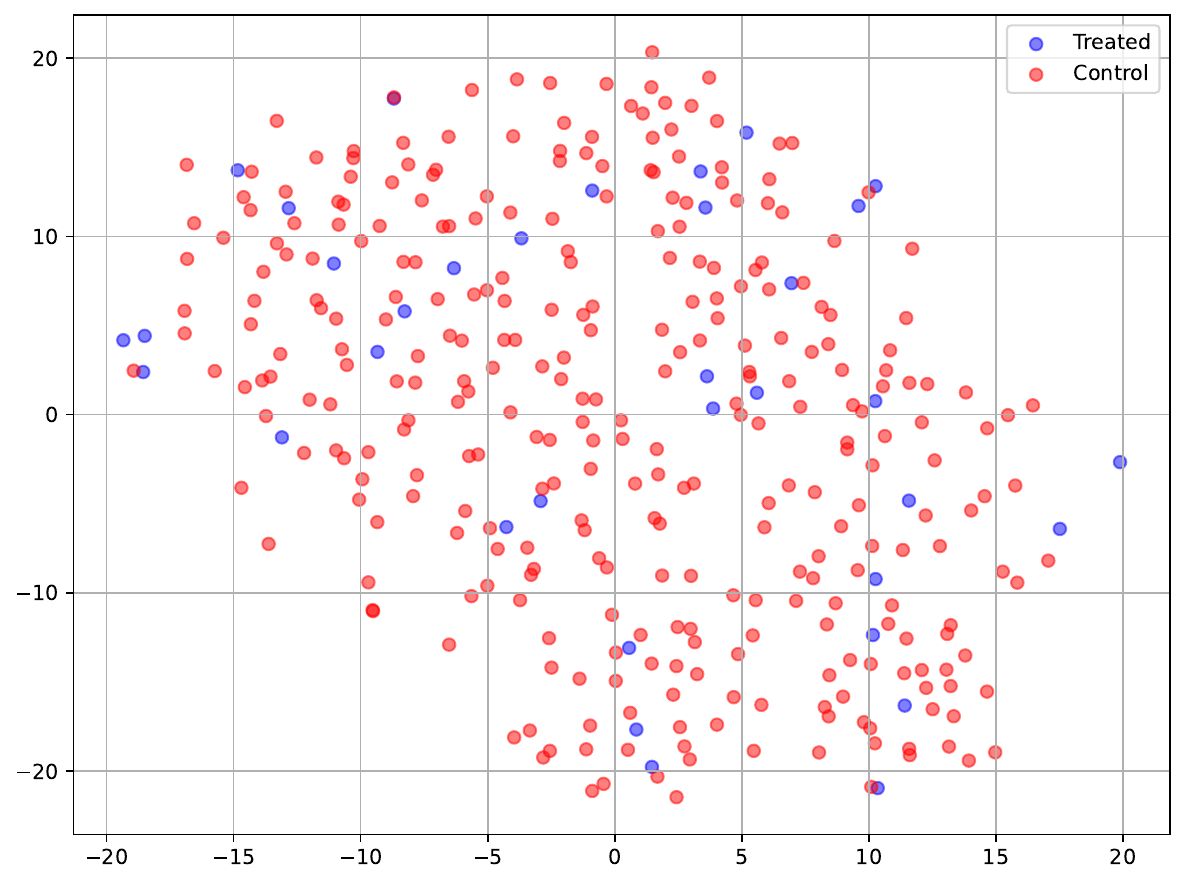}}\\
    \subfigure[QHTE at Step 10]{\includegraphics[width=0.25\textwidth]{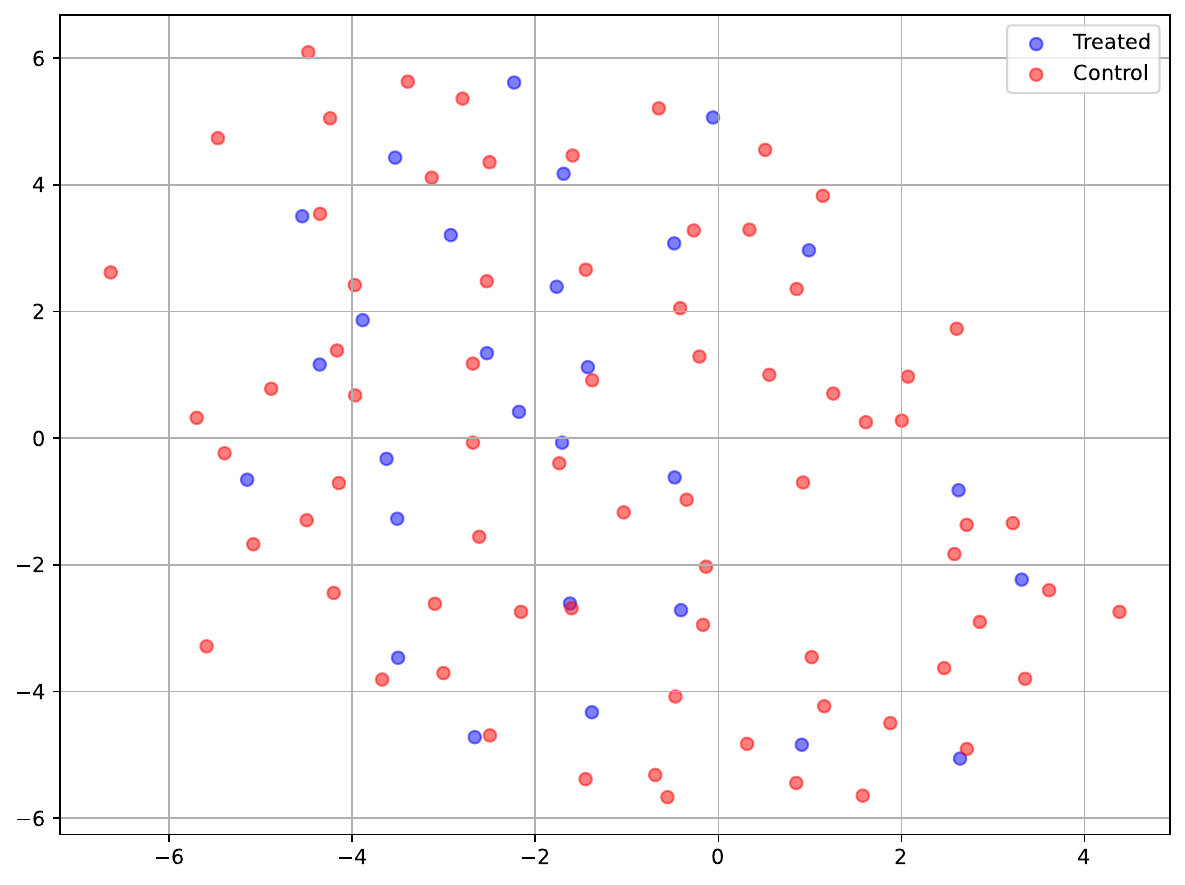}}
  \subfigure[QHTE at Step 15]{\includegraphics[width=0.25\textwidth]{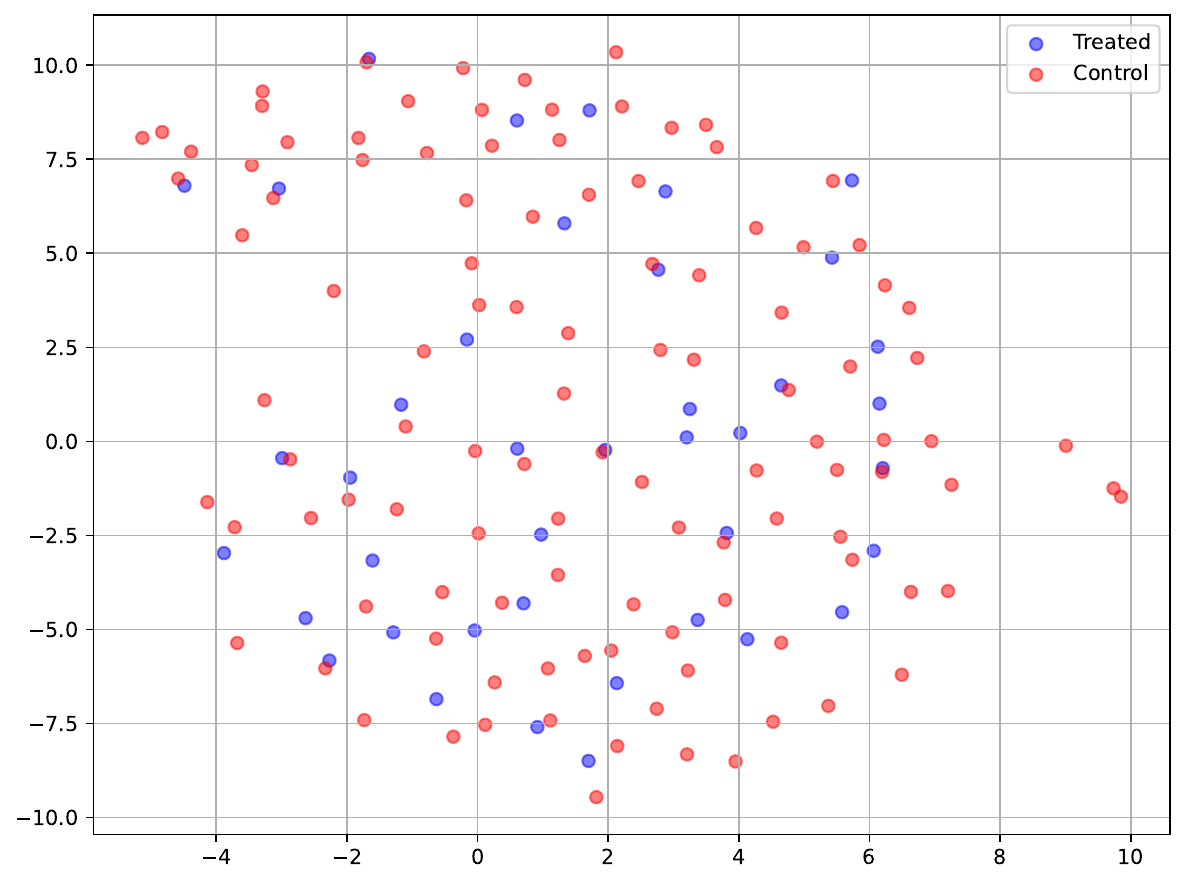}}
  \subfigure[QHTE at Step 35]{\includegraphics[width=0.255\textwidth]{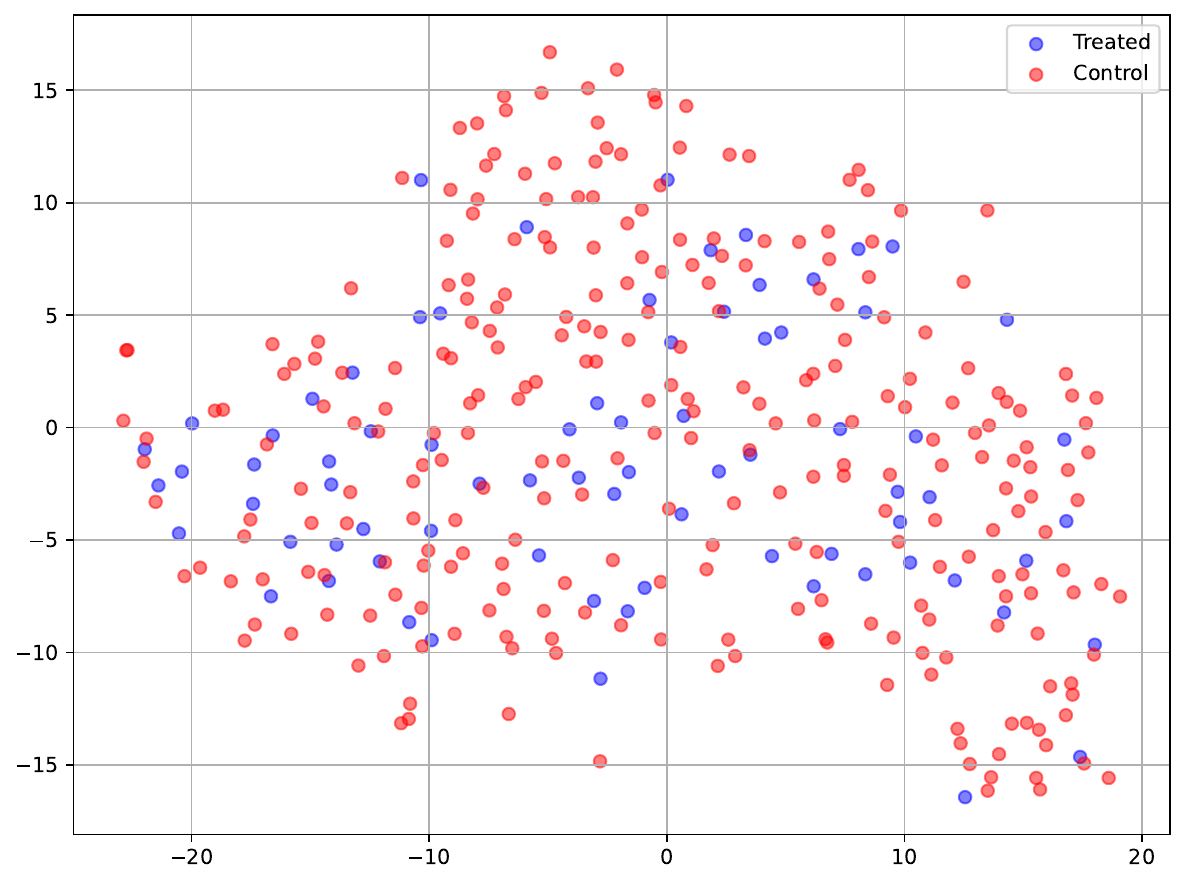}}\\
    \subfigure[$\mu\rho$BALD at Step 10]{\includegraphics[width=0.25\textwidth]{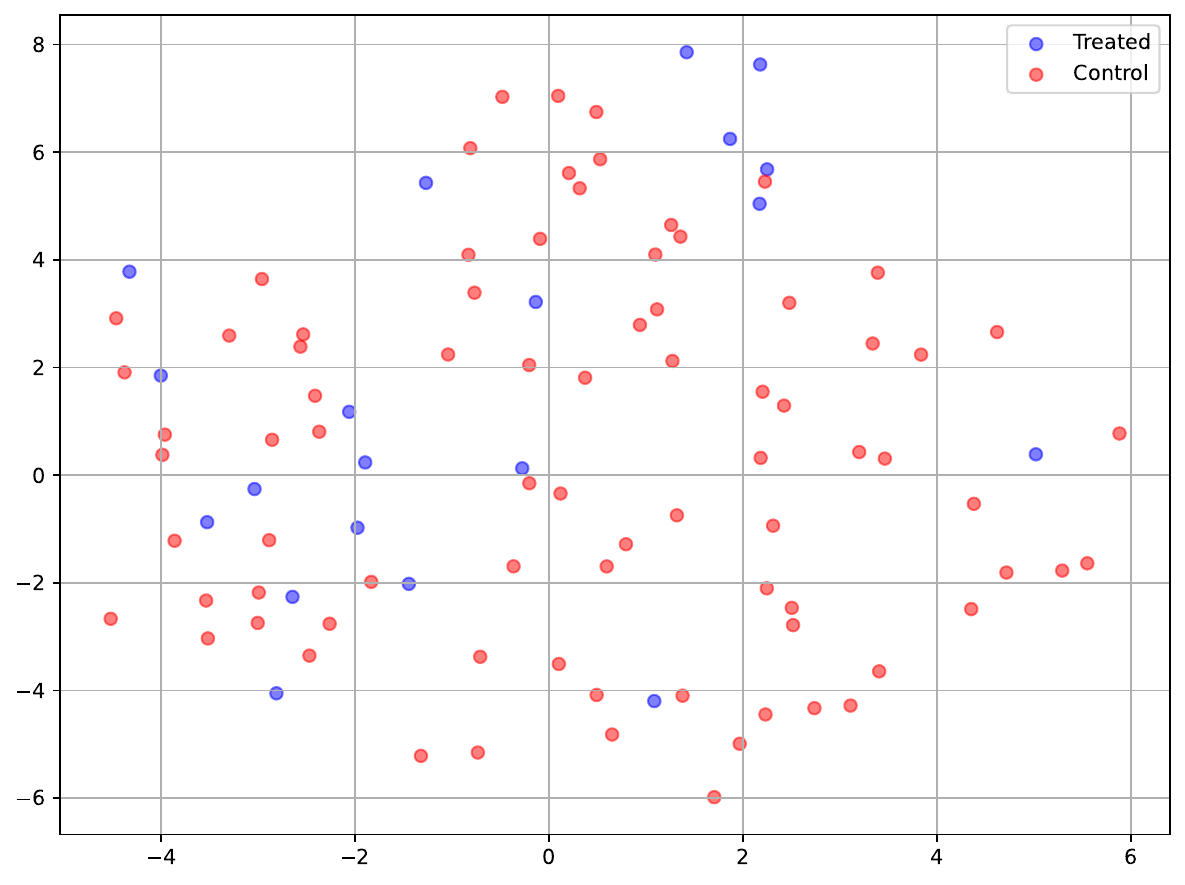}}
  \subfigure[$\mu\rho$BALD at Step 15]{\includegraphics[width=0.25\textwidth]{figures/appendix_tsne/ihdp_truemurho_14.pdf}}
  \subfigure[$\mu\rho$BALD at Step 35]{\includegraphics[width=0.255\textwidth]{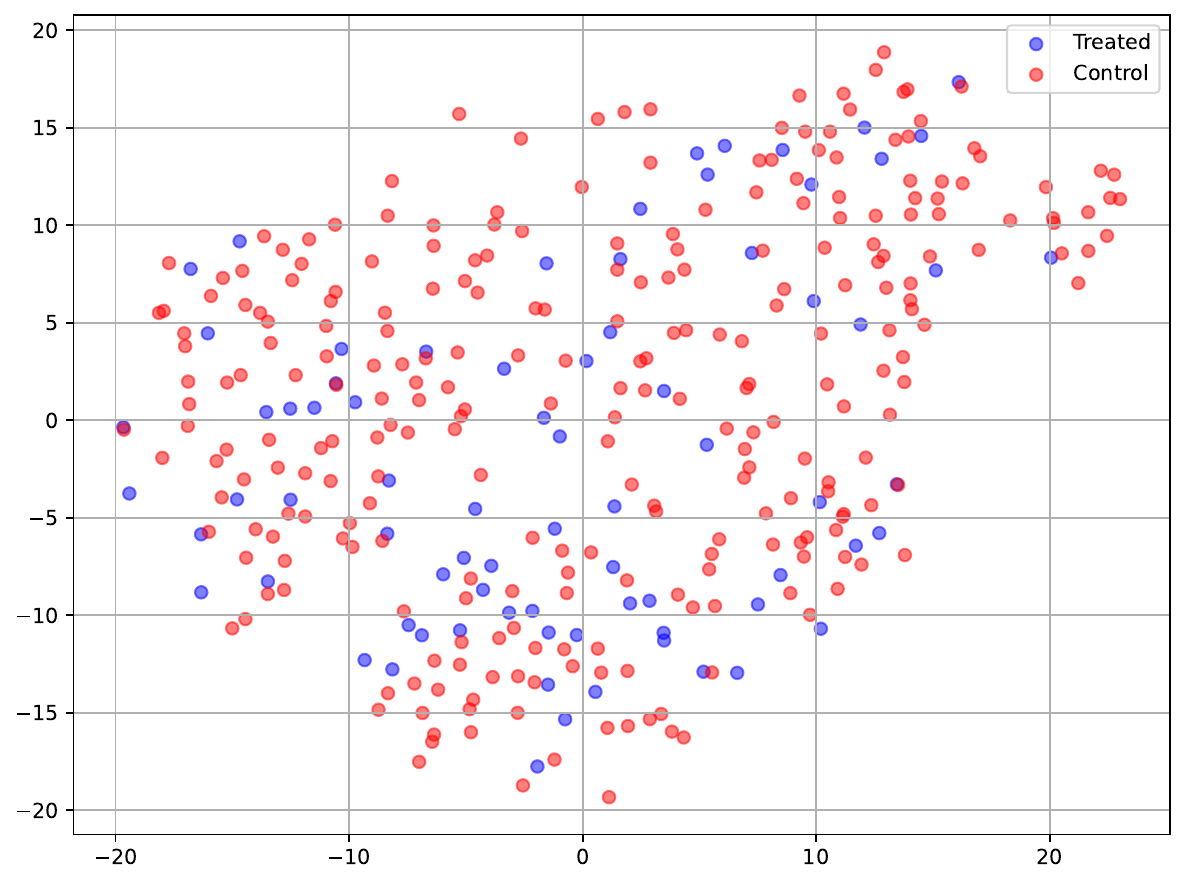}}
  
  \caption{Visualization of the post-acquisition training set at query step 10, 15, and 35 via t-SNE on IHDP dataset.}\label{appendix:tsne_ihdp}
\end{figure*}

\begin{table*}[t!]
\centering
\caption{Search Space and Tuned Hyperparameters for DUE-DNN and DUE-CNN}\label{table:search_space}
\begin{tabularx}{\textwidth}{X *{4}{>{\centering\arraybackslash}X}} 
\toprule
\multicolumn{1}{l}{Hyperparameters} & \multicolumn{1}{c}{Search Space}  & \multicolumn{1}{c}{DUE-DNN} & \multicolumn{1}{c}{DUE-CNN}\\
\midrule
Kernel & [RBF, Matern, RQ] & RBF & Matern\\
Inducing Points & [50, 100, 200] & 100 & 100\\
Hidden Neurons & [100, 200, 500] & 200 & 200\\
Depth & [2,3,5] & 3 & 2\\
Dropout Rate & [0.05, 0.1, 0.25] & 0.1 & 0.05\\
Spectral Norm & [0.95, 1.5, 3.0] & 0.95 & 3.0\\
Batch Size & [64, 100, 200] & 100 & 64\\
Learning Rate & [1e-3, 1e-4] & 1e-3 & 1e-3\\
\bottomrule
\end{tabularx}
\end{table*}

\subsection{Toy Dataset\label{appendix:toy_dataset}}

We simulate the one-dimensional toy dataset for a simple demonstration of the importance of considering minimizing the model variance and distributional discrepancy altogether during label acquisition.

\textbf{For samples with treatment status $t=1$}: the first 100 samples are from the interval of [-12, 10] with equal spacing, and the second 400 samples are from the normal distribution with mean -2.5 and variance 1. 

\textbf{For samples with treatment status $t=0$:} the first 500 samples are from the interval of [-10, 11] with equal spacing, and the second 2000 samples are from the normal distribution with mean 2.5 and variance 1. 

We have in total 500 samples with treatment status $t=1$ and 2500 samples with treatment status $t=0$ to form the imbalanced treatment groups as the entire dataset, then we do train/test split with 3:1 ratio for the model evaluation. The data-generating process is described mathematically as follows:
\begin{subequations}
\begin{align}
    &\begin{cases}
        &x^{t=1}_{i} = -12 + i\cdot\frac{10-(-12)}{100-1}, \text{ for } i\in[1, 100]\\
        &x^{t=1}_{i}\sim\mathcal{N}(-2.5, 1), \text{ for } i\in[101, 400]
    \end{cases}\\
    &\begin{cases}
        &x^{t=0}_{j} = -10 + j\cdot\frac{11-(-10)}{500-1}, \text{ for } j\in[1, 500]\\
        &x^{t=0}_{j}\sim\mathcal{N}(2.5, 1), \text{ for } j\in[101, 2000]
    \end{cases}\\
    &\begin{cases}
        &y^{t=1}_{i} = \sin{(2\cdot x^{t=1}_{i})}, \forall i\\
        &y^{t=0}_{j} = \cos{(2\cdot x^{t=1}_{j})}, \forall j
    \end{cases}
\end{align}
\end{subequations}

\subsection{Hyperparameters\label{appendix:hyperparameters}}

We conduct all the experiments with 48GB NVIDIA A40 on Ubuntu 22.04 LTS platform where GPU training is enabled, otherwise the 12th Gen Intel i7-12700K 12-Core 20-Thread CPU is used. The standard hyperparameter tuning on the validation set which is further split from the train set with 3:1 ratio, the best hyperparameters are selected with the smallest validation loss. Since the DUE models are borrowed from \cite{jesson2021causal}, we acknowledge the model set up from the previous literature and adopt a similar search space as shown in Table \ref{table:search_space}.

\section{Limitation and Future Work\label{appendix:limitation}}

In our proposed risk upper reduction theory, we make further claims for the risk convergence behaviour under two extreme circumstances due to the negligibility of the bounded constant $C_{\phi}$. We believe, the convergence analysis for each of the extreme situations can help justify the algorithm design, i.e., with negligible $C_{\phi}$ the risk upper bound shrinks to the variance term, where keep acquiring the most uncertain samples can enable the rate of convergence is lower-bounded by $\Omega(\beta^{i})$, while, with dominant $C_{\phi}$, the rate of convergence is upper-bounded by $\mathcal{O}(\frac{1}{i+\gamma_{0}})$. We also empirically observe these situations by setting different $C_{\phi}$ via the ablation study in Appendix \ref{section:ablations}, where it is clearly observed that MACAL with dominant $C_{\phi}$ performs the best at the start, but in the mid of the acquisition, a smaller $C_{\phi}$ (not negligible yet) obtains the best performance. However, the limitation of the convergence analysis is, due to technical difficulties, we do not obtain the risk convergence for the entire risk upper bound, i.e., when the $C_{\phi}$ sitting in the middle and making both the variance and the distributional discrepancy comparably important (which can be more realistic). We believe this point of research remains a import direction to be figured out in future work. 

Additionally, our designed algorithm MACAL, even though bring down the NP-hard combinatorial optimization to be approximately solved in polynomial time, i.e., $\mathcal{O}(N^{2}_{pool})$. When facing a significant large pool set with hundreds of millions of samples, the squared time complexity still suffers from considerable computational problems and become undesirable. Thus, future research on how to further reduce the algorithm time complexity is also an important direction to go when facing large real-world datasets.

\section{Broader Impacts\label{appendix:broader_impact}}

Causal effect estimation with active Learning could potentially have broader impacts on society if the algorithm is leveraged to deal with the treatment effect estimation in reality. One of the representative examples can be the hospital scenario, where patients' information is used for the training of the treatment effect estimator. 

When doing the active learning to selectively screen the samples and label them, once the AL algorithm identifies the informative sample to be labelled, the patient's individual information (features), and the corresponding treatment effect would be revealed. Subsequently, by labelling more informative samples, the positive impact is that a more precise treatment effect estimator can be trained on the ongoing growing training set, and help make more precise decision on the patient's treatment plan. However, the negative impact is, that the identified patients need to reveal their treatment information which can introduce privacy concerns and go against their will. Thus, when the causal effect active learning algorithm is used in the real world, the conductors should strictly consider the negative impact on the patient's privacy and its willing during the label acquisition process.


\end{document}